\documentclass[11pt,letterpaper]{article}

\usepackage[utf8]{inputenc}
\usepackage{comment}
\usepackage{microtype}
\usepackage{graphicx}
\usepackage{subcaption}
\usepackage{booktabs} 
\usepackage{tabularx}
\usepackage{caption}
\usepackage{placeins}
\usepackage{scalerel}
\usepackage[margin=1in]{geometry}
\usepackage[most]{tcolorbox}

\usepackage{amssymb}
\usepackage{amsthm}
\usepackage{amsmath}
\usepackage{amsfonts}
\usepackage{mathrsfs}
\usepackage{textcomp}
\usepackage{mathtools}

\usepackage{nicefrac}

\usepackage{algorithm}
\usepackage{algorithmicx}
\usepackage{algpseudocode}
\usepackage[shortlabels]{enumitem}
\usepackage{dsfont}

\usepackage{url}
\usepackage{float}

\usepackage{pifont}
\usepackage{hhline}
\usepackage{multirow}

\usepackage{graphicx,wrapfig}

\usepackage{xcolor} 

\newcommand{\xmark}{\ding{55}}

\newcommand{\cmark}{\ding{51}}

\newcommand{\ma}[1]{\textcolor{black}{#1}}
\newcommand{\rma}[1]{\textcolor{black}{#1}} 
\DeclareMathOperator*{\argmin}{arg\,min}

\newtheorem{assumption}{Assumption}
\newtheorem{definition}{Definition}
\newtheorem{theorem}{Theorem}
\newtheorem{lemma}[theorem]{Lemma}

\newtheorem{corollary}[theorem]{Corollary}
\theoremstyle{remark}
\newtheorem{remark}{Remark}

\numberwithin{assumption}{section}
\numberwithin{definition}{section}
\numberwithin{theorem}{section}
\numberwithin{remark}{section}

\usepackage{hyperref}
\hypersetup{colorlinks,citecolor=blue,linktocpage,breaklinks=true}
\begin{document}

\title{\Huge Simultaneous Learning and Optimization via Misspecified Saddle Point Problems}

\author{Mohammad Mahdi Ahmadi\thanks{Department of Industrial and Systems Engineering, Mississippi State University, Starkville, MS, USA.\\ \qquad ahmadi@ise.msstate.edu}
\quad Erfan Yazdandoost Hamedani\thanks{Department of Systems and Industrial Engineering, The University of Arizona, Tucson, AZ, USA.\\ \qquad erfany@arizona.edu}
\vspace{5mm}
}

\date{}

\maketitle

\begin{abstract} 

We study a class of misspecified saddle point (SP) problems, where the optimization objective depends on an unknown parameter that must be learned concurrently from data. Unlike existing studies that assume parameters are fully known or pre-estimated, our framework integrates optimization and learning into a unified formulation, enabling a more flexible problem class. To address this setting, we propose two algorithms based on the accelerated primal-dual (APD) by \cite{hamedani2021primal}. In particular, we first analyze the \textit{Naive} extension of the APD method by directly substituting the evolving parameter estimates into the primal-dual updates; then, we design a new \textit{Learning-aware} variant of the APD method that explicitly accounts for parameter dynamics by adjusting the momentum updates. Both methods achieve a provable convergence rate of $\mathcal{O}(\log K / K)$, while the Learning-aware approach attains a tighter \ma{$\mathcal{O}(1)$} constant and further benefits from an adaptive step-size selection enabled by a backtracking strategy. Furthermore, we extend the framework to problems where the learning problem admits multiple optimal solutions, showing that our modified algorithm for a structured setting achieves an $\mathcal{O}(1/\sqrt{K})$ rate. 
To demonstrate practical impact, we evaluate our methods on a misspecified portfolio optimization problem and show superior empirical performance compared to state-of-the-art algorithms.
\end{abstract}



\section{Introduction}\label{sec:intro}

Saddle point (SP) problems, also known as min-max optimization, represent a significant class of optimization problems that have received considerable attention in recent years \cite{chambolle2016ergodic,hamedani2021primal,nemirovski2004prox,ouyang2015accelerated}. Their versatile formulation encompasses a wide range of optimization challenges, including convex optimization with nonlinear conic constraints. This broad category further includes linear programming (LP), quadratic programming (QP), quadratically constrained quadratic programming (QCQP), second-order cone programming (SOCP), and semidefinite programming (SDP) as subclasses. In these optimization problems, a common and crucial assumption is that all model parameters are precisely known; however, real-world scenarios often lack explicit parameter availability, leading to uncertainty that affects decision-making \cite{bertsekas1997nonlinear,murray2019practical}. 
Stochastic optimization \cite{jie2018stochastic,shapiro2021lectures,wilson2018adaptive} provides a framework for addressing uncertainty by modeling input parameters as random variables with known probability distributions. However, in many practical scenarios, these distributions are not explicitly known. Robust optimization offers an alternative by assuming that the uncertain parameters or their distributions lie within a predefined uncertainty set. While this approach avoids reliance on precise distributions, it often leads to overly conservative solutions, limiting its applicability in certain contexts. A more flexible alternative is \textit{misspecified optimization}, which incorporates learning the unknown parameters through a secondary optimization problem. \ma{Let ($\mathcal{X},\|\cdot\|_{\mathcal{X}}$), ($\mathcal{Y},\|\cdot\|_{\mathcal{Y}}$), ($\Theta,\|\cdot\|_{\Theta}$), ($\mathcal{W},\|\cdot\|_{\mathcal{W}}$) be finite-dimensional real vector spaces}. In this paper, we adopt a broader perspective and investigate the following misspecified SP problem:
\begin{equation}\label{eq:sp-prob}
\begin{aligned}
   & (Optimization):\quad\phi(\theta^*) \triangleq \min_{x\in \mathcal{X}}\max_{y\in \mathcal{Y}} \mathcal{L}(x,y;\theta^*) \triangleq f(x) + \Phi(x,y;\theta^*) - h(y),\\
   & (Learning):\quad \qquad\theta^* \in \argmin_{\theta \in \Theta}\max_{w\in\mathcal{W}} l(\theta,w) \triangleq f_L(\theta) + \ell(\theta,w) - h_L(w),
\end{aligned}
\end{equation}
where $f: \mathcal{X} \rightarrow \mathbb{R} \cup \{+\infty\}$, $h: \mathcal{Y} \rightarrow \mathbb{R} \cup \{+\infty\}$, $f_L: \Theta \rightarrow \mathbb{R} \cup \{+\infty\}$, and $h_L: \mathcal{W} \rightarrow \mathbb{R} \cup \{+\infty\}$ are convex functions \ma{(possibly nonsmooth)}; $\Phi: \mathcal{X} \times \mathcal{Y} \times \Theta \rightarrow \mathbb{R}$ is a continuously differentiable function such that for any $\theta \in \Theta$, $\Phi(\cdot, \cdot; \theta)$ is convex in $x$ and concave in $y$; and $\ell: \Theta \times \mathcal{W} \rightarrow \mathbb{R}$ is continuously differentiable, convex in $\theta$, and linear in $w$.
The problem in \eqref{eq:sp-prob} has a wide range of applications including portfolio optimization problem \cite{aybat2021analysis,bertsimas2008robust,goldfarb2003robust,Markowitz_1952}, inventory control problem \cite{gallego1993distribution,ghosh2021new,janssens2011linear,thorsen2017robust}, and power systems \cite{jiang2016solution,jiang2011robust,jiang2014two,zhao2013unified}. \\
One approach for solving a problem under misspecification, such as problem \eqref{eq:sp-prob}, is to use a sequential method. First, estimate $\theta^*$ with high accuracy through the \textit{(Learning)} process, and then solve the parameterized \textit{(Optimization)} problem in \eqref{eq:sp-prob}. However, this approach could have several downsides as discussed in \cite{ahmadi2020resolution,ho2019exploiting,jiang2016solution,yang2025analysis}: first, in the case of a large-scale \textit{(Learning)} problem, achieving highly accurate parameter estimation within a reasonable time becomes impractical, leading to a significant delay in solving the \textit{(Optimization)} problem. Second, due to the lack of asymptotic convergence in the sequential method, this approach can only provide approximate solutions for the \textit{(Learning)} problem. Consequently, the resulting approximation error will impact the resolution of the subsequent \textit{(Optimization)} problem. An alternative approach is to solve the \textit{(Learning)} and \textit{(Optimization)} problems simultaneously, such that, using the available observations, at each iteration $k$, $\theta_k$, which is an increasingly exact estimate of the misspecified parameter $\theta^*$, is generated by a distinct process as a minimizer of the \textit{(Learning)} problem and then used in the \textit{(Optimization)} problem to solve it in the same iteration \cite{aybat2021analysis,jiang2013solution,jiang2016solution}. This scheme produces a sequences $\left\{(x_k, y_k)\right\}$ and $\left\{\theta_k\right\}$ such that: $\lim_{k \rightarrow \infty} (x_k, y_k) = (x^*, y^*)$ and $\lim_{k \rightarrow \infty} \theta_k = \theta^*$.
\\
It is noteworthy to mention that the structure of the misspecified problem in \eqref{eq:sp-prob} differs from that of a simple bilevel optimization setting \cite{shen2023online,dempe2020bilevel}. In simple bilevel optimization, the goal is to find the optimal decision set for the lower-level objective and use it to minimize the smooth upper-level objective, thereby determining the overall optimal decision. In contrast, misspecified optimization involves learning an \textit{unknown parameter} by solving a separate learning problem, whose decision variables are independent of the main optimization task. This learned parameter is then used to solve the original optimization problem.

\subsection{Literature Review}
In this section, we briefly discuss common approaches for solving optimization problems with unknown or misspecified parameters. We then focus on the existing methods and problem setup in misspecified optimization.

\paragraph{Stochastic Optimization.}
In the structure of a stochastic optimization problem, it is assumed that the unavailable parameter $\theta$ follows a probability distribution $\mathcal{D}$, which can be stated as: $\min_{x\in \mathcal{X}}\mathbb{E}_{\theta\sim \mathcal{D}}\left[\Phi(x;\theta) \right]$
\cite{jie2018stochastic,shapiro2021lectures,wilson2018adaptive}. The stochastic optimization approach has a wide range of applications in solving optimization problems \cite{uryasev2013stochastic}; however, it may face challenges when data availability for estimation is limited or noisy \cite{shapiro2021lectures}.
\paragraph{Robust Optimization.} 
In the formulation of robust optimization (RO) problems, the parameter $\theta$ is unknown and is represented by an uncertainty set $\mathcal{U}_{\theta}$ \cite{bertsimas2011theory,namkoong2016stochastic}. In RO problems, the objective is to optimize the worst-case scenario across all possible values of the unknown parameter $\theta$ within the uncertainty set $\mathcal{U}_{\theta}$: $\min_{x\in \mathcal{X}}\max_{\theta \in \mathcal{U}_{\theta}}\Phi(x;\theta)$.
Despite its usefulness in certain applications such as optimization \cite{ben2009robust}, design, and control \cite{bertsimas2007constrained,burger2013polyhedral}, RO can lead to conservative solutions and become intractable if the uncertainty set is not appropriately defined \cite{aybat2021analysis,bertsimas2018data}.
\paragraph{Misspecified Problems and Data-driven Approaches.}
\ma{In recent optimization literature, there has been a growing focus on problems involving one or more unknown or misspecified parameters. The key distinction between these problems and those studied in stochastic or robust optimization lies in how the uncertainty is addressed. In stochastic and robust optimization, the uncertain parameter is typically modeled as a random variable or captured through an uncertainty set, which often leads to conservative solutions. In contrast, misspecified problems are approached in a \emph{data-driven} manner: the unknown parameter is learned from the observational data, rather than being treated purely as randomness or bounded uncertainty \cite{yang2025data,yang2025analysis}.}

\ma{Several works have emphasized the theoretical modeling and formulation of such problems \cite{bertsimas2018data,shang2017data}. For instance, \cite{bertsimas2018data} leverages data to design uncertainty sets for robust optimization. On the other hand, there is increasing interest in developing algorithms that explicitly integrate learning mechanisms to address parameter misspecification through data-driven approaches \cite{yang2025data, yang2024data}}. 

\ma{In the existing literature, the optimization problems with misspecified parameters have been studied in various settings, including misspecified optimization, Nash equilibrium problems, and variational inequality (VI) formulations.}

\textbf{Misspecified Optimization.} In this area, several studies concentrate solely on minimizing an objective function in both deterministic \cite{ahmadi2016rate,ahmadi2014data,ahmadi2020resolution,aybat2021analysis,ho2019exploiting} and stochastic settings \cite{jiang2013solution,jiang2016solution,jiang2017distributed,song2019stochastic,yang2025data,yang2024data,yang2025analysis}. In deterministic setup, \cite{ahmadi2014data,ahmadi2020resolution} consider a misspecified optimization problem where the objective is to minimize a strongly or merely convex function \(\Phi(x;\theta^*)\) with respect to \(x\) over a closed and convex set \(\mathcal{X}\). Here, \(\theta^*\) is an unknown or misspecified vector of parameters that represents the solution of a distinct strongly convex objective function $g(\theta)$. It is assumed that \(\theta^*\) can be learned through a parallel learning process, which produces a sequence of estimators \(\theta_k\), each progressively improving in accuracy as an approximation of \(\theta^*\). \ma{The authors in \cite{ahmadi2014data,ahmadi2020resolution} study two main schemes: gradient-based methods and subgradient methods, to achieve the canonical convergence rates of \(\mathcal{O}(1/K)\) and \(\mathcal{O}(1/\sqrt{K})\), respectively. When strong convexity of $g(\cdot)$ assumptions are relaxed, their results include an additional error term in the convergence bounds that scales with the initial parameter estimation error, specifically proportional to \(\|\theta_0 - \theta^*\|\).}
In other studies, \cite{ahmadi2016rate,aybat2021analysis} investigate a convex optimization problem subject to a constraint set that includes a misspecified parameter, which needs to be learned through a separate process. They develop a first-order inexact parametric augmented Lagrangian method (IPALM) that consists of inner and outer loops. The inner loop corresponds to iterations that update \( x_{k+1} \) by employing a particular implementation of the accelerated proximal gradient (APG) algorithm \cite{nesterov2013introductory}.  
The outer loop is designed to update the Lagrangian multiplier. They compute the overall iteration complexity for the constant penalty case and the increasing penalty parameter sequence and show that it requires at most $\mathcal{O}(\epsilon^{-4})$ and $\mathcal{O}(\epsilon^{-1}\log(\epsilon^{-1}))$ proximal-gradient computations, respectively. In the stochastic setting, \cite{jiang2013solution,jiang2016solution} consider a misspecified stochastic convex optimization problem, $ \mathbb{E}\left[\Phi(x; \theta^*, \xi)\right] $, where $ \theta^* $ is a misspecified parameter that may be learned by minimizing a distinct stochastic objective function $ \mathbb{E}\left[g(\theta; \eta)\right] $. They propose a coupled stochastic approximation scheme to solve both the optimization and learning problems simultaneously, demonstrating that the method possesses almost sure convergence properties when $\Phi$ is strongly convex or convex. \ma{Moreover, they establish convergence rates for the upper bound on the mean-squared error of the decision variable as $\mathcal{O}(1/K)$ under strong convexity of $\Phi$. In contrast, for the merely convex setting, they derive an upper bound on the expected objective gap, which decays at the rate of $\mathcal{O}(\sqrt{\ln{K}}/\sqrt{K})$.}
 In a recent work, \cite{yang2024data} present a class of primal-dual algorithms to solve a nonsmooth convex-concave stochastic minimax optimization problem with expectation constraints as $\min_{x\in\mathcal{X}}\max_{\theta}\mathbb{E}\left[\Phi(x;\theta,\xi)\right]$ s.t. $\mathbb{E}\left[g(\theta,\eta)\right]\leq g^*:=\min_{\theta}\mathbb{E}\left[g(\theta,\eta)\right]$. Moreover, they demonstrate the optimal rate of $\mathcal{O}(1/\sqrt{K})$ for their proposed algorithms. It is important to emphasize that the problem setting considered in their work does not encompass the misspecified SP problem \eqref{eq:sp-prob} addressed in this paper. When adapted to misspecified problems, as outlined in their motivating example, their focus is on misspecified minimization problems in which the parameter $\theta$ is learned through a secondary minimization task that may have multiple solutions. To handle the challenge of multiple learning solutions, their approach adopts a pessimistic formulation -- optimizing for the worst-case parameter -- subject to certain structural assumptions. In particular, their framework requires the objective function to be concave with respect to $\theta$, and the reformulated constrained problem includes the optimal value of the learning problem.\\
\textbf{Nash Game Problems.} The authors in \cite{jiang2011learning} and \cite{jiang2017distributed} investigate misspecified convex Nash games under deterministic and stochastic settings, respectively. Specifically, \cite{jiang2011learning} examines a constrained Nash-Cournot game where the learning process is linear in the parameter. The study introduces Tikhonov regularization and single-timescale distributed schemes that ensure convergence to an equilibrium and the solution of the equivalent variational inequality problem. \ma{In particular, the learning problem is formulated as a variational inequality where the associated mapping is linear in the unknown parameter.} In \cite{jiang2017distributed}, the authors study the distributed computation of Nash equilibria in convex stochastic games with a parameter misspecification. They propose two distributed learning schemes: (i) A stochastic gradient scheme designed for stochastic Nash games, which involves two steps of projected gradient updates. They demonstrate that the mean-squared error of the equilibrium estimates converges to zero at the optimal rate of $\mathcal{O}(1/K)$. (ii) An iterative fixed-point scheme developed for stochastic Nash-Cournot games, where the generated iterates are shown to converge almost surely to the Nash-Cournot equilibrium. In \cite{lei2020asynchronous}, the authors focus on computing Nash equilibria in convex stochastic noncooperative games, where the associated potential function may be nonconvex. One class of such games involves studying a misspecified potential stochastic game, where the misspecified parameter is the solution to a stochastic convex optimization problem. To address this, they develop an asynchronous inexact proximal best-response (BR) scheme, where each player updates their equilibrium strategy and their belief about the misspecified parameter using delayed information about their rivals. They propose an asynchronous inexact proximal BR scheme with a stochastic learning algorithm, and it is proven that the iterates generated by this scheme converge almost surely to a connected subset of the Nash equilibrium set. Later, the authors in \cite{huang2022distributed} propose a distributed learning algorithm that guarantees almost-sure convergence to stochastic Nash equilibria (SNEs) in locally coupled network games with unknown parameters. Their approach combines the proximal-point iteration with an ordinary least squares estimator and ensures convergence when the updating step-sizes decay at an appropriate rate.\\
 \textbf{Variational Inequality Problems.} Variational inequality (VI) problems provide a broad framework for modeling optimization and equilibrium problems \cite{facchinei2003finite}. In \cite{ahmadi2020resolution}, the authors investigate a class of misspecified monotone variational inequality problems, considering both general misspecifications and the possibility of a parameter misspecification in the constraint set. To address these challenges, they propose extragradient and regularized first-order schemes specifically designed for misspecified monotone variational inequality problems with strongly convex learning problems. These methods offer a structured approach to handling misspecifications in both the objective function and the constraint set while ensuring reliable convergence properties.
 The authors in \cite{jiang2016solution} \ma{study a class of stochastic variational inequality (SVI) problems as part of their analysis of misspecified problems.} These problems encompass both stochastic convex optimization and various stochastic equilibrium models.
 They establish almost-sure convergence for both strongly monotone and merely monotone regimes, with the latter utilizing iterative Tikhonov regularization. Their proposed schemes achieve optimal convergence rates of $\mathcal{O}(1/K)$ in the strongly monotone setting and $\mathcal{O}(\sqrt{\ln K}/\sqrt{K})$ in the monotone setting, where the latter result relies on the assumption of weak sharpness. \ma{Notably, these results assume that the learning problem is strongly convex, which is crucial for ensuring the convergence of both the learning and optimization iterates.}
\rma{To position our contributions within the literature, Table~\ref{tab:comparison}
contrasts our three results with the most closely related works on misspecified problems in terms of the primary (optimization) and secondary (learning) problems, the misspecified setting, whether the learning problem has a unique (U) or multiple
(M) solutions, whether a backtracking line search is available, and the resulting rate.}
\begin{table}[t]
\centering
\scriptsize
\setlength{\tabcolsep}{3.5pt}
\renewcommand{\arraystretch}{1.25}
\caption{Comparison with related work on misspecified problems in terms of the
\emph{primary problem} (the optimization problem) and the \emph{secondary problem}
(the learning problem); the \emph{misspecified setting}, i.e.\ which component carries
the unknown parameter; the \emph{learning solution} (Learn.\ Sol.), which is unique (U)
or multiple (M); whether a backtracking \emph{line search} (LS) is available; and the
resulting convergence rate in the number of iterations $K$. Abbreviations: min.,
minimization; str., strongly; mon., monotone; s-mon., strongly monotone; VI,
variational inequality; SVI, stochastic VI; SP, saddle point.}
\label{tab:comparison}
\begin{tabularx}{\textwidth}{@{}
  >{\raggedright\arraybackslash\hsize=0.78\hsize}X  
  >{\raggedright\arraybackslash\hsize=1.06\hsize}X  
  >{\raggedright\arraybackslash\hsize=1.10\hsize}X  
  >{\raggedright\arraybackslash\hsize=0.79\hsize}X  
  >{\centering\arraybackslash\hsize=0.52\hsize}X    
  >{\centering\arraybackslash\hsize=0.36\hsize}X    
  >{\raggedright\arraybackslash\hsize=2.40\hsize}X @{}}
\toprule
\textbf{Ref.\ /\ Result} & \textbf{Primary problem} & \textbf{Secondary problem} & \textbf{Misspec.\ setting} & \textbf{Learn.\ Sol.} & \textbf{LS} & \textbf{Rate} \\
\midrule
\cite{ahmadi2014data,ahmadi2020resolution} & Convex min.\ & Convex min.\ & Objective & U & \xmark & $\mathcal O(1/\sqrt K)$ \\
\addlinespace[2pt]
\cite{ahmadi2020resolution} & Mon.\ VI & Str.\ convex min.\ & Operator & U & \xmark & -- \\
\addlinespace[2pt]
\cite{jiang2013solution,jiang2016solution} & Stoch.\ convex min.\ / SVI (mon.) & Stoch.\ str.\ convex min.\ / SVI (s-mon.) & Objective/\ operator & U & \xmark & $\mathcal O(\sqrt{\ln K/K})$; $\mathcal O(1/\sqrt K)$ (modified averaging) \\
\addlinespace[2pt]
\cite{aybat2021analysis} & Conic-constr.\ convex min.\ & Convex min.\ (linear rate) & Objective \& constraints & U & \xmark & $\mathcal O(K^{-1/4})$ (const.\ penalty); $\tilde{\mathcal O}(\log K/K)$ (incr.\ penalty) \\
\addlinespace[2pt]
\cite{yang2024data} & Stoch.\ convex min.\ & Stoch.\ convex min.\ & Objective & M & \xmark & $\mathcal O(1/\sqrt K)$ \\
\midrule
Thm.~\ref{thm:naive-approach} / Cor.~\ref{cor:naive-rate} & Convex--concave SP & Convex--concave SP & Coupling function & U & \xmark & $\mathcal O(\log K/K)$; conv. bound scales with domain diameters \\
\addlinespace[2pt]
Thm.~\ref{thm:missp-approach} / Cor.~\ref{cor:missp-rate} & Convex--concave SP & Convex--concave SP & Coupling function & U & \cmark & $\mathcal O(\log K/K)$; conv. bound independent of domain diameters \\
\addlinespace[2pt]
Thm.~\ref{thm:ext-multsol} / Cor.~\ref{cor:ext-multisol-rate} & Convex--concave SP & Convex min.\ & Coupling function & M & \cmark & $\mathcal O(1/\sqrt K)$ \\
\bottomrule
\end{tabularx}
\end{table}

 \subsection{Contributions}

In this paper, we study a class of misspecified saddle point (SP) problems where the optimization objective depends on an unknown parameter that must be learned simultaneously from data through a lower-level learning problem. Unlike existing works that assume parameters are fully known or pre-estimated, our framework couples optimization and learning into a unified SP formulation, making the problem more general and challenging. To address this setting, we propose two accelerated primal-dual (APD) algorithms. The first, called the \textit{Naive APD}, extends the method of \cite{hamedani2021primal} to the misspecified case by directly substituting the current parameter estimate into the APD updates; however, this approach ignores the dynamics of the evolving parameter, which may lead to instability and weaker performance. To overcome these limitations, we develop the \textit{Learning-aware APD}, which explicitly incorporates the parameter evolution into the dual momentum step and employs an adaptive backtracking strategy for step-size selection. This scheme dynamically adjusts step-sizes using only observable quantities at each iteration, without requiring knowledge of the problem parameter, e.g., Lipschitz constants, thereby ensuring stable updates and improved convergence guarantees. 
Our main theoretical guarantees for the proposed methods are summarized as follows:
\begin{itemize}
    \item Assuming that $\textbf{dom}\; f\times \textbf{dom}\; h$ is compact, the proposed Naive APD method achieves a provable convergence rate of $\mathcal{O}(\log K / K)$ where the $\mathcal O(1)$ constant depends on  
    the diameters of the minimization and maximization domains. 
    \item The proposed Learning-aware APD method also achieves a convergence rate of $\mathcal{O}(\log K / K)$ without requiring the compactness assumption. Moreover, compared with the Naive APD, it has a smaller $\mathcal O(1)$ constant. More specifically, this method addresses two major limitations of the Naive approach: (i) it adaptively adjusts the momentum and dual updates based on the evolving $\theta_k$, and (ii) it introduces a principled backtracking line search enabling adaptive step-size selection.
    \item Considering a structured coupling function in the objective as $\Phi(x,y;\theta)=g_1(x,\theta)+g_2(x,y)$ for some convex-concave and continuously differentiable functions $g_1,g_2$, we extend our framework to a setting in which the learning problem may admit multiple optimal solutions. By adopting a pessimistic formulation, we reformulate the problem as a misspecified SP problem that remains tractable within our framework. We design a modified version of the Learning-aware APD algorithm that incorporates an update for the learning part and preserves the backtracking mechanism. This extension achieves a provable convergence rate of $\mathcal{O}(1/\sqrt{K})$, thereby demonstrating the flexibility and robustness of our approach beyond the unique learning solution setting.
\end{itemize}



\subsection{Motivating examples}\label{subsec:examples}
The misspecified SP problem in \eqref{eq:sp-prob} finds broad applications in areas such as portfolio optimization \cite{aybat2021analysis, Markowitz_1952}, power systems \cite{jiang2016solution}, inventory control \cite{qi2021integrated,ghosh2021new}, and Robust Optimization with Anomaly-aware Learning \cite{ruff2018deep,tax2004support,bertsimas2018data}. Next, we provide a detailed description of two specific examples.

\ma{\subsection{Portfolio Optimization Problem}\label{subsec:portfolio}
Consider the Markowitz portfolio optimization problem discussed in \cite{aybat2021analysis,bertsimas2008robust,goldfarb2003robust,Markowitz_1952}. Suppose the random returns of $n$ financial assets, denoted by $\{\mathcal{R}_i\}_{i=1}^n$. Suppose the joint distribution of the aggregated return $\mathcal{R} = \{\mathcal{R}_i\}_{i=1}^n$ follows a multivariate Normal distribution $N(\mu^0, \Sigma^0)$, where the mean $\mu^0$ is defined as $\mu^0 \triangleq \mathbb{E}[\mathcal{R}]$ and the covariance matrix $\Sigma^0 \triangleq \mathbb{E}[(\mathcal{R} - \mu^0)^\top (\mathcal{R} - \mu^0)]$. Here, $\Sigma^0 = [\sigma_{ij}]_{1 \leq i,j \leq n}$ is assumed to be positive definite, indicating that there are no redundant assets within the set.}

\ma{Let $x_i \in \mathbb{R}$ represent the proportion of asset $i$ maintained in the portfolio during the specified period. Consequently, the portfolio vector $x = [x_i]_{i=1}^n \in \mathbb{R}^n$ satisfies the conditions $\sum_{i=1}^n x_i = 1$ and $x_i \geq 0$ for all $i = 1, \dots, n$, which corresponds to a portfolio without short selling.}
\ma{Assume there are $s$ sectors, with $m_j$ denoting the maximum proportion of the portfolio that can be invested in sector $j$ for $j = 1, \dots, s$. The sector constraints are expressed as $\sum_{i \in I_j} x_i \leq m_j$ for $j = 1, \dots, s$, where $I_j \subset \{1, \dots, n\}$ represents the set of indices corresponding to assets in sector $j$. These constraints can be succinctly written as $Ax \leq b$, where $b \triangleq [m_1, \dots, m_s]^\top$, and $A_{ji} = 1$ if asset $i$ belongs to sector $j$.}

\ma{According to Markowitz's theory, the optimal portfolio is determined by solving the following problem:
\begin{align}\label{eq:portfolio-pr1}
    \min_{x\in \mathbb{R}^n}\left\{\frac{1}{2}x^\top \Sigma^* x - \kappa \mu^\top x: Ax \leq b, x \in \mathcal{X}\right\},
\end{align}
where $\kappa > 0$ is a trade-off parameter balancing the expected portfolio return and risk (measured by variance), $\mu$ and $\Sigma$ represent estimates of $\mu^0$ and $\Sigma^0$, respectively; and $\mathcal{X} \triangleq \{x \in \mathbb{R}^n: \sum_{i=1}^n x_i = 1, x \geq 0\}$.}

\ma{To avoid the computational complexity of solving the constrained optimization problem in \eqref{eq:portfolio-pr1}, we can reformulate it using the Lagrangian method and solve the dual problem as 
\begin{align}\label{eq:portfolio-sp-opt}
   (Optimization): \quad \min_{x\in \mathcal{X}} \max_{y\geq 0} \left\{\frac{1}{2}x^\top \Sigma^* x - \kappa \mu^\top x + y^\top(Ax-b)\right\}.
\end{align}}

\ma{In this problem, we assume that the vector $\mu = \mu^0$ is known, but the covariance matrix $\Sigma$ is unknown or misspecified. Consequently, the problem in \eqref{eq:portfolio-sp-opt} can be regarded as a misspecified SP problem. For the \textit{(Learning)} problem, given sample returns for $n$ assets with a sample size $p$ for each asset, let $S = (s_{ij})_{1 \leq i,j \leq n} \in \mathbb{R}^{n \times n}$. Since, in practice, the number of assets is much larger than the sample size, $S$ cannot be positive definite, which contrasts with $\Sigma^0 \succ \mathbf{0} \in \mathbb{R}^{n \times n}$. Therefore, instead of using $S$ as our true covariance estimator, we consider the sparse covariance selection (SCS) problem proposed in \cite{xue2012positive} as follows:
\begin{align*}
    \Sigma^* \triangleq \argmin_{\Sigma \in \mathbb{S}^n}\left\{\frac{1}{2}\|\Sigma - S\|^2_F + v\|\Sigma\|_1 : \Sigma \succeq \epsilon I \right\},
\end{align*}
where $v$ and $\epsilon$ are positive regularization parameters, $\mathbb{S}^n$ denotes the set of $n \times n$ symmetric matrices, $\|\cdot\|_F$ represents the Frobenius norm, $\|\Sigma\|_1$ is the $\ell_1$ norm of the vector composed of all off-diagonal elements of $\Sigma$, and $\Sigma \succeq \epsilon I$ means that all eigenvalues of $\Sigma$ are greater than or equal to $\epsilon > 0$. Using a symmetric Lagrange multiplier matrix $W \in \mathbb{S}^n$, where $W\succeq 0$ denotes that $W$ is positive semi-definite, the constraint $\Sigma \succeq \epsilon I$ can be incorporated into the objective function and form an SP problem as 
\begin{align*}
    (Learning): \quad \Sigma^* \in \argmin_{\Sigma \in \mathbb{S}^n}\max_{W\in \mathbb{S}^n_+}\left\{\frac{1}{2}\|\Sigma - S\|^2_F + v\|\Sigma\|_1 - Tr\big( W^\top( \Sigma -\epsilon I)\big)  \right\}.
\end{align*}
Indeed, one can easily verify that in this setting the \textit{(Optimization)} and \textit{(Learning)} problems form a misspecified SP problem, which is a special case of our main formulation in Problem~\eqref{eq:sp-prob}.}

\subsection{Robust Optimization with Anomaly-aware Learning}
We consider a robust decision-making problem under uncertainty, where the objective is to optimize a decision variable $x \in \mathcal{X} \subseteq \mathbb{R}^d$ in the presence of an uncertain parameter $u \in \mathbb{R}^d$. The cost function $f(x, u)$ depends on this uncertain parameter, which is not known a priori but is partially observed through historical data. A canonical example of such a problem is:
\begin{equation*}
    \min_{x \in \mathcal{X}} \max_{u \in U} f(x, u),
\end{equation*}
where $U \subset \mathbb{R}^d$ is an uncertainty set that captures plausible realizations of $u$. A commonly used form of $U$ is the Euclidean ball
$U = \{ u \in \mathbb{R}^d : \|u - c\| \leq r \}$,
where $c$ and $r$ denote the center and radius of the uncertainty set, respectively. For instance, in linear settings such as $f(x, u) = u^\top x$, the worst-case realization of $u$ lies on the boundary of this set and significantly affects the robust solution.

However, both $c$ and $r$ may be unknown or inaccurately specified due to noise or anomalies in the historical data. Let $\{ u_i \}_{i=1}^n$ be a set of observed samples drawn from an unknown data-generating process. These samples may include outliers or corrupted data points due to sensor errors, adversarial manipulations, or rare events. Naively using all samples to construct the uncertainty set could severely degrade the robustness of the solution. 
To address this challenge, we adopt an \emph{anomaly detection} approach to learn a reliable uncertainty set from the data. In particular, we employ the Deep Support Vector Data Description (Deep SVDD) method \cite{ruff2018deep,tax2004support}. Deep SVDD is a kernel-based or neural network-based anomaly detection framework that aims to learn a mapping $\phi(u)$ such that the majority of the transformed data points lie within a minimal-volume hypersphere in feature space. Formally, the Deep SVDD optimization problem is:
\begin{align*}
    \min_{c, r, \{\xi_i\}} \quad r^2 + \frac{1}{2n} \sum_{i=1}^n \xi_i^2 \quad \text{subject to} \quad \| \phi_k(u_i) - c \|^2 \leq r^2 + \xi_i, \quad \xi_i \geq 0, \quad \forall i,
\end{align*}
where $\phi_k: \mathbb{R}^d \to \mathcal{F}_k$ maps data to a reproducing kernel Hilbert space (RKHS) associated with a positive semi-definite kernel $k(\cdot, \cdot)$, and $\langle \cdot, \cdot \rangle_{\mathcal{F}_k}$ denotes the inner product in this space.

The parameters $\theta = (c, r)$ define the center and radius of the learned uncertainty set and are computed from the solution of this anomaly detection problem. Note that by solving this optimization, we effectively remove the influence of anomalous samples by inflating their slack variables $\xi_i$, preventing them from overly influencing $c$ and $r$.

With the uncertainty set learned from the Deep SVDD procedure, we now couple the anomaly detection and robust optimization into a single problem. Specifically, the decision-maker solves a misspecified  SP problem:
\begin{equation*}
\begin{aligned}
    &\textit{(Optimization)}: \min_{x \in \mathcal{X}, \lambda \geq 0} \max_{u \in \mathbb{R}^d} f(x, u) - \lambda (\|u - c\|^2 - r^2),\\
    & \textit{(Learning)}: \quad(c, r) \in \argmin_{c, r, \{\xi_i\}} \left\{ r^2 + \frac{1}{2n} \sum_{i=1}^n \xi_i^2 : \| \phi_k(u_i) - c \|^2 \leq r^2 + \xi_i, \; \xi_i \geq 0, \; \forall i \right\}.
\end{aligned}
\end{equation*}
This formulation unifies robust optimization and anomaly detection, which can be viewed as a special case of \eqref{eq:sp-prob}. 

\section{Preliminaries}\label{sec:pre}
In this section, we state the assumptions and introduce the notions used in the paper. 

\subsection{Assumptions and Definitions}

\begin{definition}\label{def:Breg-dis}
    Let $ \varphi_{\mathcal{X}}: \mathcal{X} \to \mathbb{R} $ and $ \varphi_{\mathcal{Y}}: \mathcal{Y} \to \mathbb{R} $ be differentiable functions defined on open sets containing $\mathbf{dom} \, f$ and $\mathbf{dom} \, h$, respectively. Assume that $ \varphi_{\mathcal{X}} $ and $ \varphi_{\mathcal{Y}} $ have closed domains and are 1-strongly convex with respect to the norms $ \|\cdot\|_{\mathcal{X}} $ and $ \|\cdot\|_{\mathcal{Y}} $, respectively. Define the Bregman distance functions associated with $ \varphi_{\mathcal{X}} $ and $ \varphi_{\mathcal{Y}} $ as $ \mathbf{D}_{\mathcal{X}} : \mathcal{X} \times \mathcal{X} \to \mathbb{R}_+ $ and $ \mathbf D_{\mathcal{Y}} : \mathcal{Y} \times \mathcal{Y} \to \mathbb{R}_+ $, respectively. Specifically, the Bregman distance for $ \mathcal{X} $ is given by  
\begin{align*}
\mathbf D_{\mathcal{X}}(x, \bar{x}) = \varphi_{\mathcal{X}}(x) - \varphi_{\mathcal{X}}(\bar{x}) - \langle \nabla \varphi_{\mathcal{X}}(\bar{x}), x - \bar{x} \rangle,
\end{align*}
with $ \mathbf D_{\mathcal{Y}}(y, \bar{y}) $ defined analogously. Moreover, it follows that  
\begin{align*}
\mathbf D_{\mathcal{X}}(x, \bar{x}) \geq \frac{1}{2} \| x - \bar{x} \|_{\mathcal{X}}^2, \quad \forall x \in \mathcal{X}, \; \bar{x} \in \mathbf{dom} \, f,
\end{align*}
and similarly,  
\begin{align*}
\mathbf D_{\mathcal{Y}}(y, \bar{y}) \geq \frac{1}{2} \| y - \bar{y} \|_{\mathcal{Y}}^2, \quad \forall y \in \mathcal{Y}, \; \bar{y} \in \mathbf{dom} \, h.
\end{align*}
Moreover, the dual spaces are denoted as $\mathcal{X}^*$ and $\mathcal{Y}^*$. For $x' \in \mathcal{X}^*$, the dual norm is defined as
\begin{align*}
    \|x'\|_{\mathcal{X}^*}\triangleq \max\{\langle x',x\rangle:\|x\|_{\mathcal{X}}\leq 1 \}.
\end{align*} 
Similarly, the dual norm $\|\cdot\|_{\mathcal{Y}^*}$ is defined for $\mathcal{Y}^*$ in the same manner.
Moreover, analogously, we define Bregman distances $\mathbf{D}_{\Theta}$ and $\mathbf{D}_{\mathcal{W}}$ using differentiable and 1-strongly convex reference functions $\varphi_{\Theta}$ and $\varphi_{\mathcal{W}}$ over the domains $\Theta$ and $\mathcal{W}$, respectively.
\end{definition}

Next, to adaptively estimate local Lipschitz constants in our proposed method, we integrate a backtracking strategy. To verify if the step-sizes selected at each iteration $k \geq 0$ align with the local Lipschitz constants, we introduce a test function $E_k(\cdot,\cdot;\cdot)$ that utilizes the linearization of $\Phi$ with respect to both $x$ and $y$.
\begin{definition}\label{def:E_k}
Given a free parameter sequence $\{ \alpha_k, \beta_k \}_{k \geq 0} \subseteq \mathbb{R}_+$, we define
    \begin{align*}
        E_k(x,y;\theta) &\triangleq \langle \nabla_x\Phi(x,y;\theta)-\nabla_x\Phi(x_k,y;\theta),x-x_k\rangle + \frac{1}{2\alpha_{k+1}}\| \nabla_y \Phi(x,y;\theta) - \nabla_y \Phi(x_k,y;\theta)\|^2_{\mathcal{Y}^*}\\
        & \quad  - \Big(\frac{1}{\sigma_k}- \eta_k(\alpha_k + \beta_k) \Big)\mathbf D_{\mathcal{Y}}(y,y_k)+ \frac{1}{\beta_{k+1}}\| \nabla_y \Phi(x_k,y;\theta_k) - \nabla_y \Phi(x_k,y_k;\theta_k)\|^2_{\mathcal{Y}^*}\\
        & \quad - \frac{1}{\tau_k} \mathbf D_{\mathcal{X}}(x,x_k).
    \end{align*} 
 \noindent
\end{definition}
Later we show that for given $\alpha_k, \beta_k \geq 0$ and $\eta_k$, the inequality $E_k(x, y; \theta) \leq 0$ lead to selecting the step-sizes $\{ \tau_k, \sigma_k \}$ to ensure that the condition test in line 9 of Algorithm~\ref{alg:mis-sp} is satisfied within the Learning-aware framework. Next, we define the gap function to measure the optimality of the solution generated by the proposed methods.

\begin{definition}\label{def:gap-function}
    \ma{Let $\bar{Z} \subseteq \textbf{dom}\, f \times \textbf{dom}\, h$ 
    be a compact and convex set containing a saddle point solution of \eqref{eq:sp-prob}, and $z \triangleq [x^\top, y^\top]^\top$. 
    The gap function for Problem~\eqref{eq:sp-prob} is defined as}
    \begin{align*}
        \mathcal{G}_{\bar{Z}}(\bar{z}_K) 
        \triangleq \sup_{z \in \bar{Z}} 
        \Big\{ \mathcal{L}(\bar{x}_K, y; \theta^*) - \mathcal{L}(x, \bar{y}_K; \theta^*) \Big\},
    \end{align*}
\ma{where $\bar{z}_K \triangleq [\bar{x}_K^\top,\bar{y}_K^\top]^\top$ denotes the weighted average of the iterates $\{z_{k}\}_{k=0}^{K-1}$.}    

\ma{The supremum is finite, and as shown in \cite[Lemma~1]{nesterov2007dual}, this construction defines a valid gap function for saddle point problems.}

\end{definition}

\begin{assumption}\label{assump:phi-lip}

The functions $f$ and $h$ are closed and convex, and $\Phi$ is continuous and satisfies the following conditions:

\begin{itemize}
    \item[(i)] For any $y \in \textbf{dom} \, h \subset \mathcal{Y}$ and $\theta \in \Theta$, the function $\Phi(\cdot, y; \theta)$ is convex and differentiable. Moreover, there exists constants $L^{\Phi}_{xx} \geq 0$ and $L^{\Phi}_{x\theta}> 0$ such that for all $x, \bar{x} \in \textbf{dom} \, f$ and $\theta, \bar{\theta} \in \Theta$, the following inequality holds:  
    \begin{align*}
    \|\nabla_x \Phi(x,y;\theta) - \nabla_x \Phi(\bar{x},y;\bar\theta)\|_{\mathcal{X}^*} \leq L^{\Phi}_{xx} \|x - \bar{x}\|_{\mathcal{X}} + L^{\Phi}_{x\theta} \|\theta - \bar{\theta}\|_{\Theta}. 
    \end{align*}
    
    \item[(ii)] For any $x \in \textbf{dom} \, f \subset \mathcal{X}$ and $\theta \in \Theta$, the function $\Phi(x, \cdot; \theta)$ is concave and differentiable. Additionally, there exist constants $L^{\Phi}_{yx} > 0$, $L^{\Phi}_{yy} \geq 0$, and $L^{\Phi}_{y\theta}> 0$ such that for all $x, \bar{x} \in \textbf{dom} \, f$ and $y, \bar{y} \in \textbf{dom} \, h$ and $\theta, \bar{\theta} \in \Theta$, the following inequality holds:  
    \begin{align*}
    \|\nabla_y \Phi(x,y;\theta) - \nabla_y \Phi(\bar{x},\bar{y};\bar \theta)\|_{\mathcal{Y}^*} \leq L^{\Phi}_{yx} \|x - \bar{x}\|_{\mathcal{X}}+ L^{\Phi}_{yy} \|y - \bar{y}\|_{\mathcal{Y}} + L^{\Phi}_{y\theta} \|\theta - \bar{\theta}\|_\Theta.
    \end{align*}
\end{itemize}
\end{assumption}

\ma{Before presenting the step-size assumptions, we clarify that these conditions are required to establish the convergence guarantees of the proposed algorithms. Later, we will provide explicit parameter selections that satisfy these conditions.}
\begin{assumption}[step-size condition I]\label{assump:step-size-cond}
    There exists a sequence $\{\tau_k,\sigma_k,\eta_k\}_{k\geq 0}$ such that the iterates $\{(x_k,y_k)\}_{k \geq 0}$, generated by Algorithm \ref{alg:mis-sp}, along with the step-size sequence, satisfy the following conditions for all $k \geq 0$:
    \begin{itemize}
        \item [(i)] $E_k(x_{k+1},y_{k+1};\theta_{k+1}) \leq 0$,
        \item [(ii)] $\frac{t_k}{\tau_k} \geq \frac{t_{k+1}}{\tau_{k+1}}$,\quad $\frac{t_k}{\sigma_k} \geq \frac{t_{k+1}}{\sigma_{k+1}}$,\quad $\frac{t_k}{t_{k+1}}= \eta_{k+1}$,
    \end{itemize}
    for some positive sequences $\{t_k, \alpha_k\}_{k \geq 0}$ with $t_0 = 1$, and a nonnegative sequence $\{\beta_k\}_{k\geq 0}$, where $E_k(\cdot,\cdot;\cdot)$ is defined in Definition \ref{def:E_k} using the parameters $\{\alpha_k,\beta_k,\eta_k\}$ as specified above. 
\end{assumption}
\begin{assumption}[step-size condition II]\label{assump:step-size-cond2}
    For any $k \geq 0$, the step-sizes $\tau_k$ and $\sigma_k$, along with the momentum parameter $\eta_k\geq 0$, satisfy $\eta_0 = 1$ as well as condition in Assumption \ref{assump:step-size-cond}-(i), and
    \begin{itemize}
        \item [(i)] $\frac{1}{\tau_k}\geq \frac{(L^{\Phi}_{yx})^2}{\alpha_{k+1}}+L^{\Phi}_{xx}$,
        \item [(ii)] $\frac{1}{\sigma_k} \geq \eta_k (\alpha_k + \beta_k) + \frac{2(L^{\Phi}_{yy})^2}{\beta_{k+1}}$,
        
    \end{itemize}
    for some positive sequences $\{t_k, \alpha_k\}_{k \geq 0}$ with $t_0 = 1$, and a nonnegative sequence $\{\beta_k\}_{k\geq 0}$.
\end{assumption}

\begin{assumption}[Structure of the Learning Problem]\label{assump:learn-structure}
We assume that the learning objective $l(\theta, w) = f_L(\theta) + \ell(\theta, w) - h_L(w)$ satisfies the following conditions:
\begin{itemize}
    \item[(i)] $f_L: \Theta \rightarrow \mathbb{R} \cup \{+\infty\}$ and $h_L: \mathcal{W} \rightarrow \mathbb{R} \cup \{+\infty\}$ are closed, convex functions (possibly nonsmooth).
    \item[(ii)] $\ell: \Theta \times \mathcal{W} \rightarrow \mathbb{R}$ is continuously differentiable, strongly convex in $\theta$ with modulus $\mu' > 0$, and linear in $w$.
\end{itemize}
\end{assumption}



\section{Proposed Methods}\label{sec:proposed-method}
In this section, we propose two accelerated primal-dual (APD) algorithms with a momentum term, building upon the framework introduced in \cite{hamedani2021primal}, to solve the misspecified saddle point (SP) problem defined in \eqref{eq:sp-prob}. The SP formulation involves a coupling function $\Phi(x, y; \theta)$ that depends on an unknown parameter $\theta^*$, which is not available a priori but must be learned via a simultaneous learning process. This introduces a structural misspecification into the optimization component of the problem.
To address this challenge, we consider two algorithmic strategies for integrating and updating the parameter estimate $\theta_k$ within the APD framework:

\begin{itemize}
    \item \textbf{Naive Approach:} This method, described in Section~\ref{subsec:naive}, directly substitutes the current estimate $\theta_k$ into the primal-dual updates at each iteration. However, it does not account for the evolution of $\theta_k$ in the momentum step, nor does it support backtracking since the backtracking test function depends on the unknown ground-truth parameter $\theta^*$.
    
    \item \textbf{Learning-aware Approach:} This method, detailed in Section~\ref{subsec:misspecified}, improves upon the Naive strategy by explicitly incorporating the temporal variation of $\theta_k$ into the acceleration dynamics. Moreover, it introduces a principled backtracking line search based on a surrogate error function that only depends on current iterates and $\theta_k$, enabling adaptive step-size selection without requiring access to $\theta^*$.
\end{itemize}

Together, these methods extend the work of \cite{hamedani2021primal} to a misspecified setting and offer a comparative perspective on how parameter estimation and its integration into the optimization process affect convergence behavior. In what follows, we formally describe each method and analyze its theoretical guarantees.

\subsection{Naive Approach: Inexact Accelerated Primal-dual Method}\label{subsec:naive}
Building on the accelerated primal-dual (APD) algorithm of \cite{hamedani2021primal}, we extend their framework to address the misspecified setting where the coupling term $\Phi(x,y;\theta)$ depends on an unknown parameter $\theta^*$. In our setting, this parameter must be estimated via a simultaneous learning process. The resulting procedure is outlined in Algorithm~\ref{alg:naive-mis-sp} and is referred to as the \textit{Naive} approach.

We use the term \textit{Naive} because this algorithm treats the evolving parameter $\theta_k$ in a simplified, non-adaptive manner. Specifically, the parameter estimate $\theta_k$ is obtained by a separate one-step APD learning update (see Algorithm~\ref{alg:APD-learning}, called in line 6 of Algorithm~\ref{alg:naive-mis-sp}), and then directly substituted into the primal and dual updates. Once inserted, $\theta_k$ is held fixed throughout that iteration’s optimization steps, as if it were the true value $\theta^*$.

However, in this Naive approach, we do not modify the acceleration steps to account for the temporal variation of $\theta$; the estimated value $\theta_k$ is treated as fixed within each primal-dual update. For example, in line 3 of Algorithm~\ref{alg:naive-mis-sp}, we have:
\begin{align*}
s_k = (1+\eta_k)\nabla_y \Phi(x_k, y_k; \theta_k) - \eta_k \nabla_y \Phi(x_{k-1}, y_{k-1}; \theta_k),
\end{align*}
which relies entirely on the current estimate $\theta_k$ and does not incorporate any information about the preceding value $\theta_{k-1}$. Thus, the momentum term is blind to the dynamics of the evolving parameter, making the design structurally naive with respect to the misspecification. This differs from the Learning-aware accelerated primal-dual algorithm we propose in the next section, where the change in $\theta$ is explicitly captured in the update rule.

Consequently, the effect of $\theta_k$'s evolution 
appears implicitly as a residual approximation error introduced by using $\theta_k$ instead of $\theta^*$. This residual is later aggregated and bounded in our convergence analysis using Lipschitz constants, including $L_{x\theta}^{\Phi}$ and $L_{y\theta}^{\Phi}$. 
Despite this limitation, the Naive method remains theoretically sound and achieves a provable convergence rate, making it a baseline against which we compare our Learning-aware accelerated primal-dual method.

\begin{algorithm}
\caption{Naive accelerated primal-dual algorithm for misspecified saddle point problem}\label{alg:naive-mis-sp}
\begin{algorithmic}[1]
\Require $\tau_0, \sigma_0, \eta_0 > 0$, $(x_0,y_0) \in \mathcal{X} \times \mathcal{Y}$, $\theta_0 \in \mathbb{R}^m$
\State Initialize: $(x_{-1}, y_{-1}; \theta_{-1}) \gets (x_0, y_0; \theta_0)$, $\sigma_{-1} \gets \sigma_0$
\For{$k = 0$ to $K-1$}
    \State $s_k \gets (1+\eta_k) \nabla_y \Phi(x_k, y_k; \theta_k) - \eta_k \nabla_y \Phi(x_{k-1}, y_{k-1}; \theta_k)$
    \State $y_{k+1} \gets \argmin_{y \in \mathcal{Y}} \left\{ h(y) - \langle s_k, y \rangle + \frac{1}{\sigma_k} \mathbf{D}_{\mathcal{Y}}(y, y_k) \right\}$
    \State $x_{k+1} \gets \argmin_{x \in \mathcal{X}} \left\{ f(x) + \langle \nabla_x \Phi(x_k, y_{k+1}; \theta_k), x \rangle + \frac{1}{\tau_k} \mathbf{D}_{\mathcal{X}}(x, x_k) \right\}$
    \State $\theta_{k+1} \gets \mathcal{M}(\theta_k, w_k)$
\EndFor
\end{algorithmic}
\end{algorithm}

\begin{algorithm}
\caption{$\mathcal{M}(\theta,w)$: One-step accelerated primal-dual algorithm with backtracking for learning problem}\label{alg:APD-learning}
\begin{algorithmic}[1]
\Require $\mu' > 0$, $\bar\tau' > 0$, $\gamma' > 0$, $\rho' \in (0,1)$, $(\theta_0, w_0) \in \Theta \times \mathcal{W}$
\State Initialize: $(\theta_{-1}, w_{-1}) \gets (\theta_0, w_0)$, $\tau'_0 \gets \bar\tau'$, $\sigma'_{-1} \gets \gamma' \tau'_0$
\While{stopping criterion not satisfied}
    \State $\sigma'_k \gets \gamma' \tau'_k$, $\eta'_k \gets \sigma'_{k-1} / \sigma'_k$, $\alpha'_{k+1} \gets 1/\sigma'_k$
    \State $s'_k \gets (1+\eta'_k) \nabla_w \ell(\theta_k, w_k) - \eta'_k \nabla_w \ell(\theta_{k-1}, w_{k-1})$
    \State $w_{k+1} \gets \argmin_{w \in \mathcal{W}} \left\{ h_L(w) - \langle s'_k, w \rangle + \frac{1}{\sigma'_k} \mathbf{D}_{\mathcal{W}}(w, w_k) \right\}$
    \State $\theta_{k+1} \gets \argmin_{\theta \in \Theta} \left\{ f_L(\theta) + \langle \nabla_{\theta} \ell(\theta_k, w_{k+1}), \theta \rangle + \frac{1}{\tau'_k} \mathbf{D}_{\Theta}(\theta, \theta_k) \right\}$
    \If{ 
    \(\langle \nabla_{\theta} \ell(\theta_{k+1}, w_{k+1}) - \nabla_{\theta} \ell(\theta_k, w_{k+1}), \theta_{k+1} - \theta_k \rangle 
    - \frac{1}{\tau'_k} \|\theta_{k+1} - \theta_k\|^2_{\Theta} + \frac{\sigma'_k}{2} \|\nabla_w \ell(\theta_{k+1}, w_{k+1}) - \nabla_w \ell(\theta_k, w_{k+1})\|^2_{\mathcal{W}^*} \leq 0\)}
        \State \textbf{go to line 14}
    \Else
        \State $\tau'_k \gets \rho' \tau'_k$
        \State \textbf{go to line 3}
    \EndIf
\EndWhile
\State $\tau'_{k+1} \gets \tau'_k$, $k \gets k+1$
\end{algorithmic}
\end{algorithm}

\begin{remark}[Backtracking Step-Size Scheme for Learning]\label{rem:learn-stepsizes}
In Algorithm~\ref{alg:APD-learning}, the step-size parameters $\tau'_k$ and $\sigma'_k$ are dynamically adjusted through a backtracking mechanism governed by parameters $\rho' \in (0,1)$ and $\gamma'>0$.
This update scheme ensures the compatibility of the primal and dual updates in the accelerated framework and is inspired by the adaptive primal-dual methods in \cite{hamedani2021primal}. The test condition in lines 7–11 of Algorithm~\ref{alg:APD-learning} ensures that the local smoothness of $\ell$ is respected.
\end{remark}

Next, we analyze the iteration complexity of the Naive approach in Theorem \ref{thm:naive-approach}.

\begin{theorem}\label{thm:naive-approach}
    Suppose Assumption \ref{assump:phi-lip}  holds. {Let $z \triangleq [x^\top, y^\top]^\top$.} If $\{x_k,y_k,\theta_k\}_{k\geq 0}$ is generated by Algorithm \ref{alg:naive-mis-sp}, using a constant parameter sequence $\{\tau_k,\sigma_k,\eta_k\}_{k\geq 0}$ satisfying Assumption \ref{assump:step-size-cond2}. 
    Then for any convex compact set $\bar{Z} \subseteq \textbf{dom}\, f \times \textbf{dom}\, h$ containing a saddle point solution of \eqref{eq:sp-prob} and 
    $K\geq 1$, 
    \begin{align}\label{eq:naive-sup}
    \mathcal{G}_{\bar{Z}}(\bar{z}_K)&\leq \frac{1}{T_K} \sup_{z\in \bar{Z}} \Big(\frac{1}{\tau_0}\mathbf D_{\mathcal{X}}(x,x_0) + \frac{1}{\sigma_0}\mathbf D_{\mathcal{Y}}(y,y_0) \Big) \nonumber\\
    & \quad + \frac{1}{T_K} \sup_{z\in \bar{Z}} \Big(L_{y\theta}^{\Phi}\sum_{k=0}^{K-1} t_k (2\eta_k+1)\| \theta_k - \theta^*\|_{\Theta} \| y_{k+1}-y\|_{\mathcal{Y}} \nonumber\\
    & \quad + L_{x\theta}^{\Phi}\sum_{k=0}^{K-1} t_k \|\theta_k - \theta^* \|_{\Theta} \|x_{k+1}-x \|_{\mathcal{X}} \Big),
\end{align}
\ma{where $\bar{z}_K \triangleq [\bar{x}_K^\top,\bar{y}_K^\top]^\top$ denotes the weighted average of the iterates 
$\{z_k\}_{k=0}^{K-1}$, i.e., $\bar{z}_K \triangleq \frac{1}{T_K} \sum_{k=0}^{K-1} t_k z_{k+1}$, and} $T_K = \sum_{k=0}^{K-1}t_k$.
\end{theorem}
\begin{proof}
    See Appendix \ref{AppndxB} for the proof.
\end{proof}


\begin{remark}\label{rem:naive-stepsize} 
\ma{To satisfy Assumption~\ref{assump:step-size-cond2} with constant step-sizes, we fix $\eta_k = \eta =1$, and select constants $\alpha > 0$ and $\beta > 0$. Then the conditions are met by choosing $\tau = \left( \frac{(L_{yx}^{\Phi})^2}{\alpha} + L_{xx}^{\Phi} \right)^{-1}$ and $\sigma = \left( \alpha + \beta + \frac{2(L_{yy}^{\Phi})^2}{\beta} \right)^{-1}$, where $L^{\Phi}_{xx}, L^{\Phi}_{yx}, L^{\Phi}_{yy}$ are the Lipschitz constants.}
\end{remark}

\begin{remark}[Convergence of Learning Sequence]\label{rem:theta-rate}
Considering Assumption \ref{assump:learn-structure}, the accelerated primal-dual algorithm by \cite{hamedani2021primal} presented in Algorithm~\ref{alg:APD-learning} for solving the learning problem in \eqref{eq:sp-prob} satisfy
\begin{align*}
\|\theta_{k+1} - \theta^*\|^2_{\Theta} \leq \frac{C_{\theta}}{(k+1)^2}, 
\end{align*}
for some constant $C_{\theta} > 0$, where $\theta^*$ denotes the unique minimizer of the \textit{(Learning)} problem. 
\end{remark}

\begin{corollary}\label{cor:naive-rate}
Suppose Assumption \ref{assump:learn-structure} holds, $\textbf{dom}\; f$ and $\textbf{dom}\; h$ are compact, and let $\bar Z=\textbf{dom}\; f\times \textbf{dom}\; h$. 
Under the premises of Theorem~\ref{thm:naive-approach} with the step-sizes selected as in Remark \ref{rem:naive-stepsize}, for any $K \geq 1$, it holds that
\begin{align*}
\mathcal{G}_{\bar{Z}}(\bar{z}_K) \leq \frac{1}{K}\left(\frac{1}{\tau_0}\mathbf R_{x}+\frac{1}{\sigma_0}\mathbf R_{y} \right) + \frac{\log(K)}{K} \left(L_{x\theta}^{\Phi} \sqrt{\mathbf R_{x}} + 3L_{y\theta}^{\Phi}\sqrt{\mathbf R_{y}} \right)\sqrt{C_{\theta}},
\end{align*}
where $\mathbf R_{ x} \;\triangleq\;
\sup_{x,\bar x \in \textbf{dom}\; f} \mathbf D_{\mathcal X}(x,\bar x)<\infty$ and
$\mathbf R_{y} \;\triangleq\;
\sup_{y,\bar y\in \textbf{dom}\; h} \mathbf D_{\mathcal Y}(y,\bar y)<\infty$ are the Bregman diameters.
\end{corollary}
\begin{proof}
From Theorem~\ref{thm:naive-approach}, we have
\begin{align}\label{eq:repeat-thm}
\mathcal{G}_{\bar{Z}}(\bar{z}_K) &\leq \frac{1}{T_K} \sup_{z\in \bar{Z}} \left( \frac{1}{\tau_0} \mathbf D_{\mathcal{X}}(x,x_0) + \frac{1}{\sigma_0} \mathbf D_{\mathcal{Y}}(y,y_0) \right) \nonumber\\
&\quad + \frac{1}{T_K} \sup_{z\in \bar{Z}} \Bigg( L_{y\theta}^{\Phi} \sum_{k=0}^{K-1} t_k (2\eta_k+1)\|\theta_k - \theta^*\|_{\Theta} \|y_{k+1}-y\|_{\mathcal{Y}}\nonumber \\
& \quad + L_{x\theta}^{\Phi} \sum_{k=0}^{K-1} t_k \|\theta_k - \theta^*\|_{\Theta} \|x_{k+1}-x\|_{\mathcal{X}} \Bigg).
\end{align}
Note that in this setting, $t_k=1$, hence, $T_K = K$. Moreover, from the compactness of domains, we have that for any $(x,y)\in \bar Z$, 
$\|x_{k+1} - x\|_{\mathcal X}
\;\le\; \sqrt{2 \mathbf R_{x}}$ and $
\|y_{k+1} - y\|_{\mathcal Y}
\;\le\; \sqrt{2 \mathbf R_{y}}$.  
Therefore,
\begin{align*}
&\sum_{k=0}^{K-1} (2\eta_k+1)\|\theta_k - \theta^*\|_{\Theta} \|y_{k+1} - y\|_{\mathcal{Y}} \leq (2\eta + 1)\sqrt{2\mathbf R_{y}} \cdot \sqrt{C_{\theta}} \sum_{k=1}^K \frac{1}{k}, \\
&\sum_{k=0}^{K-1} \|\theta_k - \theta^*\|_{\Theta} \|x_{k+1} - x\|_{\mathcal{X}} \leq \sqrt{\mathbf R_{x}} \cdot \sqrt{C_{\theta}} \sum_{k=1}^K \frac{1}{k}.
\end{align*}
Combining the above inequalities with \eqref{eq:repeat-thm} leads to 
\begin{align*}
\mathcal{G}_{\bar{Z}}(\bar{z}_K)\leq \frac{1}{K}\left(\frac{1}{\tau_0}\mathbf R_{x}+\frac{1}{\sigma_0}\mathbf R_{y} \right) + \frac{\log(K)}{K} \cdot \left(L_{x\theta}^{\Phi} \sqrt{\mathbf R_{x}} + L_{y\theta}^{\Phi}(2\eta+1)\sqrt{\mathbf R_{y}} \right)\sqrt{C_{\theta}}. 
\end{align*}
Finally, using Remark~\ref{rem:naive-stepsize} to set $\eta = 1$ completes the proof.
\end{proof}


\subsection{Learning-aware Accelerated Primal-dual Method}\label{subsec:misspecified}
In this section, we propose a refined accelerated primal-dual algorithm for solving the misspecified saddle point (SP) problem in \eqref{eq:sp-prob}. 
In particular, one of the key innovations of this method lies in the way the parameter update $\theta_k$ is embedded in the algorithm. In contrast to the Naive method, where $\theta_k$ is kept fixed within each iteration, here $\theta_k$ is updated \textit{before} the primal variable update (line 7 of Algorithm~\ref{alg:mis-sp}) using the learning subroutine $\mathcal{M}$, and then injected into the subsequent primal update (line 8). More importantly, $\theta_k$ also influences the momentum term used in the dual update (line 5), where the extrapolated gradient takes the form:
\begin{align*}
s_k = (1 + \eta_k)\nabla_y \Phi(x_k, y_k; \theta_k) - \eta_k \nabla_y \Phi(x_{k-1}, y_{k-1}; \theta_{k-1}),
\end{align*}
which now explicitly includes the temporal variation of $\theta$. This enables the dual variable $y$ to adaptively track the changing dynamics introduced by the learning process. 
We refer to this approach as a \textit{Learning-aware} because it directly captures the effect of parameter misspecification in the acceleration step.

Another significant improvement over the Naive method is the development of a \textit{backtracking line search} for adaptive step-size selection. This is made possible by introducing the test function $E_k(x,y;\theta)$, defined in Definition~\ref{def:E_k}, which uses the generated iterate $\theta_{k+1}$ without requiring the knowledge of the unknown ground-truth $\theta^*$. This allows the algorithm to verify local smoothness conditions dynamically and adjust step-sizes $(\tau_k, \sigma_k)$ at each iteration (lines 9--13). Such a mechanism was infeasible in the Naive method, as the construction of such a test function would depend on $\theta^*$. 


\begin{assumption}\label{assump:lip-phi-theta}
    For any given $x\in\mathcal X$ and $y\in\mathcal Y$, function $\Phi(x,y;\cdot)$ is  Lipschitz continuous with constant $L_{\theta}\geq 0$; that is $\left|\Phi(x,y;\theta_1) - \Phi(x,y;\theta_2)\right| \leq L_{\theta} \|\theta_1 - \theta_2\|_{\Theta}$ for any $\theta_1,\theta_2 \in \Theta$. 
\end{assumption}

\begin{algorithm}
\caption{Learning-aware accelerated primal-dual algorithm with backtracking for misspecified saddle point problem}\label{alg:mis-sp}
\begin{algorithmic}[1]
\Require $c_{\alpha}, c_{\beta} \geq0$, $\rho \in (0,1)$, $\bar \tau, \gamma>0$, $(x_0,y_0) \in \mathcal{X} \times \mathcal{Y}$, $\theta_0 \in \mathbb R^m$
\State Initialize: $(x_{-1},y_{-1};\theta_{-1}) \leftarrow (x_0,y_0;\theta_0)$, $\tau_0 \leftarrow \bar \tau$, $\sigma_{-1} \leftarrow \gamma \tau_0$ 
\For{$k\geq0$}
\While{stopping criterion not satisfied}
    \State $\sigma_k \gets \gamma \tau_k$, $\eta_k \gets \frac{\sigma_{k-1}}{\sigma_k}$, $\alpha_{k+1}\gets c_{\alpha}/\sigma_k$, $\beta_{k+1}\gets c_{\beta}/\sigma_k$
    \State $s_k \gets (1+\eta_k)\nabla_y \Phi(x_k,y_k;\theta_k) -\eta_k \nabla_y \Phi(x_{k-1},y_{k-1};\theta_{k-1}) $
    \State $y_{k+1} \gets \argmin_{y \in \mathcal{Y}} h(y) - \langle s_k, y \rangle + \frac{1}{\sigma_k}\mathbf D_{\mathcal{Y}}(y,y_k)$
    \State $\theta_{k+1} \gets \mathcal{M}(\theta_k, w_k)$
    \State $x_{k+1} \gets \argmin_{x \in \mathcal{X}} f(x) + \langle \nabla_x \Phi(x_k,y_{k+1};\theta_{k+1}),x \rangle + \frac{1}{\tau_k} \mathbf D_{\mathcal{X}}(x,x_k)$
    \If{$E_k(x_{k+1},y_{k+1};\theta_{k+1})\leq 0$}
        \State \textbf{go to line 16}
    \Else
        \State $\tau_k \gets \rho \tau_k$
        \State \textbf{go to line 4}
    \EndIf
\EndWhile
\State $\tau_{k+1} \gets \tau_k$, $k \gets k+1$
\EndFor
\end{algorithmic}
\end{algorithm}

Note that Assumption \ref{assump:lip-phi-theta} clearly holds if $\Theta$ is a compact set. Next, we establish some key properties of the sequences $\{\tau_k, \sigma_k, \eta_k\}$ associated with the Learning-aware algorithm, which will be used to derive the convergence rate in Theorem~\ref{thm:missp-approach}.

\begin{lemma}\label{lem:step-size}
     Suppose the sequence $\{ \tau_k, \sigma_k, \eta_k \}_{k \geq 0}$ satisfies the step-size conditions in Assumption \ref{assump:step-size-cond2} for some positive sequences $\{ t_k, \alpha_k \}_{k \geq 0}$, and nonnegative sequence $\{ \beta_k \}_{k \geq 0}$. Let $\{ x_k, y_k \}$ be the iterate sequence of Algorithm \ref{alg:mis-sp} corresponding to $\{ \tau_k, \sigma_k, \eta_k \}$. Then $\{ x_k, y_k \}$ and $\{ \tau_k, \sigma_k, \eta_k \}$ satisfy the condition in Assumption \ref{assump:step-size-cond}-(i) with the same sequences $\{ t_k, \alpha_k, \beta_k \}$.
\end{lemma}
\begin{proof}
    Using Assumption \ref{assump:phi-lip}, and Definition \ref{def:Breg-dis}, for any $k\geq 0$, and $(x,y) \in \mathcal{X} \times \mathcal{Y}$, we obtain
    \begin{align*}
        \langle \nabla_x\Phi(x,y;\theta)-\nabla_x\Phi(x_k,y;\theta),x-x_k\rangle &\leq \frac{L^{\Phi}_{xx}}{2}\|x-x_k\|^2_{\mathcal{X}} \leq L^{\Phi}_{xx}\mathbf D_{\mathcal{X}}(x,x_k),\\
        \frac{1}{2}\|\nabla_y \Phi(x,y;\theta)-\nabla_y \Phi(x_k,y;\theta)\|^2_{\mathcal{Y}^*}&\leq \frac{(L^{\Phi}_{yx})^2}{2}\|x-x_k\|^2_{\mathcal{X}}\leq (L^{\Phi}_{yx})^2 \mathbf D_{\mathcal{X}}(x,x_k),\\
        \| \nabla_y \Phi(x_k,y;\theta_k) - \nabla_y \Phi(x_k,y_k;\theta_k)\|^2_{\mathcal{Y}^*}&\leq (L^{\Phi}_{yy})^2\|y-y_k\|^2_{\mathcal{Y}}\leq 2(L^{\Phi}_{yy})^2\mathbf D_{\mathcal{Y}}(y,y_k).
        \end{align*}
        
    By evaluating the above inequality at $(x,y;\theta)=(x_{k+1},y_{k+1};\theta_{k+1})$, and using Assumption \ref{assump:step-size-cond2} along with the nonnegativity of Bregman functions, we can show
    \begin{align*}
        E_k \triangleq E_k(x_{k+1},y_{k+1};\theta_{k+1})&\leq \big(\frac{(L^{\Phi}_{yx})^2}{\alpha_{k+1}}+L^{\Phi}_{xx}-\frac{1}{\tau_k} \big)\mathbf D_{\mathcal{X}}(x_{k+1},x_k)\\
        &\quad + \big(\frac{2(L^{\Phi}_{yy})^2}{\beta_{k+1}}+\eta_k(\alpha_k+\beta_k)-\frac{1}{\sigma_k} \big)\mathbf D_{\mathcal{Y}}(y_{k+1},y_k).
    \end{align*}
\end{proof}

\begin{lemma}\label{lem:t-cond}
Let $\{ \tau_k \}_{k \geq 0} \subset \mathbb{R}_{++}$ and suppose $\bar{\tau}, \gamma > 0$. Define $\sigma_{-1} = \gamma \bar{\tau}$ and set $\sigma_k = \gamma \tau_k$, and $\eta_k = \sigma_{k-1} / \sigma_k$ for all $k \geq 0$. Then the sequences $\{ \tau_k, \sigma_k, \eta_k \}$ satisfy Assumption~\ref{assump:step-size-cond}-(ii) for the choice $t_k = \sigma_k / \sigma_0$.
\end{lemma}

\begin{proof}
By definition, $t_k = \sigma_k / \sigma_0$, and since $\tau_k > 0$ for all $k \geq 0$, Assumption~\ref{assump:step-size-cond}-(ii) is equivalent to verifying the two conditions:
\begin{align*}
\frac{\sigma_{k+1} \tau_k}{\sigma_k \tau_{k+1}} \leq 1 \quad \text{and} \quad \eta_{k+1} = \frac{\sigma_k}{\sigma_{k+1}}.
\end{align*}
The second condition holds directly by the definition of $\eta_k$. For the first condition, observe that since $\sigma_k = \gamma \tau_k$, we have
\begin{align*}
\frac{\sigma_{k+1} \tau_k}{\sigma_k \tau_{k+1}} = \frac{\gamma \tau_{k+1} \tau_k}{\gamma \tau_k \tau_{k+1}} = 1.
\end{align*}
Hence, the inequality holds with equality, and the result follows.
\end{proof}
\begin{lemma}\label{lem:tau-cond-alt}
Suppose Assumption \ref{assump:phi-lip} holds. Consider Algorithm \ref{alg:mis-sp} with constants $c_{\alpha}, c_{\beta} \geq 0$ such that $c_{\alpha} + c_{\beta} \leq 1$. If $L^{\Phi}_{yy} > 0$, assume $c_{\alpha}, c_{\beta} > 0$; otherwise, if $L^{\Phi}_{yy} = 0$, take $c_{\alpha} > 0$ and $c_{\beta} = 0$. Then the iterates $\{x_k, y_k\}_{k \geq 0}$ and step-size sequences $\{\tau_k, \sigma_k, \eta_k\}_{k \geq 0}$ in Algorithm \ref{alg:mis-sp} are well-defined. Specifically, for every iteration index $k \geq 0$, the backtracking condition $E_k(x_{k+1}, y_{k+1}; \theta_{k+1}) \leq 0$ is satisfied after a finite number of backtracking steps. 
Furthermore, there exists a constant $\hat{\tau} > 0$ such that $\tau_k \geq \rho \hat{\tau}$. When $L^{\Phi}_{yy} = 0$, we have $\hat{\tau} \geq \Psi_1$ for all $k \geq 0$. In the case where $L^{\Phi}_{yy} > 0$, it holds that $\hat{\tau} \geq \min\{\Psi_1, \Psi_2\}$ for all $k \geq 0$, where 
$\Psi_1 \triangleq \frac{-L^{\Phi}_{xx} + \sqrt{(L^{\Phi}_{xx})^2 + \frac{4 (L^{\Phi}_{yx})^2 \gamma}{c_{\alpha}}}}{2 (L^{\Phi}_{yx})^2 \gamma / c_{\alpha}}$ and $\Psi_2 \triangleq \frac{\sqrt{c_{\beta}(1 - (c_{\alpha} + c_{\beta}))}}{\sqrt{2} \gamma L^{\Phi}_{yy}}$. 

\end{lemma}
\begin{proof}
Fix any $k \geq 0$. From Lemma~\ref{lem:step-size}, we know that if Assumption~\ref{assump:step-size-cond2} holds, then Assumption~\ref{assump:step-size-cond}-(i) is automatically satisfied. We now demonstrate that there exists a positive threshold $\hat{\tau}$ such that Assumption~\ref{assump:step-size-cond2} is valid for all $\tau_k \in (0, \hat{\tau}]$.

Using the relationships $\sigma_k = \gamma \tau_k$ and $\eta_k = \sigma_{k-1}/\sigma_k$, the two inequalities in Assumption~\ref{assump:step-size-cond2} become:
\begin{align}\label{eq:assump-equivalent}
0 \geq -1 + L^{\Phi}_{xx} \tau_k + \frac{(L^{\Phi}_{yx})^2}{c_{\alpha}} \gamma \tau_k^2, \quad 
1 - (c_{\alpha} + c_{\beta}) \geq \frac{2 (L^{\Phi}_{yy})^2}{c_{\beta}} \gamma^2 \tau_k^2.
\end{align}

If $L^{\Phi}_{yy} > 0$, both inequalities must be satisfied. The first inequality leads to a quadratic condition in $\tau_k$ with a known positive root, and the second provides a bound based on a square root expression. Combining both gives:
\begin{align}\label{eq:tau-hat}
\hat{\tau} \triangleq \min\left\{ \frac{-L^{\Phi}_{xx} + \sqrt{(L^{\Phi}_{xx})^2 + \frac{4 (L^{\Phi}_{yx})^2 \gamma}{c_{\alpha}}}}{2 (L^{\Phi}_{yx})^2 \gamma / c_{\alpha}}, \; \frac{\sqrt{c_{\beta}(1 - (c_{\alpha} + c_{\beta}))}}{\sqrt{2} \gamma L^{\Phi}_{yy}} \right\}.
\end{align}

When $L^{\Phi}_{yy} = 0$, the second inequality in \eqref{eq:assump-equivalent} trivially holds, so only the first term in \eqref{eq:tau-hat} is relevant.

During the backtracking procedure, $\tau_k$ is geometrically reduced by a factor $\rho \in (0,1)$ at each inner step. Therefore, once the stopping criterion is met, it must hold that $\tau_k \geq \rho \hat{\tau}$.

Finally, from Algorithm \ref{alg:mis-sp}, we know $\gamma>0$. Hence, the expression for $\hat{\tau}$ remains constant over iterations. This implies that for $L^{\Phi}_{yy} = 0$, we have $\hat{\tau} \geq \Psi_1$, and for $L^{\Phi}_{yy} > 0$, we have $\hat{\tau} \geq \min\{\Psi_1, \Psi_2\}$, as desired.
\end{proof}

Next, we proceed with the iteration complexity analysis of the Learning-aware approach, as detailed in Theorem \ref{thm:missp-approach}.
\begin{theorem}\label{thm:missp-approach}
    Suppose Assumptions \ref{assump:phi-lip} and \ref{assump:lip-phi-theta} hold. {Let $z \triangleq [x^\top, y^\top]^\top$.} If $\{x_k,y_k,\theta_k\}_{k\geq 0}$ is generated by Algorithm \ref{alg:mis-sp}, using a parameter sequence $\{\tau_k,\sigma_k,\eta_k\}_{k\geq 0}$ that satisfies Assumptions \ref{assump:step-size-cond} or \ref{assump:step-size-cond2}. Then for any bounded set $\bar{Z} \subseteq \textbf{dom}\, f \times \textbf{dom}\, h$ containing a saddle point solution of \eqref{eq:sp-prob} and $K\geq 1$,
    \begin{align}\label{eq:missp-sup}
    \mathcal{G}_{\bar{Z}}(\bar{z}_K) &\leq \frac{1}{T_K} \sup_{z\in \bar{Z}} \Big(\frac{1}{\tau_0}\mathbf D_{\mathcal{X}}(x,x_0) + \frac{1}{\sigma_0}\mathbf D_{\mathcal{Y}}(y,y_0) \Big) \nonumber\\
    & \quad + \frac{2L_{\theta}}{T_K} \sum_{k=0}^{K-1}t_k\|\theta_{k+1}-\theta^*\|_{\Theta} +\frac{L^2_{y\theta}}{T_K} \sum_{k=0}^{K-1}\frac{t_k}{\beta_{k+1}}\|\theta_{k+1}-\theta_k\|^2_{\Theta},
\end{align}
\ma{where $\bar{z}_K \triangleq [\bar{x}_K^\top,\bar{y}_K^\top]^\top$ denotes the weighted average of the iterates 
$\{z_k\}_{k=0}^{K-1}$, i.e., $\bar{z}_K \triangleq \frac{1}{T_K} \sum_{k=0}^{K-1} t_k z_{k+1}$, and} $T_K = \sum_{k=0}^{K-1}t_k$.
\end{theorem}
\begin{proof}
Considering that in Algorithm \ref{alg:mis-sp}, the variable $x$ is updated after the variable $y$ and the parameter $\theta$ are updated.\\

Based on Algorithm \ref{alg:mis-sp}, we can define the following:
\begin{align*}
    &q_k \triangleq \nabla_y \Phi(x_k,y_k;\theta_k) - \nabla_y \Phi(x_{k-1},y_{k-1};\theta_{k-1}),\\
    &s_k \triangleq \nabla_y \Phi(x_k,y_k;\theta_k) + \eta_k q_k,
\end{align*}
where $q_k$ is the momentum parameter. For $k \geq  0$, using Lemma \ref{lem:f-bound} for the $y$- and $x$-subproblems in Algorithm \ref{alg:mis-sp}, we get the following two inequalities that hold for any $y \in \mathcal{Y}$ and $x \in \mathcal{X}$:
\begin{align}
    h(y_{k+1}) & - \langle s_k, y_{k+1}-y \rangle \nonumber\\
    &\leq h(y) + \frac{1}{\sigma_k}\Big[\mathbf D_{\mathcal{Y}}(y,y_k) - \mathbf D_{\mathcal{Y}}(y,y_{k+1}) - \mathbf D_{\mathcal{Y}}(y_{k+1},y_k)\Big] \label{eq:h-bound},\\
    f(x_{k+1}) & + \langle \nabla_x \Phi(x_k,y_{k+1};\theta_{k+1}),x_{k+1}-x \rangle \nonumber\\
    & \leq f(x) + \frac{1}{\tau_k}\Big[\mathbf D_{\mathcal{X}}(x,x_k)-\mathbf D_{\mathcal{X}}(x,x_{k+1})-\mathbf D_{\mathcal{X}}(x_{k+1},x_k)\Big]\label{eq:f-bound}.
\end{align}
For all $k \geq 0$, let $A_{k+1} \triangleq \frac{1}{\sigma_k}\Big[\mathbf D_{\mathcal{Y}}(y,y_k) - \mathbf D_{\mathcal{Y}}(y,y_{k+1}) - \mathbf D_{\mathcal{Y}}(y_{k+1},y_k)\Big]$ and $B_{k+1} \triangleq \frac{1}{\tau_k}\Big[\mathbf D_{\mathcal{X}}(x,x_k)-\mathbf D_{\mathcal{X}}(x,x_{k+1})-\mathbf D_{\mathcal{X}}(x_{k+1},x_k)\Big]$. By adding and subtracting $\nabla_x \Phi(x_k,y_{k+1};\theta_{k+1})x_k$, the inner product in \eqref{eq:f-bound} can be lower bounded using convexity of $\Phi(x_k,y_{k+1};\cdot)$ as follows:
\begin{align}\label{eq:inner-eq}
    &\langle \nabla_x \Phi(x_k,y_{k+1};\theta_{k+1}),x_{k+1}-x \rangle \nonumber\\
    & \quad= \langle \nabla_x \Phi(x_k,y_{k+1};\theta_{k+1}),x_{k}-x \rangle + \langle \nabla_x \Phi(x_k,y_{k+1};\theta_{k+1}),x_{k+1}-x_k \rangle,\nonumber\\
    &\quad \geq \Phi(x_k,y_{k+1};\theta_{k+1}) - \Phi(x,y_{k+1};\theta_{k+1})  + \langle \nabla_x \Phi(x_k,y_{k+1};\theta_{k+1}),x_{k+1}-x_k \rangle.
\end{align}
Using this inequality after adding $\Phi(x_{k+1},y_{k+1};\theta_{k+1})$ to both sides of \eqref{eq:f-bound}, we can get
\begin{align}\label{eq:f-bound2}
    f(x_{k+1}) + \Phi(x_{k+1},y_{k+1};\theta_{k+1}) \leq f(x) + \Phi(x,y_{k+1};\theta_{k+1}) + B_{k+1} + \Lambda_k,
\end{align}
where $\Lambda_k \triangleq \Phi(x_{k+1},y_{k+1};\theta_{k+1}) - \Phi(x_{k},y_{k+1};\theta_{k+1}) - \langle \nabla_x \Phi(x_k,y_{k+1};\theta_{k+1}),x_{k+1}-x_k \rangle$, for $k \geq 0$. For $k\geq 0$, summing \eqref{eq:h-bound} and \eqref{eq:f-bound2} and rearranging the terms lead to
\begin{align}\label{eq:L-bound1}
    &\mathcal{L}(x_{k+1},y;\theta_{k+1})-\mathcal{L}(x,y_{k+1};\theta_{k+1}) \nonumber\\
    & \quad= f(x_{k+1}) + \Phi(x_{k+1},y;\theta_{k+1})-h(y)-f(x)-\Phi(x,y_{k+1};\theta_{k+1})+h(y_{k+1})\nonumber\\
    &\quad \leq \Phi(x_{k+1},y;\theta_{k+1}) - \Phi(x_{k+1},y_{k+1};\theta_{k+1})+ \langle s_k, y_{k+1}-y \rangle + \Lambda_k  + A_{k+1} + B_{k+1}\nonumber\\
    &\quad \leq - \langle q_{k+1},y_{k+1}-y \rangle + \eta_k \langle q_k,y_{k+1}-y \rangle + \Lambda_k + A_{k+1} + B_{k+1},
\end{align}
where in the last inequality we use the concavity of $\Phi(x_{k+1},\cdot;\theta_{k+1})$.\\
To obtain a telescoping sum later, by adding and subtracting $\eta_k q_ky_k$, we can rewrite the bound in \eqref{eq:L-bound1} as
\begin{align}\label{eq:L-bound2}
    &\mathcal{L}(x_{k+1},y;\theta_{k+1})-\mathcal{L}(x,y_{k+1};\theta_{k+1}) \nonumber\\
    & \quad\leq \Big[\frac{1}{\tau_k}\mathbf D_{\mathcal{X}}(x,x_k) + \frac{1}{\sigma_k}\mathbf D_{\mathcal{Y}}(y,y_k) + \eta_k \langle q_k,y_k-y \rangle \Big]\nonumber\\
    & \qquad - \Big[\frac{1}{\tau_k}\mathbf D_{\mathcal{X}}(x,x_{k+1}) + \frac{1}{\sigma_k}\mathbf D_{\mathcal{Y}}(y,y_{k+1}) + \langle q_{k+1},y_{k+1}-y \rangle \Big] \nonumber\\
    & \qquad + \Lambda_k - \frac{1}{\tau_k}\mathbf D_{\mathcal{X}}(x_{k+1},x_k) - \frac{1}{\sigma_k}\mathbf D_{\mathcal{Y}}(y_{k+1},y_k) + \eta_k \langle q_k, y_{k+1} - y_k \rangle.
\end{align}
One can bound $\langle q_k,y-y_k \rangle$ for any given $y \in \mathcal{Y}$ as follows. Let $p^x_k \triangleq \nabla_y \Phi(x_k,y_k;\theta_k) - \nabla_y \Phi(x_{k-1},y_k;\theta_k)$, $p^{\theta}_k \triangleq \nabla_y \Phi(x_{k-1},y_k;\theta_k)-\nabla_y\Phi(x_{k-1},y_k;\theta_{k-1})$, and $p^y_k \triangleq \nabla_y \Phi(x_{k-1},y_k;\theta_{k-1}) - \nabla_y \Phi(x_{k-1},y_{k-1};\theta_{k-1})$ which imply that $q_k = p^x_k + p^{\theta}_k+p^y_k$. Moreover, for any $y\in \mathcal{Y}$, $y'\in \mathcal{Y}^*$, and $a>0$, we have $|\langle y',y \rangle| \leq \frac{a}{2} \|y\|^2_{\mathcal{Y}} + \frac{1}{2a} \|y'\|^2_{\mathcal{Y}^*}$. Hence, using this inequality for each $\langle p^x_k,y-y_k \rangle$, and $\langle p^y_k+p^{\theta}_k,y-y_k \rangle$ and the fact that $\mathbf D_{\mathcal{Y}}(y,\overline{y})\geq \frac{1}{2}\|y-\overline{y}\|^2_{\mathcal{Y}}$, we obtain for all $k\geq0$:
\begin{align}\label{eq:pxyk}
    |\langle q_k,y-y_k \rangle| &\leq \alpha_k \mathbf D_{\mathcal{Y}}(y,y_k) + \frac{1}{2\alpha_k}\|p^x_k\|^2_{\mathcal{Y}^*} + \beta_k \mathbf D_{\mathcal{Y}}(y,y_k) + \frac{1}{2\beta_k}\|p^y_k+p^{\theta}_k\|^2_{\mathcal{Y}^*},
\end{align}
which holds for any $\alpha_k, \beta_k >0$. Therefore, using \eqref{eq:pxyk} within \eqref{eq:L-bound2}, we can show for $k\geq0$
\begin{subequations}
\begin{align}
    &\mathcal{L}(x_{k+1},y;\theta_{k+1})-\mathcal{L}(x,y_{k+1};\theta_{k+1})  \leq Q_k(z) - R_{k+1}(z) + E_k + \frac{1}{\beta_{k+1}}\|p^{\theta}_{k+1}\|^2_{\mathcal{Y}^*}\label{eq:L-bound-fnl},\\
    &Q_k(z) \triangleq \frac{1}{\tau_k}\mathbf D_{\mathcal{X}}(x,x_k) + \frac{1}{\sigma_k}\mathbf D_{\mathcal{Y}}(y,y_k) + \eta_k \langle q_k,y_k-y \rangle + \frac{\eta_k}{2\alpha_k}\|p^x_k\|^2_{\mathcal{Y}^*} \nonumber\\
    &\qquad \qquad+ \frac{\eta_k}{2\beta_k}\|p^y_k+p^{\theta}_k\|^2_{\mathcal{Y}^*},\\
    &R_{k+1}(z) \triangleq \frac{1}{\tau_k}\mathbf D_{\mathcal{X}}(x,x_{k+1}) + \frac{1}{\sigma_k}\mathbf D_{\mathcal{Y}}(y,y_{k+1}) + \langle q_{k+1},y_{k+1}-y \rangle + \frac{1}{2\alpha_{k+1}}\|p^x_{k+1}\|^2_{\mathcal{Y}^*}\nonumber\\
    & \qquad \qquad + \frac{1}{2\beta_{k+1}}\|p^y_{k+1}+p^{\theta}_{k+1}\|^2_{\mathcal{Y}^*},\\
    & E_k \triangleq \Lambda_k + \frac{1}{2\alpha_{k+1}}\|p^x_{k+1}\|^2_{\mathcal{Y}^*} - \frac{1}{\tau_k}\mathbf D_{\mathcal{X}}(x_{k+1},x_k) + \frac{1}{\beta_{k+1}}\|p^y_{k+1}\|^2_{\mathcal{Y}^*} \nonumber\\
    & \qquad \qquad - \Big(\frac{1}{\sigma_k} - \eta_k(\alpha_k + \beta_k) \Big)\mathbf D_{\mathcal{Y}}(y_{k+1},y_k),
\end{align}
\end{subequations}

It should be mentioned that $E_k=E_k(x_{k+1},y_{k+1};\theta_{k+1})$ where $E_k(\cdot,\cdot;\cdot)$ is defined as in Definition \ref{def:E_k} for $k\geq 0$.\\
Considering Assumption \ref{assump:lip-phi-theta}, we can show
\begin{align*}
    &|\mathcal{L}(x_{k+1},y;\theta_{k+1}) - \mathcal{L}(x_{k+1},y;\theta^*)| \leq L_{\theta} \|\theta_{k+1} - \theta^*\|_{\Theta}\\
    &|\mathcal{L}(x,y_{k+1};\theta_{k+1}) - \mathcal{L}(x,y_{k+1};\theta^*)| \leq L_{\theta} \|\theta_{k+1} - \theta^*\|_{\Theta}
\end{align*}
For $\theta^*, \theta_{k+1} \in \Theta$ we can show
\begin{align*}
    &{\mathcal{L}(x_{k+1},y;\theta^*) - \mathcal{L}(x,y_{k+1};\theta^*)} \nonumber\\
    & \quad= \mathcal{L}(x_{k+1},y;\theta^*) - \mathcal{L}(x_{k+1},y;\theta_{k+1})  + \mathcal{L}(x_{k+1},y;\theta_{k+1}) - \mathcal{L}(x,y_{k+1};\theta_{k+1}) \nonumber\\
    & \qquad + \mathcal{L}(x,y_{k+1};\theta_{k+1}) - \mathcal{L}(x,y_{k+1};\theta^*)
\end{align*}
Then, using the triangle inequality, we can obtain
\begin{align}\label{eq:L-bound-missp}
    &{\mathcal{L}(x_{k+1},y;\theta^*) - \mathcal{L}(x,y_{k+1};\theta^*)} \nonumber\\
    &\quad \leq | \mathcal{L}(x_{k+1},y;\theta^*) - \mathcal{L}(x_{k+1},y;\theta_{k+1})| + \mathcal{L}(x_{k+1},y;\theta_{k+1}) - \mathcal{L}(x,y_{k+1};\theta_{k+1}) \nonumber\\
    & \qquad + | \mathcal{L}(x,y_{k+1};\theta_{k+1}) - \mathcal{L}(x,y_{k+1};\theta^*)|\nonumber\\
    &\quad\leq 2L_{\theta}\|\theta_{k+1}-\theta^*\|_{\Theta} + Q_k(z) - R_{k+1}(z)+ E_k + \frac{1}{\beta_{k+1}}\|p^{\theta}_{k+1}\|^2_{\mathcal{Y}^*}\nonumber\\
    &\quad\leq Q_k(z) - R_{k+1}(z)+ E_k + 2L_{\theta}\|\theta_{k+1}-\theta^*\|_{\Theta} + \frac{L^2_{y\theta}}{\beta_{k+1}}\|\theta_{k+1}-\theta_k\|^2_{\Theta},
\end{align}
where the inequality in \eqref{eq:L-bound-missp} is obtained by using Assumption \ref{assump:lip-phi-theta} in \eqref{eq:L-bound-fnl} and Assumption \ref{assump:phi-lip}-(ii).\\

All the derivations up to now, including \eqref{eq:L-bound-missp}, hold for any Bregman distance functions $\mathbf D_{\mathcal{X}}$ and $\mathbf D_{\mathcal{Y}}$. Now, multiplying both sides by $t_k>0$, summing over $k=0$ to $K-1$, and then using Jensen's inequality, we obtain
\begin{align}\label{eq:sum-bound-missp}
    &T_K (\mathcal{L}(\overline{x}_K,y;\theta^*) - \mathcal{L}(x,\overline{y}_K;\theta^*)) \nonumber\\
    & \quad \leq \sum_{k=0}^{K-1} t_k \Big(Q_k(z) - R_{k+1}(z) + E_k + 2L_{\theta}\|\theta_{k+1}-\theta^*\|_{\Theta}+\frac{L^2_{y\theta}}{\beta_{k+1}}\|\theta_{k+1}-\theta_k\|^2_{\Theta} \Big),\nonumber\\
    &\quad\leq t_0Q_0(z) - t_{K-1}R_K(z) + \sum_{k=0}^{K-1}t_kE_k + 2L_{\theta}\sum_{k=0}^{K-1}t_k\|\theta_{k+1}-\theta^*\|_{\Theta} \nonumber\\
    & \qquad+ L^2_{y\theta}\sum_{k=0}^{K-1}\frac{t_k}{\beta_{k+1}}\|\theta_{k+1}-\theta_k\|^2_{\Theta},
\end{align}
where $T_K = \sum_{k=0}^{K-1}t_k$ and the last inequality follows from the step-size conditions in Assumption \ref{assump:step-size-cond}-(ii), which imply that $t_{k+1}Q_{k+1}(z) - t_kR_{k+1}(z)\leq0$ for $k=0$ to $K-2$.\\

According to Assumption \ref{assump:step-size-cond}, $\tau_k,\sigma_k$ and $\eta_k$ are chosen such that $E_k \leq0$ for $k=0,...,K-1$; then \eqref{eq:sum-bound-missp} implies that
\begin{align}\label{eq:sum-bound-missp2}
    &T_K (\mathcal{L}(\overline{x}_K,y;\theta^*) - \mathcal{L}(x,\overline{y}_K;\theta^*)) \nonumber\\
    & \quad \leq t_0Q_0(z) - t_{K-1}R_K(z) + 2L_{\theta}\sum_{k=0}^{K-1}t_k\|\theta_{k+1}-\theta^*\|_{\Theta}+ L^2_{y\theta}\sum_{k=0}^{K-1}\frac{t_k}{\beta_{k+1}}\|\theta_{k+1}-\theta_k\|^2_{\Theta},\nonumber\\
    &\quad\leq \frac{t_0}{\tau_0} \mathbf D_{\mathcal{X}}(x,x_0) + \frac{t_0}{\sigma_0} \mathbf D_{\mathcal{Y}}(y,y_0) + t_K \eta_K \langle q_K,y-y_K\rangle - t_K \big[\frac{1}{\tau_K} \mathbf D_{\mathcal{X}}(x,x_K)  \nonumber\\
    & \qquad  + \frac{1}{\sigma_K} \mathbf D_{\mathcal{Y}}(y,y_K)+ \frac{\eta_K}{2\alpha_K}\|p^x_K \|^2_{\mathcal{Y}^*}+ \frac{\eta_K}{2\beta_K}\|p^y_K +p^{\theta}_K\|^2_{\mathcal{Y}^*}\big] +2L_{\theta}\sum_{k=0}^{K-1}t_k\|\theta_{k+1}-\theta^*\|_{\Theta} \nonumber\\
    & \qquad + L^2_{y\theta}\sum_{k=0}^{K-1}\frac{t_k}{\beta_{k+1}}\|\theta_{k+1}-\theta_k\|^2_{\Theta},
\end{align}
where in the last inequality we use $t_KQ_K(z) \leq t_{K-1}R_K(z)$ and $q_0 = p^x_0 = p^{\theta}_0 = p^y_0 = \bold{0}$. Using \eqref{eq:pxyk} to bound $\langle q_K,y-y_K\rangle$ and dividing both sides of \eqref{eq:sum-bound-missp2} by $T_K$, we can get
\begin{align}\label{eq:missp-bound}
    &\mathcal{L}(\overline{x}_K,y;\theta^*) - \mathcal{L}(x,\overline{y}_K;\theta^*)\nonumber\\
    & \quad\leq \frac{1}{T_K} \Big(\frac{1}{\tau_0}\mathbf D_{\mathcal{X}}(x,x_0) + \frac{1}{\sigma_0}\mathbf D_{\mathcal{Y}}(y,y_0) \Big) + \frac{2L_{\theta}}{T_K} \sum_{k=0}^{K-1}t_k\|\theta_{k+1}-\theta^*\|_{\Theta} - \frac{t_K}{T_K}\Big[\frac{1}{\tau_K}\mathbf D_{\mathcal{X}}(x,x_K)  \nonumber\\
    &\qquad + \Big(\frac{1}{\sigma_K}-\eta_K(\alpha_K+ \beta_K) \Big)\mathbf D_{\mathcal{Y}}(y,y_K) \Big] +\frac{L^2_{y\theta}}{T_K} \sum_{k=0}^{K-1}\frac{t_k}{\beta_{k+1}}\|\theta_{k+1}-\theta_k\|^2_{\Theta}.
\end{align}
By applying $sup$ to both sides of the inequality \eqref{eq:missp-bound} and using Definition \ref{def:gap-function}, we obtain the desired result in \eqref{eq:missp-sup}.
\end{proof}

\begin{corollary}\label{cor:missp-rate}
Under the premises of Theorem~\ref{thm:missp-approach}, for any $K \geq 1$, it holds that:
\begin{align*}
\mathcal{G}_{\bar{Z}}(\bar{z}_K) &\leq \frac{1}{T_K} \left( \frac{1}{\tau_0} \mathbf{R}_{x} + \frac{1}{\sigma_0} \mathbf{R}_{y} \right)
+ \frac{2L_{\theta} \sqrt{C_{\theta}}}{T_K} (\log K + 1)
+ \frac{4C_{\theta} L_{y\theta}^2 \pi^2}{6T_K \beta},
\end{align*}
where $T_K = \sum_{k=0}^{K-1}t_k$, and \ma{$\mathbf R_{x} \;\triangleq\;
\sup_{x,\bar x \in \bar{ X}} \mathbf D_{\mathcal X}(x,\bar x)<\infty$ and
$\mathbf R_{y}\;\triangleq\;
\sup_{y,\bar y \in \bar{Y}} \mathbf D_{\mathcal Y}(y,\bar y)<\infty$ are the finite Bregman distances}. 
\end{corollary}

\begin{proof}
From Theorem~\ref{thm:missp-approach}, we have the bound:
\begin{align}\label{eq:cor7-1}
\mathcal{G}_{\bar{Z}}(\bar{z}_K) &\leq \frac{1}{T_K} \sup_{z\in \bar{Z}} \left( \frac{1}{\tau_0} \mathbf D_{\mathcal{X}}(x,x_0) + \frac{1}{\sigma_0} \mathbf D_{\mathcal{Y}}(y,y_0) \right) \nonumber\\
&\quad + \underbrace{\frac{2L_{\theta}}{T_K} \sum_{k=0}^{K-1} t_k\|\theta_{k+1} - \theta^*\|_{\Theta}}_{(I)} + \underbrace{\frac{L_{y\theta}^2}{T_K} \sum_{k=0}^{K-1} \frac{t_k}{\beta_{k+1}} \|\theta_{k+1} - \theta_k\|^2_{\Theta}}_{(II)}.
\end{align}
For the term (I), using Remark \ref{rem:theta-rate}, we get:
\begin{align}\label{eq:cor7-2}
\frac{2L_{\theta}}{T_K} \sum_{k=0}^{K-1}t_k \|\theta_{k+1} - \theta^*\|_{\Theta} \leq \frac{2L_{\theta}}{T_K} \sqrt{C_{\theta}} \sum_{k=1}^{K} \frac{1}{k} \leq \frac{2L_{\theta}\sqrt{C_{\theta}}}{T_K}  (\log K + 1).
\end{align}
On the other hand, by adding and subtracting $\theta^*$, using Young's inequality and Remark \ref{rem:theta-rate}:
\begin{align*}
\|\theta_{k+1} - \theta_k\|^2_{\Theta} \leq 2 \left( \|\theta_{k+1} - \theta^*\|^2_{\Theta} + \|\theta_k - \theta^*\|^2_{\Theta} \right) \leq 2C_{\theta} \left( \frac{1}{(k+1)^2} + \frac{1}{k^2} \right),
\end{align*}
\rma{Recall that Algorithm~\ref{alg:mis-sp} uses $t_k=\sigma_k/\sigma_0$ and
$\beta_{k+1}=c_\beta/\sigma_k$. Hence
$\frac{t_k}{\beta_{k+1}}=\frac{\sigma_k^2}{\sigma_0 c_\beta}\le
\frac{\sigma_0}{c_\beta}\triangleq\frac{1}{\beta}$, with
$\beta\triangleq c_\beta/\sigma_0$, where the inequality uses the monotonicity of
the accepted backtracking step-sizes, $\sigma_k\le\sigma_0$. Therefore, using
Remark~\ref{rem:theta-rate},
\begin{align*}
\sum_{k=0}^{K-1}\frac{t_k}{\beta_{k+1}}\|\theta_{k+1}-\theta_k\|_\Theta^2
&\le \frac{1}{\beta}\sum_{k=0}^{K-1}\|\theta_{k+1}-\theta_k\|_\Theta^2\\
&\le \frac{2C_\theta}{\beta}\sum_{k=1}^{K}\Big(\frac{1}{k^2}+\frac{1}{(k+1)^2}\Big)
\le \frac{4C_\theta}{\beta}\cdot\frac{\pi^2}{6}.
\end{align*}}
Therefore, the term (II) is bounded by
\begin{align}\label{eq:cor7-3}
\frac{L_{y\theta}^2}{T_K} \sum_{k=0}^{K-1} \frac{t_k}{\beta_{k+1}} \|\theta_{k+1} - \theta_k\|^2_{\Theta} \leq \frac{4C_{\theta} L_{y\theta}^2 \pi^2}{6T_K \beta}.
\end{align}
Combining the inequalities in \eqref{eq:cor7-2} and \eqref{eq:cor7-3} within \eqref{eq:cor7-1} we obtain, 
\begin{align*}
\mathcal{G}_{\bar{Z}}(\bar{z}_K) &\leq \frac{1}{T_K} \ma{\left( \frac{1}{\tau_0} \mathbf{R}_{x} + \frac{1}{\sigma_0} \mathbf{R}_{y} \right)}
+ \frac{2L_{\theta} \sqrt{C_{\theta}}}{T_K} (\log K + 1)
+ \frac{4C_{\theta} L_{y\theta}^2 \pi^2}{6T_K \beta},
\end{align*}
which concludes the proof.
\end{proof}
\rma{%
\begin{remark}\label{rem:la-not-heuristic}
We would like to remark that the Learning-aware modification is structural rather than a tuning of the Naive method. In the Naive updates, $\theta_k$ is fixed within each iteration; hence, the error between $\theta_k$ and $\theta^*$ enters the analysis as a first-order residual multiplied by the distance of primary iterations $(x_k,y_k)$
(Theorem~\ref{thm:naive-approach}). Carrying $\theta_{k-1}$ into the momentum instead routes the parameter drift through the computable difference $p^{\theta}_k$, which the analysis absorbs into $\tfrac{t_k}{\beta_{k+1}}\|\theta_{k+1}-\theta_k\|^2_\Theta$ together with $t_k\|\theta_{k+1}-\theta^*\|_\Theta$
(Theorem~\ref{thm:missp-approach}). This change in how the learning error couples to the optimization iterates is precisely what makes the backtracking test
$E_k(x_{k+1},y_{k+1};\theta_{k+1})\le0$ evaluable at the available iterate
$\theta_{k+1}$ rather than at the unknown $\theta^*$. The adaptive, Lipschitz-free step-sizes are therefore a consequence of the modified momentum.
\end{remark}}%

\begin{remark}\label{rem:comp-naive}
When comparing the convergence bounds of the Naive and Learning-aware approaches (Corollaries \ref{cor:naive-rate} and \ref{cor:missp-rate}), we see that both share a leading term of order $\frac{1}{K}\left(\frac{1}{\tau_0}\mathbf R_{x}+\frac{1}{\sigma_0}\mathbf R_{y} \right)$ and noting that $T_K=\Omega(K)$, which reflects the baseline convergence guaranteed by the primal–dual updates. The main distinction lies in the subsequent terms: in particular, the $\mathcal{O}(1)$ contribution, corresponding to the terms with a slower rate of $\log(K)/K$, captures how each method accounts for the interaction between optimization and learning dynamics. Assuming Lipschitz constants are of order $\mathcal{O}(1)$, the Naive approach (Corollary~\ref{cor:naive-rate}) depends on the domain diameter $\sqrt{\mathbf R_x}+\sqrt{\mathbf R_y}$, while the Learning-aware approach (Corollary~\ref{cor:missp-rate}) does not, which demonstrates the benefit of the proposed Learning-aware approach.
\end{remark}

\section{Multiple Learning Solutions}\label{sec:extension} 
In this section, we extend our framework to a setting where the learning subproblem may admit multiple optimal solutions by focusing on a scenario where the learning problem is a minimization of the form: $\min_{\theta \in \Xi} \ell(\theta)$ for some convex and continuously differentiable function $\ell(\cdot)$\rma{, where $\Xi \subseteq \Theta$ is the feasible set for the learning variable and
$\Theta$ is the space introduced in Problem~\eqref{eq:sp-prob}}. In this case, the structure becomes more complex, as the decision-making process requires determining which element of the learning solution set $\Xi^*$ should be selected in the objective function. To address this, we adopt a pessimistic (worst-case) formulation of the original saddle point problem in \eqref{eq:sp-prob}, leading to the following formulation. 
\begin{align}\label{eq:sp-ext}
    \min_{x \in \mathcal{X}} \max_{y \in \mathcal{Y}} \max_{\theta \in \Xi^*} \mathcal{L}(x, y, \theta),
\end{align}
where $\Xi^* \triangleq \argmin_{\theta \in \Xi} \ell(\theta) = \left\{ \theta \in \Xi \mid \ell(\theta) \leq \ell^* \right\}$ and $\ell^* \triangleq \min_{\theta \in \Xi} \ell(\theta)$. In this pessimistic formulation, $\theta$ is treated as a decision variable that captures the worst-case performance of the objective function $\mathcal L$ with respect to the learning solution set $\Xi^*$.

\rma{The challenge of having multiple learning solution in our setup is closely related to the similar phenomena arising bilevel optimization when the lower-level problem admits multiple optimal solutions, 
one must specify which element of the lower-level solution set is selected to evaluate the upper-level objective. This leads to different possible models, most notably the optimistic and pessimistic formulations
\cite{dempe2020bilevel,sinha2017review,liu2021investigating}, which can have different values and substantially different analytical properties. Much of the tractable bilevel literature, therefore, either assumes uniqueness of the lower-level solution or imposes structure, such as optimistic formulation, that restores well-posedness. Our approach is relevant to the pessimistic (worst-case) setting and we obtain tractability through a structural assumption on the coupling.}

It is worth noting that the pessimistic formulation in \eqref{eq:sp-ext} introduces two key challenges: First, the objective $\mathcal{L}(x,\cdot,\cdot)$ in \eqref{eq:sp-ext} is not necessarily jointly concave in $y$ and $\theta$.
Second, the set $\Xi^*$ is implicitly defined through the optimal value $\ell^*$ and is not directly available. To overcome the first difficulty, we focus on a structured case in which the roles of $\theta$ and $y$ are decoupled. Specifically, we assume that the saddle function $\Phi$ can be written as 
$\Phi(x, y; \theta) = g_1(x, \theta) + g_2(x, y)$,
where $g_1$ and $g_2$ are continuously differentiable and convex–concave. This structure commonly appears in constrained optimization problems where $\theta$ affects only the objective (via $g_1$), while $y$ represents the dual variables for the constraints (via $g_2$) -- see the portfolio optimization example in Section~\ref{subsec:portfolio} for an instance of this setting. 


Under this structure, we can rewrite \eqref{eq:sp-ext} as:
\begin{align}\label{eq:pes-form1}
    \min_{x \in \mathcal{X}} \max_{y \in \mathcal{Y}} \max_{\theta \in \Xi^*} f(x) + g_1(x, \theta) + g_2(x, y) - h(y).
\end{align}
To address the second challenge and handling the constraint $\ell(\theta) \leq \ell^*$, we introduce a Lagrange multiplier $w$ and write the saddle point reformulation for the inner maximization in \eqref{eq:pes-form1} as follows:
\begin{align}\label{eq:lag-ext}
    \max_{\theta \in \Xi} \min_{w \geq 0} \left\{ g_1(x, \theta) - \langle w, \ell(\theta) - \ell^* \rangle \right\}.
\end{align}
Note that since the feasible set $\{\theta \in \Xi \mid \ell(\theta) \leq \ell^*\}$ does not satisfy strict feasibility, Slater's condition fails and strong duality may not hold. To remedy this, we relax the constraint by introducing an $\epsilon$-approximation and define the relaxed solution set: $\left\{ \theta \in \Xi \mid \ell(\theta) \leq \ell^* + \epsilon \right\}$. This relaxation ensures that Slater's condition and subsequently strong duality hold \cite{boyd2004convex}, allowing us to safely exchange the max–min in \eqref{eq:lag-ext} with a min–max. \rma{Other possible relaxations include a penalty approach by adding the constraint residual to the objective function. However, the difficulty there is that a
suitable penalty parameter must be chosen, and its choice is problem-dependent. The $\epsilon$-relaxation instead exposes a single, directly interpretable quantity in which the user can specify the accuracy $\epsilon$ to which the lower-level problem is to be solved.}


Substituting the dual formulation into \eqref{eq:pes-form1}, we obtain the following min-max reformulation:
\begin{align}\label{eq:f-ext}
    &\min_{x \in \mathcal{X},\, w \geq 0} \max_{y \in \mathcal{Y},\, \theta \in \Xi} 
    \mathcal{L}(x, w, y, \theta; \ell^*) \\
     & \mathcal{L}(x, w, y, \theta; \ell^*)\triangleq f(x) + g_1(x, \theta) + g_2(x, y) - \langle w, \ell(\theta) - \ell^* - \epsilon \rangle - h(y).\nonumber
\end{align}
In this setting, the problem becomes a misspecified saddle point problem where $\ell^*$ serves as an unknown parameter estimated through a distinct learning process. The variable $\theta$ now selects among multiple near-optimal learning solutions, while $w$ acts as the dual multiplier associated with the learning constraint.

\begin{remark}
\rma{We would like to clarify why, in the multiple-solution setting, the secondary learning problem is restricted to a convex \emph{minimization} problem $\min_\theta\ell(\theta)$. 
Retaining a saddle-point formulation would require characterizing the learning-solution set through the value function $L(\theta)\triangleq \max_w l(\theta,w)$. Under convex-concavity of $l(\cdot,\cdot)$, $L$ is generally nondifferentiable, leading to nonsmooth constraints and couplings that are incompatible with the smoothness requirements of our analysis and backtracking line search scheme. If we assume $l(\theta,\cdot)$ is strongly concave for any $\theta$ then we can restore differentiability of $L$, with $\nabla L(\theta)=\nabla_\theta l(\theta,w^*(\theta))$; however, evaluating this gradient requires the exact inner maximizer. Using an inexact inner solution introduces a gradient error that depends on the unknown distance to $w^*(\theta)$. Because this error enters the backtracking condition, the corresponding certificate cannot be evaluated, and the line-search procedure is no longer implementable.}
\end{remark}


To solve problem \eqref{eq:f-ext}, we employ a modified version of our Learning-aware APD algorithm with backtracking proposed in Section \ref{subsec:misspecified}, incorporating the following adjustments.  
In particular, the minimization variables are $(x,w)$ and the maximization variables are $(y,\theta)$.  
Accordingly, we define two momentum terms, $s^y_k$ and $s^{\theta}_k$, which are used in the maximization variable updates (lines 5--6).  A key difference from the standard setting is that we introduce an auxiliary sequence $\{\tilde{\theta}\}_k$ updated via one-step of accelerated projected gradient descent \cite{nesterov1983method},  
\begin{align*}
\zeta_k\gets \tilde{\theta}_k + \frac{k-2}{k+1}(\tilde{\theta}_k-\tilde{\theta}_{k-1}),\quad 
\tilde{\theta}_{k+1} \gets \mathcal{P}_{\Xi}\Big(\zeta_k
          - \gamma_{\theta}\nabla\ell\big(\zeta_k\big)\Big),
\end{align*}
and then evaluated at the learning objective $\ell$ to estimate the objective value of the learning problem, i.e., $\ell_{k+1} = \ell(\tilde{\theta}_{k+1})$ \ma{(see lines 9-11 in Algorithm \ref{alg:ext-multisol})}. 


Next, to adaptively select the step-sizes using local Lipschitz constants in our method, we employ a backtracking strategy. At each iteration $k \geq 0$, we check whether the chosen step-sizes are consistent with the local Lipschitz behavior by introducing a test function $\bar E_k(\cdot,\cdot,\cdot,\cdot;\cdot)$. Note that, compared to the test function $E_k$ in the Learning-aware approach in Definition \ref{def:E_k}, we further need to estimate the local Lipschitz constants corresponding to variables $w$ and $\theta$. 

\begin{definition}\label{def:E_k-ext}
Given a free parameter sequence $\{ \alpha_k, \beta_k \}_{k \geq 0} \subseteq \mathbb{R}_+$, we define
    \begin{align*}
        &\bar E_k(x,w,y,\theta;\ell)\\
        & \quad\triangleq \langle \nabla_xg_1(x,\theta)-\nabla_x g_1(x_k,\theta),x-x_k\rangle +\langle \nabla_xg_2(x,y)-\nabla_x g_2(x_k,y),x-x_k\rangle \\
        & \qquad + \frac{1}{2\alpha_{k+1}}\| \nabla_y g_2(x,y) - \nabla_y g_2(x_k,y)\|^2_{\mathcal{Y}^*} + \frac{1}{\beta_{k+1}}\| \nabla_y g_2(x_k,y) - \nabla_y g_2(x_k,y_k)\|^2_{\mathcal{Y}^*}\\
        & \qquad + \frac{1}{\alpha_{k+1}}\| \nabla_{\theta} g_1(x,\theta) - \nabla_{\theta} g_1(x_k,\theta)\|^2_{\Theta^*} + \frac{1}{\alpha_{k+1}}\| \nabla_{\theta} g_1(x_k,\theta) - \nabla_{\theta} g_1(x_k,\theta_k)\|^2_{\Theta^*} \\
        & \qquad + \frac{1}{\beta_{k+1}}\|-w  \nabla \ell(\theta) + w_k \nabla \ell(\theta)\|^2_{\Theta^*} + \frac{1}{\beta_{k+1}}\|-w_k  \nabla \ell(\theta) + w_k \nabla \ell(\theta_k)\|^2_{\Theta^*}  \\
        &\qquad - \Big(\frac{1}{\sigma_k}- \eta_k(\alpha_k + \beta_k) \Big)\Big(\mathbf D_{\mathcal{Y}}(y,y_k) + \mathbf D_{\Theta}(\theta,\theta_k) \Big) - \frac{1}{\tau_k} \mathbf D_{\mathcal{X}}(x,x_k)- \frac{1}{\tau_k} \mathbf D_{\mathcal W}(w,w_k).
    \end{align*} 
 \noindent
\end{definition}
Later, we demonstrate that for given $\alpha_k, \beta_k \geq 0$ and $\eta_k\in[0,1]$, the inequality $\bar E_k(x,w, y, \theta;\ell) \leq 0$ determines the selection of step-sizes $\{ \tau_k, \sigma_k \}$ such that the condition in line 14 of Algorithm~\ref{alg:ext-multisol} is satisfied.

\begin{algorithm}
\caption{Learning-aware accelerated primal-dual algorithm with backtracking for misspecified saddle point problem with multiple learning solutions}\label{alg:ext-multisol}

\begin{algorithmic}[1]
\Require $c_{\alpha}, c_{\beta} \geq0$, $\rho \in (0,1)$, $\bar \tau, \gamma>0$, $(x_0,w_0,y_0,\theta_0) \in \textbf{dom}\;f \times W \times \textbf{dom}\;h \times \Xi$, $\ell_0 \in \mathbb R^m$
\State Initialize: $(x_{-1}, w_{-1},y_{-1},\theta_{-1};\ell_{-1}) \leftarrow (x_0,w_0,y_0,\theta_0;\ell_0)$, $\tau_0 \leftarrow \bar \tau$, $\sigma_{-1} \leftarrow \gamma \tau_0$, \ma{$\tilde{\theta}_0\gets \theta_0$,  $\zeta_{0} \gets {\theta}_0$, $\tilde{\theta}_{-1} \gets \theta_0$, $\gamma_\theta \le 1/L_{\theta}$} 
\For{$k\geq0$}
\While{stopping criterion not satisfied}
    \State $\sigma_k \gets \gamma \tau_k$, $\eta_k \gets \frac{\sigma_{k-1}}{\sigma_k}$, $\alpha_{k+1}\gets c_{\alpha}/\sigma_k$, $\beta_{k+1}\gets c_{\beta}/\sigma_k$
    \State $s^y_k \gets (1+\eta_k) \nabla_y g_2(x_k, y_k) - \eta_k \nabla_y g_2(x_{k-1}, y_{k-1})$
    \State $s^{\theta}_k \gets (1+\eta_k)\!\left(\nabla_{\theta} g_1(x_k,\theta_k)-w_k\nabla\ell(\theta_k)\right)$
\Statex\hspace{\algorithmicindent}$\qquad \qquad-\ \eta_k\!\left(\nabla_{\theta} g_1(x_{k-1},\theta_{k-1})-w_{k-1}\nabla\ell(\theta_{k-1})\right)$

    \State $y_{k+1} \gets \argmin_{y \in \mathcal{Y}} h(y) - \langle s^y_k, y \rangle + \frac{1}{\sigma_k}\mathbf D_{\mathcal{Y}}(y,y_k)$
    \State $\theta_{k+1} \gets \argmin_{\theta \in \Xi} - \langle s^{\theta}_k, \theta \rangle + \frac{1}{\sigma_k}\mathbf D_{\Theta}(\theta,\theta_k)$
    \State $\zeta_k\gets \tilde{\theta}_k + \frac{k-2}{k+1}(\tilde{\theta}_k-\tilde{\theta}_{k-1})$
    \State $\tilde{\theta}_{k+1} \gets \mathcal{P}_{\Xi}\Big(\zeta_k
          - \gamma_{\theta}\nabla\ell\big(\zeta_k\big)\Big)$
\State $\ell_{k+1} \gets \ell(\tilde{\theta}_{k+1})$
    \State $x_{k+1} \gets \argmin\limits_{x\in\mathcal{X}} \Big\{ f(x) + \langle \nabla_x g_1(x_k,\theta_{k+1})+\nabla_x g_2(x_k,y_{k+1}), x \rangle$
\Statex\hspace{\algorithmicindent}$\qquad \qquad \quad+\ \tfrac{1}{\tau_k}\,\mathbf D_{\mathcal X}(x,x_k) \Big\}$
    \State $w_{k+1} \gets \argmin_{w\geq0} \left\{-\langle \ell(\theta_{k+1})-\ell_{k+1}-\epsilon,w \rangle +\frac{1}{\tau_k} \mathbf{D}_{\mathcal W}(w, w_k)\right\}$
    \If{$\bar E_k(x_{k+1},w_{k+1},y_{k+1},\theta_{k+1};\ell_{k+1})\leq 0$}
        \State \textbf{go to line 21}
    \Else
        \State $\tau_k \gets \rho \tau_k$
        \State \textbf{go to line 4}
    \EndIf
\EndWhile
\State $\tau_{k+1} \gets \tau_k$, $k \gets k+1$
\EndFor
\end{algorithmic}
\end{algorithm}

\begin{assumption}\label{assump:bounded-domains}
We assume that the feasible sets $\Xi$ and $\textbf{dom}\, f$ are bounded. 
\end{assumption}

Recall that by introducing the $\epsilon$-approximation (with $\epsilon>0$), Slater’s condition holds for the relaxed constraint set $\{\theta \in \Xi \mid \ell(\theta) \leq \ell^* + \epsilon \}$. In particular, for any given $x\in\mathcal X$, there exists a Slater point $\theta^{\mathrm S}$ such that the constraint violation is strictly negative.  
This allows us to bound the optimal dual variable $w^*(x)\in\argmin_{w\geq 0}\max_{\theta\in\Xi}\left\{ g_1(x, \theta) - \langle w, \ell(\theta) - \ell^*-\epsilon \rangle \right\}$. 
Following \cite[Lemma~1.1]{nedic201010}, by selecting $\theta^S\in\Xi^*$, we conclude that for any given $x$, 
\[
\|w^*(x)\|_2 \;\le\; \frac{g_1(x,\theta^{\mathrm S}) - \inf_{\theta\in \Xi}g_1(x,\theta)} {\epsilon}.
\]
Taking supremum over $x\in\textbf{dom}\; f$ and noting that $g_1$ is a continuous function and $\textbf{dom}\; f$ and $\Xi$ are compact sets, we conclude that 
$\sup_{x\in\textbf{dom} f}\|w^*(x)\|_2 \;\le\; \sup_{x\in\textbf{dom}\; f}\frac{g_1(x,\theta^{\mathrm S}) - \inf_{\theta\in \Xi}g_1(x,\theta)} {\epsilon}<+\infty$. Now, let us define
\begin{equation*}
    B\triangleq 1+\sup_{x\in\textbf{dom}\; f}\frac{g_1(x,\theta^{\mathrm S}) - \inf_{\theta\in \Xi}g_1(x,\theta)} {\epsilon}<+\infty,
\end{equation*}
and let $W\triangleq [0,B]\subseteq \mathcal{W} = \mathbb{R}$, then problem \eqref{eq:f-ext} can be equivalently rewritten as
\begin{align}\label{eq:f-ext-2}
    &\min_{x \in \mathcal{X},\, w \in  W} \max_{y \in \mathcal{Y},\, \theta \in \Xi} 
    \mathcal{L}(x, w, y, \theta; \ell^*). 
\end{align}
The boundedness of the dual multiplier $w$ corresponding to the learning constraint helps us to derive a Lipschitz constant $L_\ell>0$ for the objective function $\mathcal L(x,w,y,\theta;\ell^*)$ with respect to the parameter $\ell^*$. In particular, one can observe that for any $(x,w,y,\theta)\in \textbf{dom}\;f \times W \times \textbf{dom}\;h \times \Xi$, 
\begin{align}\label{eq:lip-L-ext}
\left|\mathcal{L}(x,w,y,\theta;\ell_1) - \mathcal{L}(x,w,y,\theta;\ell_2)\right| \leq \|w\|_2 \|\ell_1 - \ell_2\|_2\leq B \|\ell_1 - \ell_2\|_2.
\end{align}
In fact, this implies $L_\ell=B=\mathcal O(1/\epsilon)$. 

{Next, we present counterparts of Assumptions~\ref{assump:phi-lip}, \ref{assump:step-size-cond}, and \ref{assump:step-size-cond2} in the form of Assumptions~\ref{assump:lipschitz-block}, \ref{assump:step-size-cond-ext}, and \ref{assump:step-size-cond2-ext}, respectively. 
These alternative formulations mirror the original assumptions but are adapted to the current setting for notational convenience.}

\begin{assumption}\label{assump:lipschitz-block}
Let $\mathcal{X},\mathcal{Y},\Theta,\rma{\mathcal W}$ be finite-dimensional normed spaces with norms
$\|\cdot\|_{\mathcal{X}},\|\cdot\|_{\mathcal{Y}},\|\cdot\|_{\Theta},\rma{\|\cdot\|_{\mathcal W}}$ and corresponding dual norms
$\|\cdot\|_{\mathcal{X}^*},\|\cdot\|_{\mathcal{Y}^*},\|\cdot\|_{\Theta^*},\rma{\|\cdot\|_{\mathcal W^*}}$. Furthermore, we let $\rma{\|\cdot\|_{\mathcal W}}=\|\cdot\|_2$.
Assume $g_1:\mathcal{X}\times\Theta\to\mathbb{R}$, $g_2:\mathcal{X}\times\mathcal{Y}\to\mathbb{R}$, and $\ell:\Theta\to\mathbb{R}$ are continuously differentiable on open sets containing $ \mathrm{\mathbf {dom}}\,f\times\Xi$, $\mathrm{\mathbf {dom}}\,f\times\mathrm{\mathbf {dom}}\,h$, and $\Xi$, respectively. There exist nonnegative constants
 $L^{g_1}_{xx},\,L^{g_1}_{x\theta},\,L^{g_2}_{xx},\,L^{g_2}_{xy},\,L^{g_2}_{yx},\,L^{g_2}_{yy},\,L^{g_1}_{\theta x},\,L^{g_1}_{\theta\theta},\,L^{\ell}_{\theta},\,L^{\ell}_{w}\ \ge 0$,
 such that the following inequalities hold for all $x,\bar x\in \mathrm{\mathbf dom}\,f\subseteq\mathcal{X}$, $y,\bar y\in \mathrm{\mathbf dom}\,h\subseteq\mathcal{Y}$, $\theta,\bar\theta\in\Xi$, and $w,\bar w\in W$:
\begin{enumerate}[label=(\roman*)]
\item $\big\|\nabla_x g_1(x,\theta)-\nabla_x g_1(\bar x,\bar\theta)\big\|_{\mathcal{X}^*}
\le L^{g_1}_{xx}\,\|x-\bar x\|_{\mathcal{X}} + L^{g_1}_{x\theta}\,\|\theta-\bar\theta\|_{\Theta}$, 
\item $\big\|\nabla_x g_2(x,y)-\nabla_x g_2(\bar x,\bar y)\big\|_{\mathcal{X}^*}
\le L^{g_2}_{xx}\,\|x-\bar x\|_{\mathcal{X}} + L^{g_2}_{xy}\,\|y-\bar y\|_{\mathcal{Y}}$, 
\item $\big\|\nabla_y g_2(x,y)-\nabla_y g_2(\bar x,\bar y)\big\|_{\mathcal{Y}^*}
\le L^{g_2}_{yx}\,\|x-\bar x\|_{\mathcal{X}} + L^{g_2}_{yy}\,\|y-\bar y\|_{\mathcal{Y}}$, 
\item $\big\|\nabla_{\theta} g_1(x,\theta)-\nabla_{\theta} g_1(\bar x,\bar\theta)\big\|_{\Theta^*}
\le L^{g_1}_{\theta x}\,\|x-\bar x\|_{\mathcal{X}} + L^{g_1}_{\theta\theta}\,\|\theta-\bar\theta\|_{\Theta}$, 
\item $\big\|{-}\,w\,\nabla\ell(\theta) + \bar w\,\nabla\ell(\bar\theta)\big\|_{\Theta^*}
\le L^{\ell}_{\theta}\,\|\theta-\bar\theta\|_{\Theta} + L^{\ell}_{w}\,\|w-\bar w\|_{\mathcal W}$. 
\end{enumerate}
\end{assumption}

\ma{Before presenting the step-size assumptions, we clarify that these conditions are required to establish the convergence guarantees of the proposed algorithm.}
\begin{assumption}[step-size condition I]\label{assump:step-size-cond-ext}
    There exists a sequence $\{\tau_k,\sigma_k,\eta_k\}_{k\geq 0}$ such that the iterates $\{(x_k,w_k,y_k,\theta_k)\}_{k \geq 0}$, generated by Algorithm \ref{alg:ext-multisol}, along with the step-size sequence, satisfy the following conditions for all $k \geq 0$:
    \begin{enumerate}[label=(\roman*)]
        \item $\bar E_k(x_{k+1},w_{k+1},y_{k+1},\theta_{k+1};\ell_{k+1}) \leq 0$,
        \item $\frac{t_k}{\tau_k} \geq \frac{t_{k+1}}{\tau_{k+1}}$,\quad $\frac{t_k}{\sigma_k} \geq \frac{t_{k+1}}{\sigma_{k+1}}$,\quad $\frac{t_k}{t_{k+1}}= \eta_{k+1}$
    \end{enumerate}
    for some positive sequences $\{t_k, \alpha_k\}_{k \geq 0}$ with $t_0 = 1$, and a nonnegative sequence $\{\beta_k\}_{k\geq 0}$, where $\bar E_k$ is defined in Definition \ref{def:E_k-ext} using the parameters $\{\alpha_k,\beta_k,\eta_k\}$ as specified above. 
\end{assumption}

\begin{assumption}[step-size condition II]\label{assump:step-size-cond2-ext}
    For any $k \geq 0$, the step-sizes $\tau_k$ and $\sigma_k$, along with the momentum parameter $\eta_k$, satisfy $\eta_0 = 1$ as well as condition in Assumption \ref{assump:step-size-cond-ext}-(i), and
    \begin{enumerate}[label=(\roman*)]
        \item $\frac{1}{\tau_k}\geq \max \left\{\frac{(L^{g_2}_{yx})^2+2(L^{g_1}_{\theta x})^2}{\alpha_{k+1}}+L^{g_1}_{xx} + L^{g_2}_{xx}, \frac{2(L^{\ell}_w)^2}{\beta_{k+1}} \right\}$,
        \item $\frac{1}{\sigma_k} \geq \max \left\{ \eta_k (\alpha_k + \beta_k) + \frac{(L^{g_2}_{yy})^2}{\beta_{k+1}}, \eta_k (\alpha_k + \beta_k) +\frac{2(L^{g_1}_{\theta \theta})^2}{\alpha_{k+1}}+\frac{2(L^{\ell}_{\theta})^2}{\beta_{k+1}} \right\}$,
    \end{enumerate}
    for some positive sequences $\{t_k, \alpha_k\}_{k \geq 0}$ with $t_0 = 1$, and a nonnegative sequence $\{\beta_k\}_{k\geq 0}$.
\end{assumption}

\ma{Next, we derive a key property of the parameter sequences $\{\tau_k, \sigma_k, \eta_k\}$ generated by Algorithm~\ref{alg:ext-multisol}, which play a central role in establishing the convergence rate stated in Theorem~\ref{thm:ext-multsol}.}

\begin{lemma}\label{lem:step-size-ext}
     Suppose the sequence $\{ \tau_k, \sigma_k, \eta_k \}_{k \geq 0}$ satisfies the step-size conditions in Assumption \ref{assump:step-size-cond2-ext} for some positive sequences $\{ t_k, \alpha_k \}_{k \geq 0}$, and nonnegative sequence $\{ \beta_k \}_{k \geq 0}$. Let $\{ x_k,w_k, y_k,\theta_k \}$ be the iterate sequence of Algorithm \ref{alg:ext-multisol} corresponding to $\{ \tau_k, \sigma_k, \eta_k \}$. Then $\{ x_k,w_k, y_k,\theta_k \}$ and $\{ \tau_k, \sigma_k, \eta_k \}$ satisfy the condition in Assumption \ref{assump:step-size-cond-ext}-(i) with the same sequences $\{ t_k, \alpha_k, \beta_k \}$.
\end{lemma}
\begin{proof}
    Using Assumption \ref{assump:lipschitz-block}, and Definition \ref{def:Breg-dis}, for any $k\geq 0$, and $(x,w,y,\theta) \in \textbf{dom}\;f \times W \times \textbf{dom}\;h \times \Xi$, we obtain
    \begin{align*}
        \langle \nabla_xg_1(x,\theta)-\nabla_xg_1(x_k,\theta),x-x_k\rangle &\leq \frac{L^{g_1}_{xx}}{2}\|x-x_k\|^2_{\mathcal{X}} \leq L^{g_1}_{xx}\mathbf D_{\mathcal{X}}(x,x_k),\\
        \langle \nabla_xg_2(x,y)-\nabla_xg_2(x_k,y),x-x_k\rangle &\leq \frac{L^{g_2}_{xx}}{2}\|x-x_k\|^2_{\mathcal{X}} \leq L^{g_2}_{xx}\mathbf D_{\mathcal{X}}(x,x_k),\\
        \frac{1}{2}\|\nabla_y g_2(x,y)-\nabla_y g_2(x_k,y)\|^2_{\mathcal{Y}^*}&\leq \frac{(L^{g_2}_{yx})^2}{2}\|x-x_k\|^2_{\mathcal{X}}\leq (L^{g_2}_{yx})^2 \mathbf D_{\mathcal{X}}(x,x_k),\\
        \frac{1}{2}\|\nabla_y g_2(x_k,y)-\nabla_y g_2(x_k,y_k)\|^2_{\mathcal{Y}^*}&\leq \frac{(L^{g_2}_{yy})^2}{2}\|y-y_k\|^2_{\mathcal{Y}}\leq (L^{g_2}_{yy})^2 \mathbf D_{\mathcal{Y}}(y,y_k),\\
        \|\nabla_{\theta} g_1(x,\theta)-\nabla_{\theta} g_1(x_k,\theta)\|^2_{\Theta^*}&\leq (L^{g_1}_{\theta x})^2\|x-x_k\|^2_{\mathcal{X}}\leq 2(L^{g_1}_{\theta x})^2 \mathbf D_{\mathcal{X}}(x,x_k),\\
        \|\nabla_{\theta} g_1(x_k,\theta)-\nabla_{\theta} g_1(x_k,\theta_k)\|^2_{\Theta^*}&\leq (L^{g_1}_{\theta \theta})^2\|\theta-\theta_k\|^2_{\Theta}\leq 2(L^{g_1}_{\theta \theta})^2 \mathbf D_{\Theta}(\theta,\theta_k),\\
        \|-w \nabla \ell(\theta) + w_k \nabla \ell(\theta)\|^2_{\Theta^*}&\leq (L^{\ell}_w)^2 \|w-w_k\|^2_{\mathcal W} \leq 2(L^{\ell}_w)^2 \mathbf D_{\mathcal W}(w,w_k),\\
        \|-w_k \nabla \ell(\theta) + w_k \nabla \ell(\theta_k)\|^2_{\Theta^*}&\leq (L^{\ell}_\theta)^2 \|\theta-\theta_k\|^2_{\Theta} \leq 2(L^{\ell}_\theta)^2 \mathbf D_{\Theta}(\theta,\theta_k).
        \end{align*}
        
    By evaluating the above inequality at $(x,w,y,\theta;\ell)=(x_{k+1},w_{k+1},y_{k+1},\theta_{k+1};\ell_{k+1})$, and using Assumption \ref{assump:step-size-cond2-ext} along with the nonnegativity of Bregman functions, we can show
    \begin{align*}
        \bar E_k &\triangleq \bar E_k(x_{k+1},w_{k+1},y_{k+1},\theta_{k+1};\ell_{k+1})\\
        &\leq \big(L^{g_1}_{xx}+L^{g_2}_{xx}+\frac{L^{g_2}_{yx}}{\alpha_{k+1}}+\frac{2(L^{g_1}_{\theta x})^2}{\alpha_{k+1}}-\frac{1}{\tau_k} \big)\mathbf D_{\mathcal{X}}(x_{k+1},x_k)\\
        &\quad + \big(\frac{(L^{g_2}_{yy})^2}{\beta_{k+1}}+\eta_k(\alpha_k+\beta_k)-\frac{1}{\sigma_k} \big)\mathbf D_{\mathcal{Y}}(y_{k+1},y_k)\\
        &\quad + \big(\frac{2(L^{g_1}_{\theta \theta})^2}{\alpha_{k+1}}+\frac{2(L^{\ell}_{\theta})^2}{\beta_{k+1}}+\eta_k(\alpha_k+\beta_k)-\frac{1}{\sigma_k} \big)\mathbf D_{\Theta}(\theta_{k+1},\theta_k)\\
        &\quad + \big(\frac{2(L^{\ell}_{w})^2}{\beta_{k+1}}-\frac{1}{\tau_k} \big)\mathbf D_{\mathcal W}(w_{k+1},w_k).
    \end{align*}
\end{proof}

We now present the iteration complexity analysis of the proposed Learning-aware approach in the setting with multiple learning solutions in the following theorem and the corollary afterwards.

\begin{definition}\label{def:gap-function-ext}
    \ma{Let $\bar{S} \triangleq \textbf{dom}\, f \times W\times \textbf{dom}\, h \times \Xi$, \rma{which is bounded under Assumption~\ref{assump:bounded-domains} and contains a
    saddle point solution of \eqref{eq:f-ext-2}}, 
    and let $s \triangleq [x^\top, w^\top, y^\top,\theta^\top]^\top$. 
    The gap function for Problem \eqref{eq:f-ext-2} is defined as}
    \rma{\begin{align*}
        \mathcal{G}_{\bar{S}}({\bar s_K}) \triangleq \sup_{s\in \bar{S}} \left\{
\mathcal L(\bar x_K,\bar w_K,y,\theta;\ell^*)
-\mathcal L(x,w,\bar y_K,\bar\theta_K;\ell^*)
\right\},
    \end{align*}
where $\bar s_K\triangleq [\bar x_K^\top,\bar w_K^\top,\bar y_K^\top,\bar\theta_K^\top]^\top$ denotes the weighted average of the iterates $\{s_{k}\}_{k=0}^{K-1}$.}


\end{definition}
Now we are ready to present the main result of this section.

\begin{theorem}\label{thm:ext-multsol}
    Suppose Assumptions \ref{assump:bounded-domains} and \ref{assump:lipschitz-block} hold and {let $s\triangleq [x^\top,w^\top, y^\top,\theta^\top]^\top$}. Let $\{x_k,w_k,y_k, \theta_k, \ell_k\}_{k\geq 0}$ be the sequence generated by Algorithm \ref{alg:ext-multisol}, using a parameter sequence $\{\tau_k,\sigma_k,\eta_k\}_{k\geq 0}$ that satisfies Assumptions \ref{assump:step-size-cond-ext} or \ref{assump:step-size-cond2-ext}. \ma{Then, for any 
    $K\geq 1$,}

    \begin{align}\label{eq:missp-sup-ext}
    \mathcal{G}_{\bar{S}}(\overline{s}_K) &\leq \frac{1}{T_K} \sup_{s\in \bar{S}} \Big(\frac{1}{\tau_0} \mathbf D_{\mathcal{X}}(x,x_0) + \frac{1}{\tau_0} \mathbf D_{\mathcal W}(w,w_0) + \frac{1}{\sigma_0} \mathbf D_{\mathcal{Y}}(y,y_0) + \frac{1}{\sigma_0} \mathbf D_{\Theta}(\theta,\theta_0) \Big) \nonumber\\
    & \quad  + \frac{2L_{\ell}}{T_K} \sum_{k=0}^{K-1}t_k\|\ell_{k+1}-\ell^*\|_2,
\end{align}
\ma{where $\bar{s}_K \triangleq [\bar{x}_K^\top,\bar{w}_K^\top,\bar{y}_K^\top,\bar{\theta}_K^\top]^\top$ denotes the weighted average of the iterates 
$\{s_k\}_{k=0}^{K-1}$, i.e., $\bar{s}_K \triangleq \frac{1}{T_K} \sum_{k=0}^{K-1} t_k s_{k+1}$, and} $T_K = \sum_{k=0}^{K-1}t_k$.
    
\end{theorem}
\begin{proof}
The proof of Theorem \ref{thm:ext-multsol} follows similar steps to the convergence analysis in the proof of Theorem \ref{thm:missp-approach} for the Learning-aware APD approach, with some modifications. For completeness, we provide the full proof of Theorem \ref{thm:ext-multsol} in Appendix \ref{AppndxC}.

\end{proof}






\begin{corollary}\label{cor:ext-multisol-rate}
Under the premises of Theorem~\ref{thm:ext-multsol}, for any $K\geq 1$, let $\epsilon= 1/\sqrt{K}$, it holds that:
\begin{align}
\mathcal{G}_{\bar{S}}(\overline{s}_K) &\leq \frac{1}{T_K} \left(\frac{1}{\tau_0}\mathbf R_{x} + \frac{1}{\tau_0}\mathbf R_{ w} + \frac{1}{\sigma_0}\mathbf R_{y} + \frac{1}{\sigma_0}\mathbf R_{\theta} \right) + \frac{c_0c_1 \pi^2 \sqrt{K}}{3T_K},\label{eq:gap-rate-mult}\\
\ell(\bar \theta_K)-\ell^*&\leq \frac{1}{T_K} \left(\frac{1}{\tau_0}\mathbf R_{x} + \frac{1}{\tau_0}\mathbf R_{ w} + \frac{1}{\sigma_0}\mathbf R_{y} + \frac{1}{\sigma_0}\mathbf R_{\theta} \right) + \frac{c_0c_1 \pi^2 \sqrt{K}}{3T_K} + \frac{1}{\sqrt{K}},\label{eq:feasibility-rate-mult}
\end{align}
where $T_K = \sum_{k=0}^{K-1}t_k=\Omega(K)$, 
$\mathbf R_{x} \;\triangleq\;
\sup_{x,\bar x \in \bar{ X}} \mathbf D_{\mathcal X}(x,\bar x)<\infty$,
$\mathbf R_{ w} \;\triangleq\;
\sup_{w,\bar w \in \bar{ W}} \mathbf D_{\mathcal W}(w,\bar w)<\infty$, $\mathbf R_{y} \;\triangleq\;
\sup_{y,\bar y \in \bar{ Y}} \mathbf D_{\mathcal Y}(y,\bar y)<\infty$, and $\mathbf R_{\theta} \;\triangleq\;
\sup_{\theta,\bar \theta \in \bar{\Xi}} \mathbf D_{\Theta}(\theta,\bar \theta)<\infty$. 
\end{corollary}

\begin{proof}
From Theorem~\ref{thm:ext-multsol}, we have that
\begin{align}\label{eq:rep-gap-ext}
    \mathcal{G}_{\bar{S}}(\overline{s}_K) &\leq \frac{1}{T_K} \sup_{s\in \bar{S}} \Big(\frac{1}{\tau_0} \mathbf D_{\mathcal{X}}(x,x_0) + \frac{1}{\tau_0} \mathbf D_{\mathcal W}(w,w_0) + \frac{1}{\sigma_0} \mathbf D_{\mathcal{Y}}(y,y_0) + \frac{1}{\sigma_0} \mathbf D_{\Theta}(\theta,\theta_0) \Big) \nonumber\\
    & \quad  + \frac{2L_{\ell}}{T_K} \sum_{k=0}^{K-1}t_k\|\ell_{k+1}-\ell^*\|_2,
\end{align}
%
Under the accelerated projected gradient descent update, we have that 
$\|\ell_{k+1} - \ell^*\|_2 \leq \tfrac{c_1}{(k+1)^2}$ for some constant $c_1 > 0$, and $L_\ell=c_0/\epsilon=c_0\sqrt{K}$ for some $c_0>0$. Hence
\begin{align*}
\frac{2L_{\ell}}{T_K}\sum_{k=0}^{K-1}t_k\|\ell_{k+1}-\ell^*\|_2
&\leq\frac{2c_0c_1}{T_K \epsilon}\sum_{k=1}^K\frac{1}{k^2} \leq \frac{2c_0c_1}{T_K \epsilon} \cdot \frac{\pi^2}{6}= \frac{c_0c_1 \pi^2 }{3T_K \epsilon}.
\end{align*}

Combining the above bound with \eqref{eq:rep-gap-ext}, we obtain
\begin{align}\label{eq:ext-gap1}
    \mathcal{G}_{\bar{S}}(\overline{s}_K) \;\leq \frac{1}{T_K} \left(\frac{1}{\tau_0}\mathbf R_{x} + \frac{1}{\tau_0}\mathbf R_{ w} + \frac{1}{\sigma_0}\mathbf R_{y} + \frac{1}{\sigma_0}\mathbf R_{\theta} \right) + \frac{c_0c_1 \pi^2 }{3T_K \epsilon}.
\end{align}
Therefore, selecting $\epsilon =1/\sqrt{K}$, we obtain the desired result in \eqref{eq:gap-rate-mult}.


Moreover, using the earlier result in \eqref{eq:ext-gap1}, and Definition \ref{def:gap-function-ext}, for any $[x,w,y,\theta]\in \bar{S}$, we have:
\begin{align*}
     &\mathcal{L}(\bar x_K, \bar w_K,y,\theta;\ell^*)-\mathcal{L}(x, w,\bar y_K,\bar \theta_K;\ell^*) \\
     & \qquad \leq \frac{1}{T_K} \left(\frac{1}{\tau_0}\mathbf R_{x} + \frac{1}{\tau_0}\mathbf R_{ w} + \frac{1}{\sigma_0}\mathbf R_{y} + \frac{1}{\sigma_0}\mathbf R_{\theta} \right) + \frac{c_0c_1 \pi^2 }{3T_K \epsilon}.
\end{align*}
Now, for any solution tuple $(x^*,w^*,y^*,\theta^*)$ of \eqref{eq:f-ext}, by letting $s=[x^*, w^*+1, y^*, \theta^*] \in \bar{S}$, one can observe that
\begin{align*}
     &\mathcal{L}(\bar x_K, \bar w_K,y^*,\theta^*;\ell^*)-\mathcal{L}(x^*, w^*+1,\bar y_K,\bar \theta_K;\ell^*)\\
     &\quad=\mathcal{L}(\bar x_K, \bar w_K,y^*,\theta^*;\ell^*)-\mathcal{L}(x^*, w^*,\bar y_K,\bar \theta_K;\ell^*) + (\ell(\bar \theta_K)-\ell^*-\epsilon)\\
     &\quad\geq \ell(\bar \theta_K)-\ell^*-\epsilon,
\end{align*}
where in the last inequality we used the fact that $(x^*,w^*,y^*,\theta^*)$ is an SP solution of \eqref{eq:f-ext}. 
Therefore, we obtain
\begin{align*}
\ell(\bar \theta_K)-\ell^*\leq \frac{1}{T_K} \left(\frac{1}{\tau_0}\mathbf R_{x} + \frac{1}{\tau_0}\mathbf R_{ w} + \frac{1}{\sigma_0}\mathbf R_{y} + \frac{1}{\sigma_0}\mathbf R_{\theta} \right) + \frac{c_0c_1 \pi^2 }{3T_K \epsilon} + \epsilon.
\end{align*}
Finally, selecting $\epsilon = 1/\sqrt{K}$ to minimize the right-hand-side with respect to $\epsilon$, we obtain the desired result in \eqref{eq:feasibility-rate-mult}. 
\end{proof}

\begin{remark}
The primal–dual method proposed in Algorithm~\ref{alg:ext-multisol} is designed for misspecified saddle point problems with multiple learning solutions in \eqref{eq:sp-ext}, which incorporates an adaptive step-size rule. Corollary~\ref{cor:ext-multisol-rate} establishes a convergence rate of $\mathcal{O}(1/\sqrt{K})$ for finding an SP solution of \eqref{eq:f-ext}, along with guarantees on the suboptimality of the associated learning problem. To the best of our knowledge, this is the first non-asymptotic convergence result for this class of problems. 
\end{remark}

\smallskip


\smallskip


\section{Numerical Experiments}\label{sec:num}
In this section, we present a set of numerical experiments to assess the effectiveness of our proposed algorithms in solving a misspecified SP problem. We benchmark their performance against two existing methods: the inexact parametric augmented Lagrangian method (IPALM) \cite{aybat2021analysis} and a variational inequality (VI) \cite{jiang2016solution} based method. 
In particular, we apply our methods to the misspecified portfolio optimization problem described in Section~\ref{subsec:portfolio}.

Consider the Markowitz portfolio optimization problem discussed in \cite{aybat2021analysis}. We formulate it as a misspecified SP problem, as described in Section~\ref{subsec:portfolio}:
\begin{equation}\label{eq:exp-portfolio}
\begin{aligned}
    &(Optimization): ~\min_{x\in \mathcal{X}} \max_{y\geq 0} \left\{ \frac{1}{2}x^\top \Sigma^* x - \kappa \mu^\top x + y^\top(Ax - b) \right\}, \\
    &(Learning):~\Sigma^* \in \argmin_{\Sigma \in \mathbb{S}^n} \max_{W\in \mathbb{S}_+^n} \left\{ \frac{1}{2}\|\Sigma - S\|_F^2 + v\|\Sigma\|_1 - \mathrm{Tr}\left(W^\top(\Sigma - \epsilon I)\right) \right\}.
\end{aligned}
\end{equation}
Here, $x \in \mathbb{R}^n$ denotes the portfolio vector, where $x_i$ represents the proportion of asset $i$ maintained in the portfolio. The set $\mathcal{X} \triangleq \{x \in \mathbb{R}^n : \sum_{i=1}^n x_i = 1,\ x \geq 0\}$ enforces a fully-invested, no-short-selling portfolio. The parameter $\kappa > 0$ governs the trade-off between risk and expected return, where $\mu \in \mathbb{R}^n$ denotes the expected return vector. The matrix $A \in \mathbb{R}^{s \times n}$ and vector $b \in \mathbb{R}^s$ define the sector constraints, with $A_{ji} = 1$ if asset $i$ belongs to sector $j$, and $b_j$ representing the maximum allowed investment in sector $j$. The dual variable $y \in \mathbb{R}^s_+$ corresponds to these sector constraints.

The matrix $\Sigma^* \in \mathbb{S}^n$ represents the estimated covariance of asset returns. It is learned from the sample covariance matrix $S \in \mathbb{R}^{n \times n}$ using a sparse covariance selection problem. The term $\|\Sigma\|_1$ denotes the $\ell_1$ norm of the off-diagonal elements of $\Sigma$, $\|\cdot\|_F$ is the Frobenius norm, and $v > 0$ is a regularization parameter promoting sparsity. The constraint $\Sigma \succeq \epsilon I$ ensures positive definiteness, enforced via a symmetric Lagrange multiplier matrix $W \in \mathbb{S}_+^n$.

We conduct experiments on both real and synthetic datasets. For the real data, we use the NASDAQ-100 ($n = 82$, $T = 596$) and Dow Jones ($n = 28$, $T = 1363$) datasets from \cite{bruni2016real}, where $n$ denotes the number of assets and $T$ the number of weekly return observations. The datasets have been preprocessed to remove errors and are adjusted for dividends and stock splits. Using these datasets, we compute the true mean return vector $\mu^0$ (the row-wise mean of the return matrix) and the true covariance matrix $\Sigma^0$ (the column-wise covariance of the returns).

For the synthetic dataset, we consider $n = 800$ assets whose joint returns $\mathcal{R} = \{\mathcal{R}_i\}_{i=1}^n$ follow a multivariate Normal distribution $\mathcal{N}(\mu^0, \Sigma^0)$. The true mean vector $\mu^0$ is drawn uniformly from the hypercube $[-1,1]^n \subset \mathbb{R}^n$, while the true covariance matrix $\Sigma^0 = [\sigma_{ij}]$, where $\sigma_{ij} = \max\{1 - \frac{|i-j|}{10}, 0\}$. We then generate $p = n/2$ i.i.d.\ samples from $\mathcal{N}(\mu^0, \Sigma^0)$, and compute the empirical covariance matrix $S$ from these samples. The \textit{(Learning)} problem is solved to estimate the parameter matrix $\Sigma^*$ using the structured covariance selection (SCS) formulation, with parameters set to $v = 0.4$ and $\kappa = 0.1$.

To evaluate algorithmic performance, we compare infeasibility and suboptimality measures over the course of iterations. Additionally, we track and plot the standardized norm of the difference between successive iterates for the parameter estimate.
All methods are run for a maximum of $K = 10^3$ iterations, with the number of market sectors set to $s = 10$. Step-sizes for each method are chosen according to their respective theoretical convergence requirements and are fine-tuned to ensure practical efficiency.
\begin{figure}[H]
  \centering
  \newcommand{\subfigheight}{4cm}

  \begin{minipage}{\textwidth}
    \centering
    \textbf{NASDAQ-100 (real data)}

    \begin{subfigure}[t]{0.32\textwidth}
      \centering
      \includegraphics[width=\linewidth,height=\subfigheight]{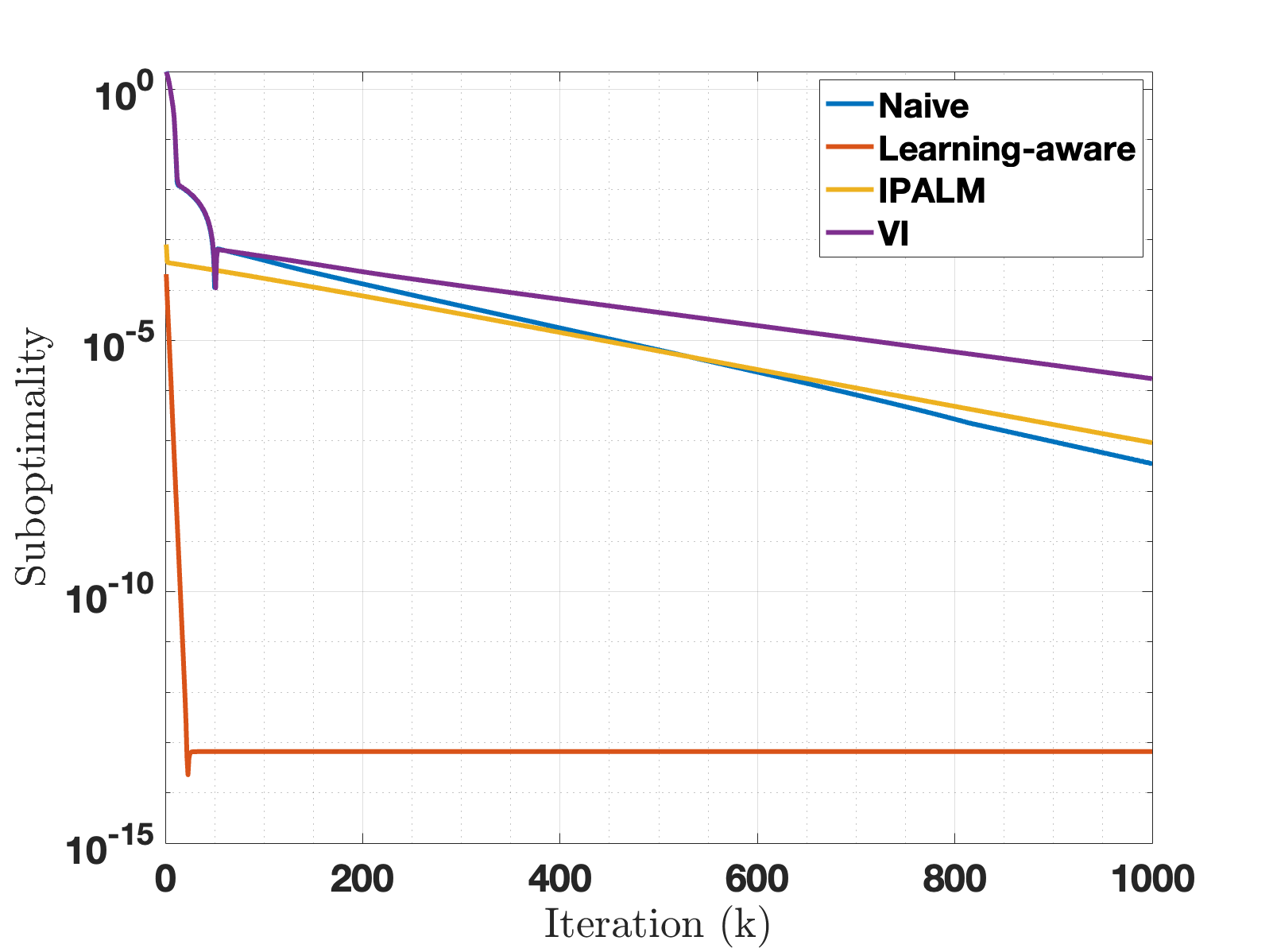}
      \label{fig:comb-I-a}
    \end{subfigure}\hfill
    \begin{subfigure}[t]{0.32\textwidth}
      \centering
      \includegraphics[width=\linewidth,height=\subfigheight]{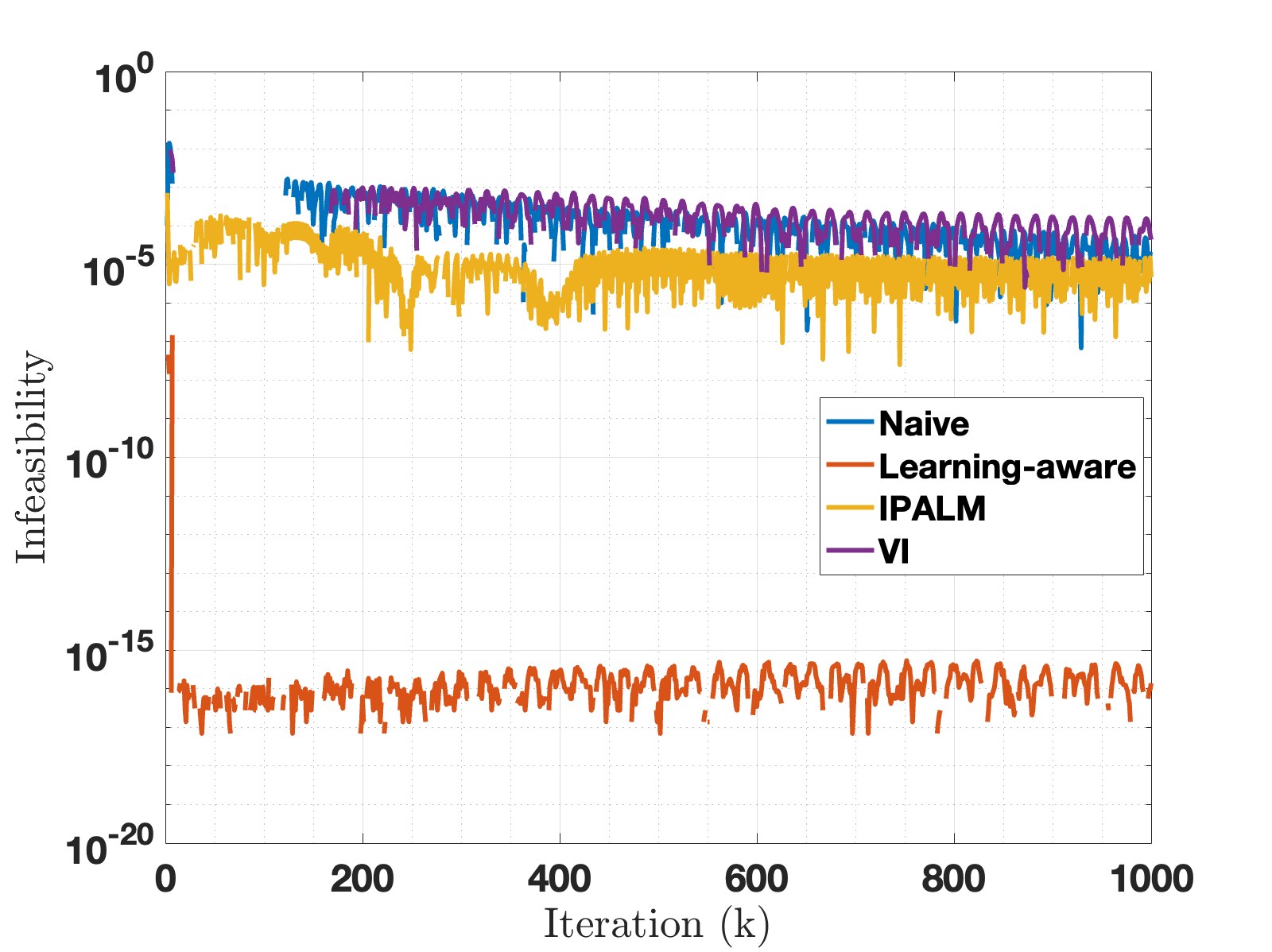}
      \label{fig:comb-I-b}
    \end{subfigure}\hfill
    \begin{subfigure}[t]{0.32\textwidth}
      \centering
      \includegraphics[width=\linewidth,height=\subfigheight]{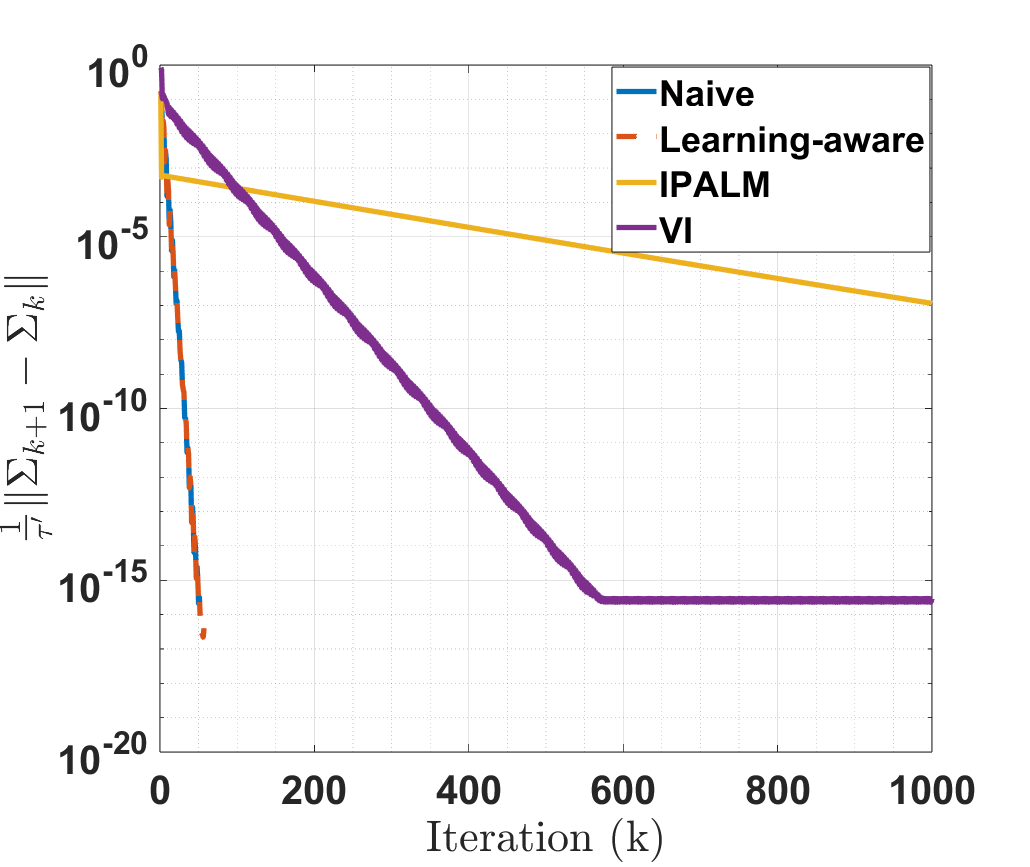}
      \label{fig:comb-I-c}
    \end{subfigure}
  \end{minipage}

  \vspace{1mm}

  \begin{minipage}{\textwidth}
    \centering
    \textbf{Dow Jones (real data)}

    \begin{subfigure}[t]{0.32\textwidth}
      \centering
      \includegraphics[width=\linewidth,height=\subfigheight]{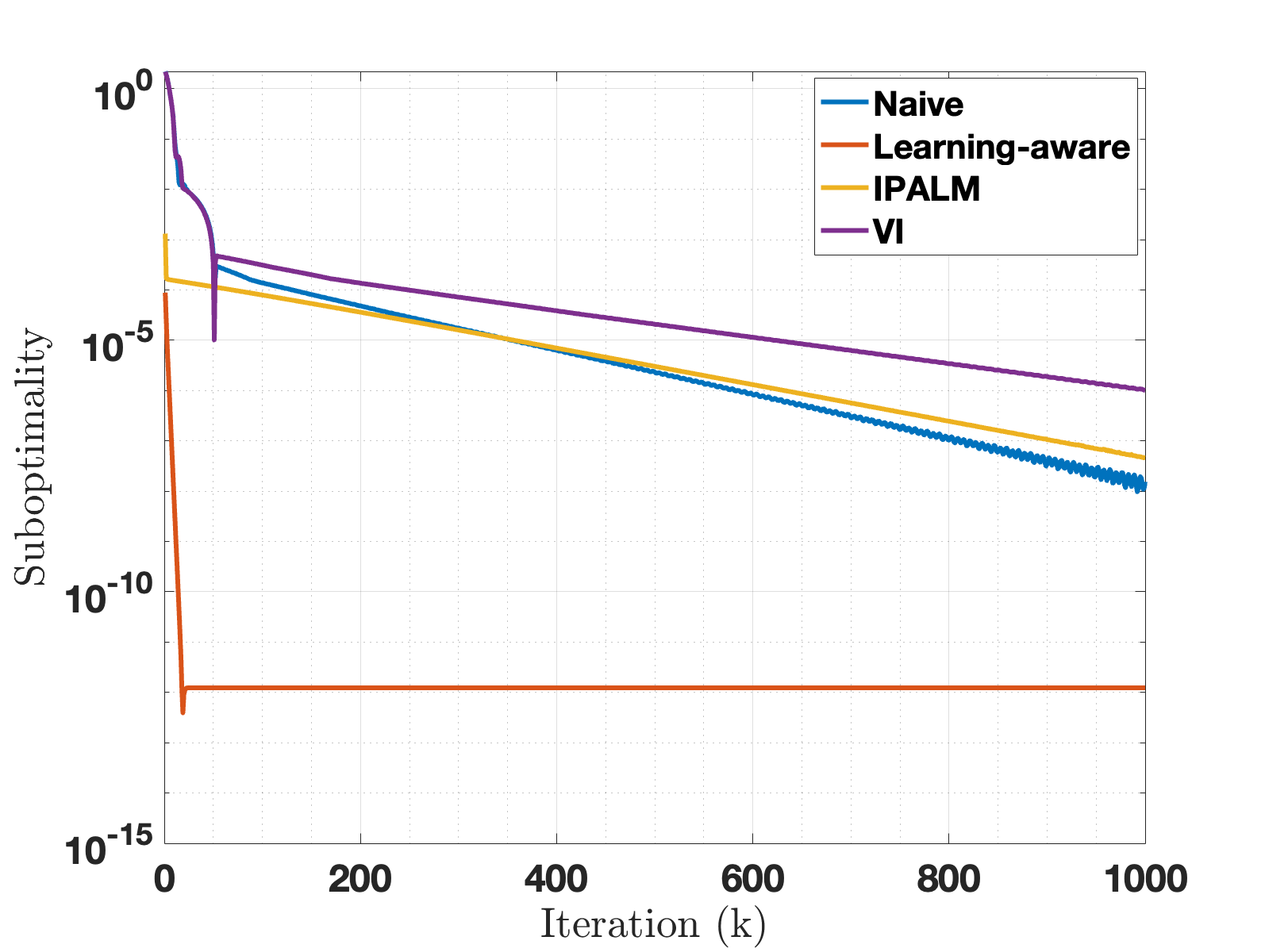}
      \label{fig:comb-II-a}
    \end{subfigure}\hfill
    \begin{subfigure}[t]{0.32\textwidth}
      \centering
      \includegraphics[width=\linewidth,height=\subfigheight]{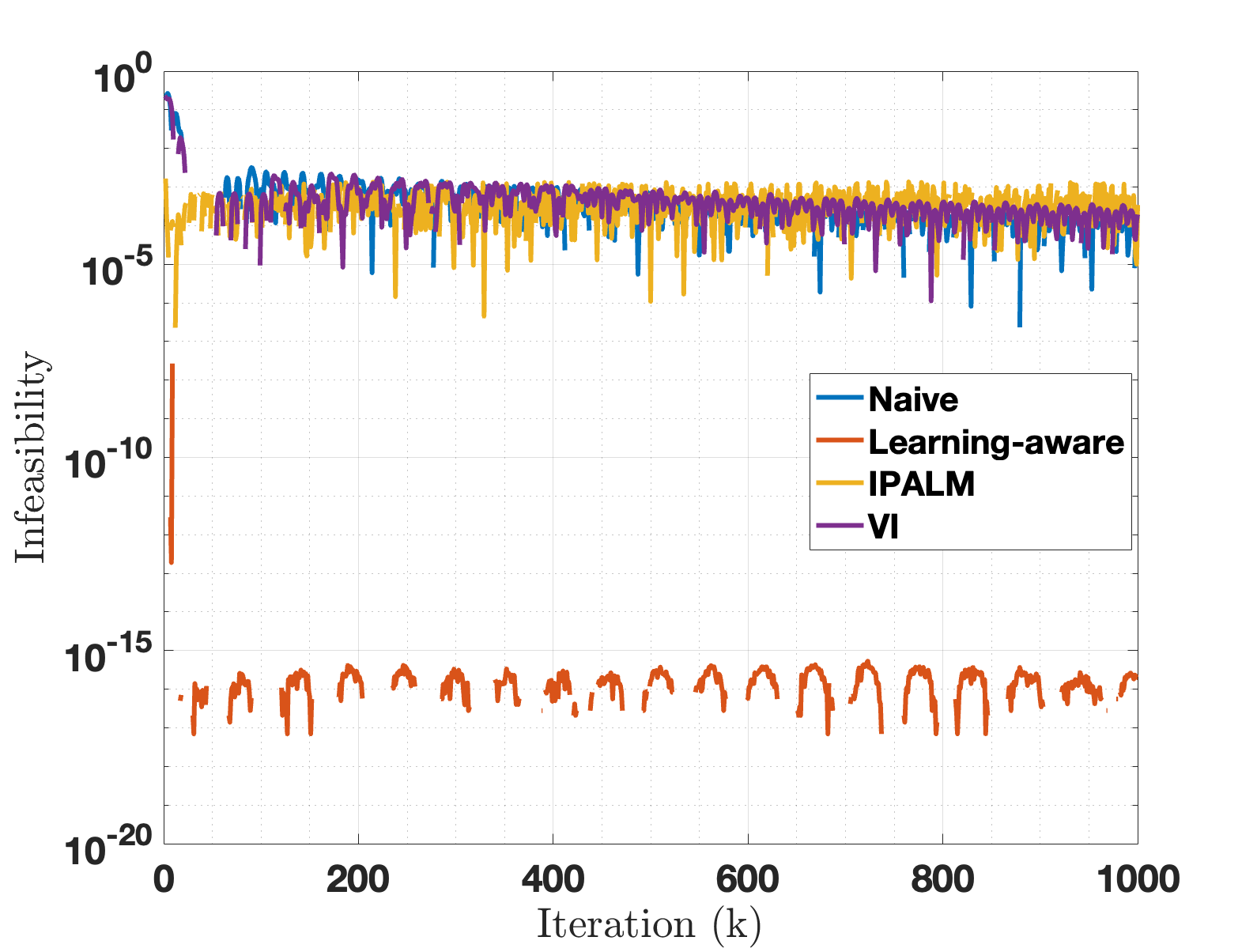}
      \label{fig:comb-II-b}
    \end{subfigure}\hfill
    \begin{subfigure}[t]{0.32\textwidth}
      \centering
      \includegraphics[width=\linewidth,height=\subfigheight]{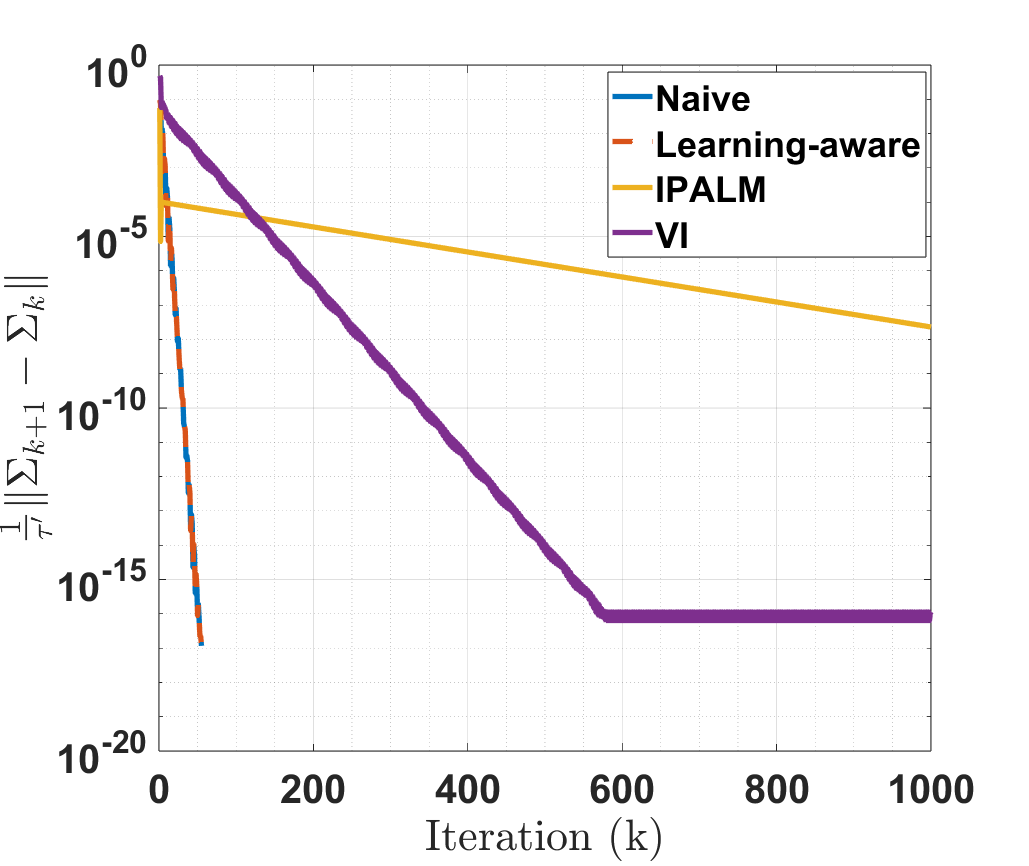}
      \label{fig:comb-II-c}
    \end{subfigure}
  \end{minipage}

  \vspace{1mm}

  \begin{minipage}{\textwidth}
    \centering
    \textbf{Synthetic data}

    \begin{subfigure}[t]{0.32\textwidth}
      \centering
      \includegraphics[width=\linewidth,height=\subfigheight]{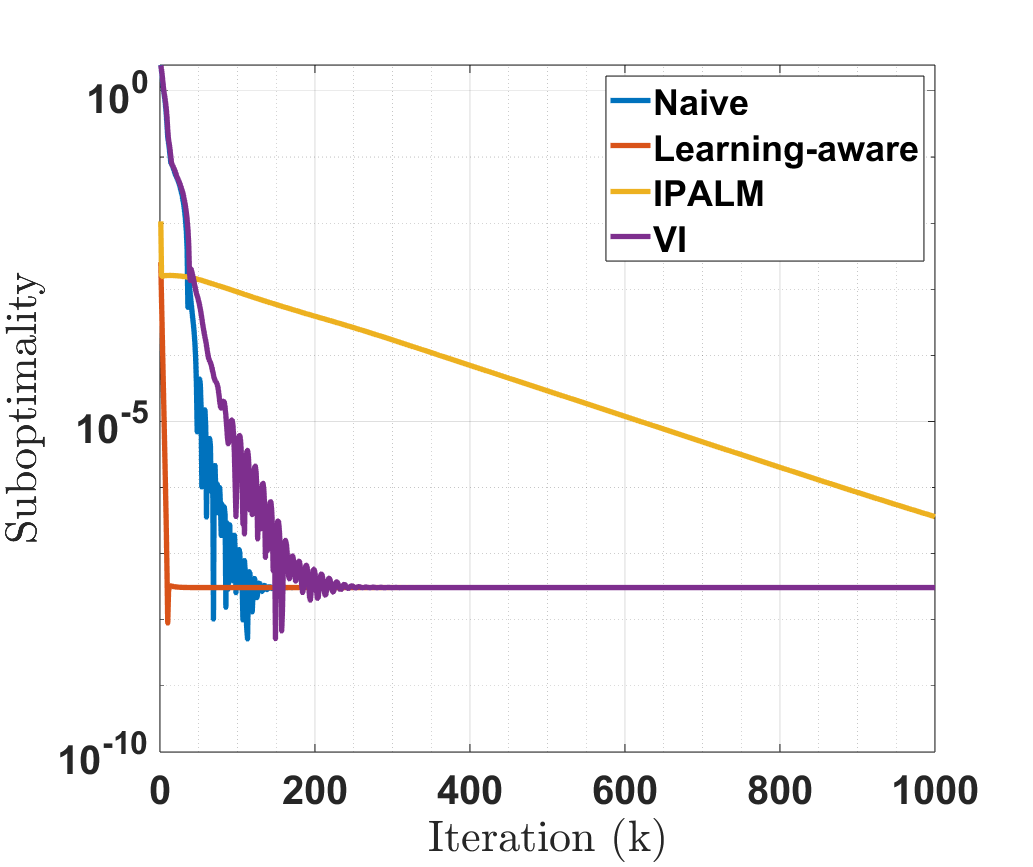}
      \caption{}
      \label{fig:comb-III-a}
    \end{subfigure}\hfill
    \begin{subfigure}[t]{0.32\textwidth}
      \centering
      \includegraphics[width=\linewidth,height=\subfigheight]{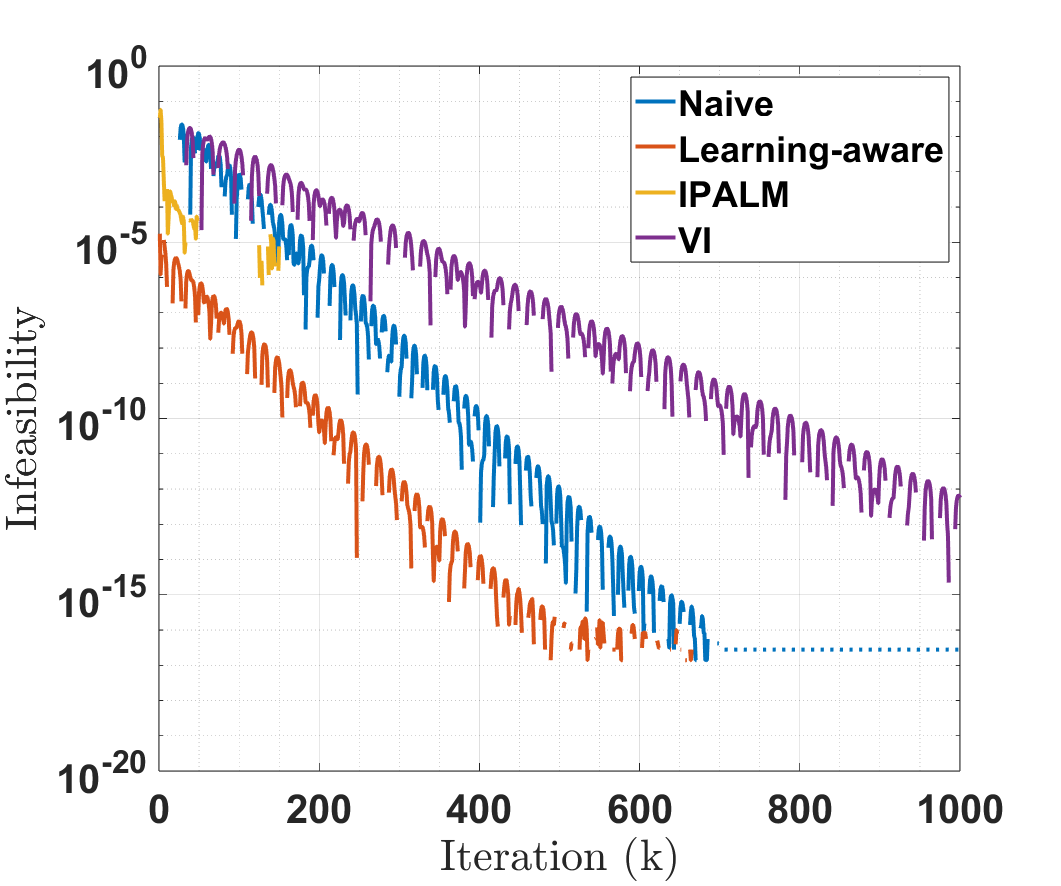}
      \caption{}
      \label{fig:comb-III-b}
    \end{subfigure}\hfill
    \begin{subfigure}[t]{0.32\textwidth}
      \centering
      \includegraphics[width=\linewidth,height=\subfigheight]{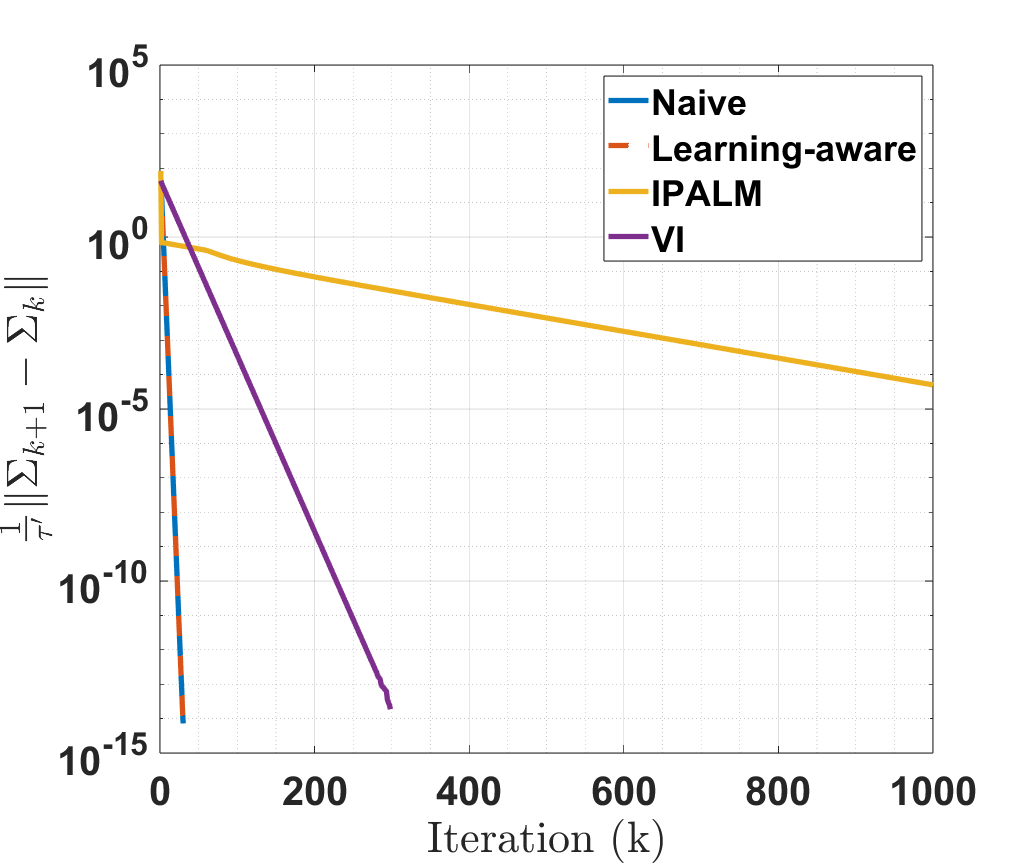}
      \caption{}
      \label{fig:comb-III-c}
    \end{subfigure}
  \end{minipage}

  \caption{Comparison of our proposed Naive (blue) and Learning-aware (red) with IPALM (yellow) and VI (purple) across three datasets and three metrics over iteration $k$. 
  From top to bottom: NASDAQ-100, Dow Jones, and Synthetic datasets.
  From left to right, the figures correspond to suboptimality, infeasibility, and learning solution metrics.}
  \label{fig:combined-3x3}
\end{figure}

The performance of the algorithms on both real and synthetic datasets is illustrated in Figure~\ref{fig:combined-3x3}. Our proposed algorithms, Naive and Learning-aware APD, demonstrate significantly faster convergence in the infeasibility and suboptimality measures compared to the IPALM and VI methods. In particular, the Learning-aware method, which incorporates a backtracking line search strategy, achieves the most rapid convergence among all approaches, outperforming the Naive method in both infeasibility and suboptimality measures.

It is also important to note that the parameter updates for the \textit{(Learning)} problem (i.e., the covariance matrix $\Sigma$) are carried out using the same learning algorithm (Algorithm~\ref{alg:APD-learning}) in both proposed methods. As a result, the convergence trajectories for $\Sigma$ are identical, leading to overlapping curves in the last columns of Figure~\ref{fig:combined-3x3}.

\rma{%
To further understand the behavior of the proposed method, we examine the separate roles of its two ingredients, the parameter-aware momentum and the backtracking line search, and the influence of the step-size choice, using the Dow Jones dataset under a common initialization, data, and learning subroutine. Figure~\ref{fig:ablation} compares three variants: the Naive approach, the Learning-aware approach with fixed step-sizes (line search disabled), and the full Learning-aware approach with backtracking. The parameter-aware momentum already reduces the infeasibility by about two orders of magnitude relative to the Naive method with identical constant step-sizes, while leaving the suboptimality essentially unchanged. Enabling the backtracking line search then improves both measures
substantially, reflecting the adaptive step-sizes. In other words, the modified momentum controls the influence of the learned parameter within the updates, which improves feasibility and, in addition, makes the computable backtracking test of Definition~\ref{def:E_k} available.}

\rma{Figure~\ref{fig:sensitivity} reports the effect of the step-size choice on the fixed-step methods. We parameterize the primal and dual step-sizes by a common scaling $\nu>0$, writing $\tau=\nu\,\tau_0$ and $\sigma=\nu\,\sigma_0$, where
$\nu>0$ is a tuning parameter varied to probe the best performance of each method, and we hold the momentum parameter $\eta$ fixed. Here $(\tau_0,\sigma_0,\eta)$ are the base step-sizes, set to $(\tau_0,\sigma_0,\eta)=(0.5,\,0.1,\,0.2)$, and we sweep $\nu\in\{10^{0},10^{-1},10^{-2}\}$. The relative behavior of the methods is preserved across all values of $\nu$: the Naive and fixed-step Learning-aware trajectories remain close on suboptimality while the Learning-aware approach retains its feasibility advantage, and smaller $\nu$ slows both methods uniformly. In contrast, the Learning-aware approach with backtracking requires no such choice, as it determines its step-sizes automatically without knowledge of the problem's Lipschitz constants.}

\begin{figure}[H]
\centering
\begin{subfigure}[t]{0.48\linewidth}
  \centering
  \includegraphics[width=\linewidth]{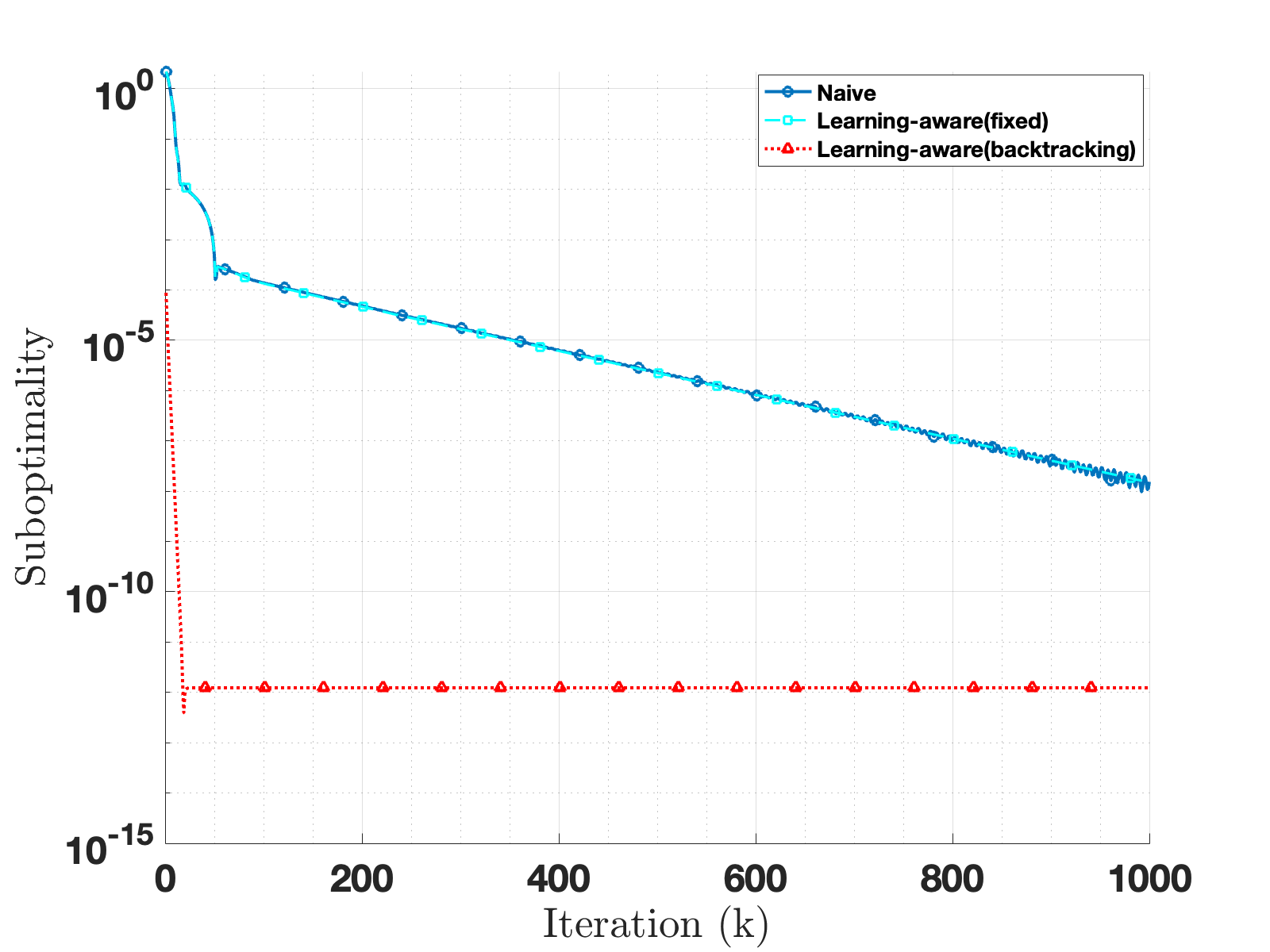}
  \label{fig:ablation-subopt}
\end{subfigure}\hfill
\begin{subfigure}[t]{0.48\linewidth}
  \centering
  \includegraphics[width=\linewidth]{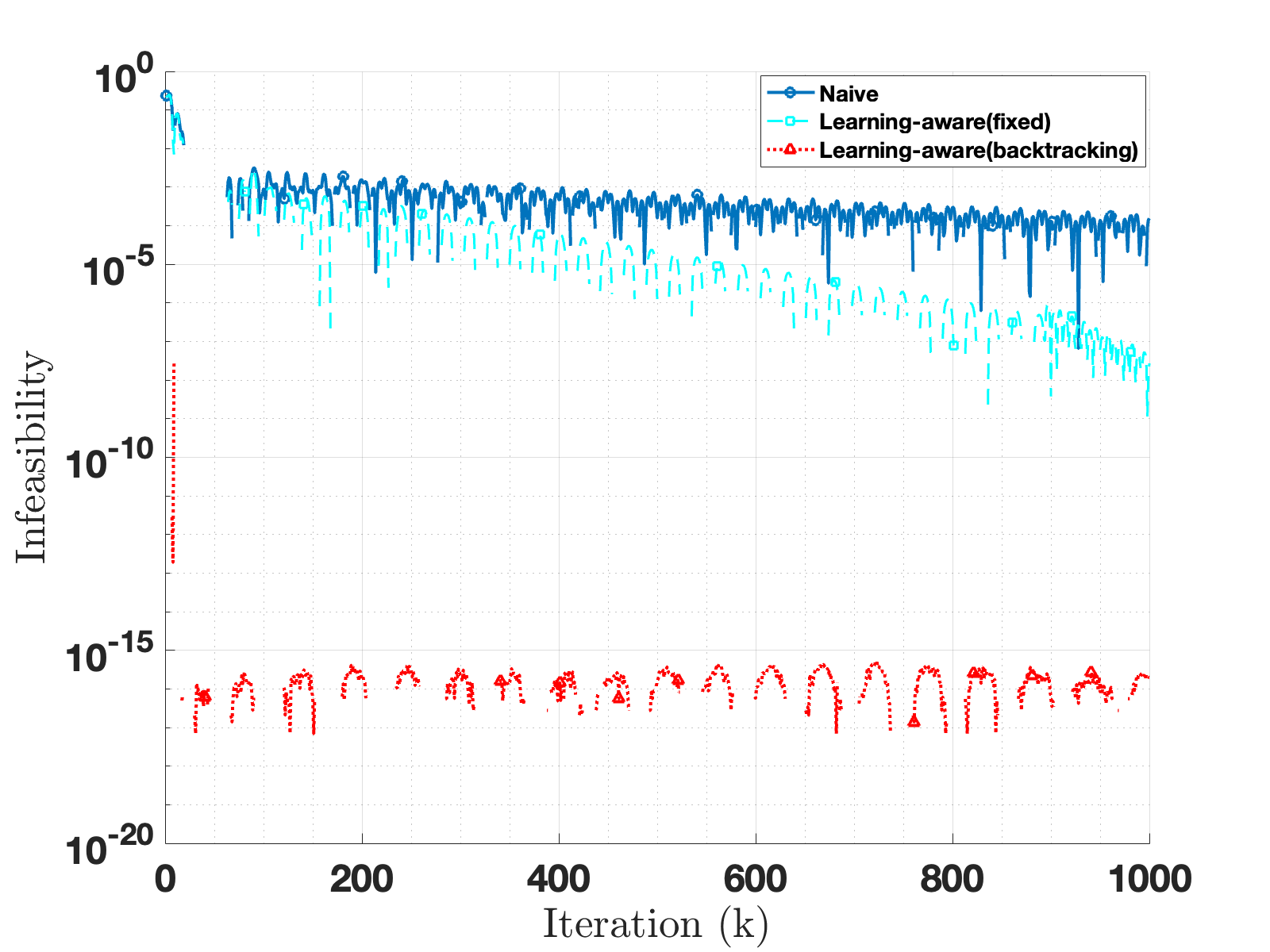}
  \label{fig:ablation-infeas}
\end{subfigure}
\caption{Separate effects of the parameter-aware momentum and the backtracking line search on the Dow Jones dataset, in terms of suboptimality (left) and infeasibility (right): Naive and Learning-aware approaches with fixed step-sizes, and the full Learning-aware approach with backtracking.}
\label{fig:ablation}
\end{figure}

\begin{figure}[H]
\centering
\begin{subfigure}[t]{0.48\linewidth}
  \centering
  \includegraphics[width=\linewidth]{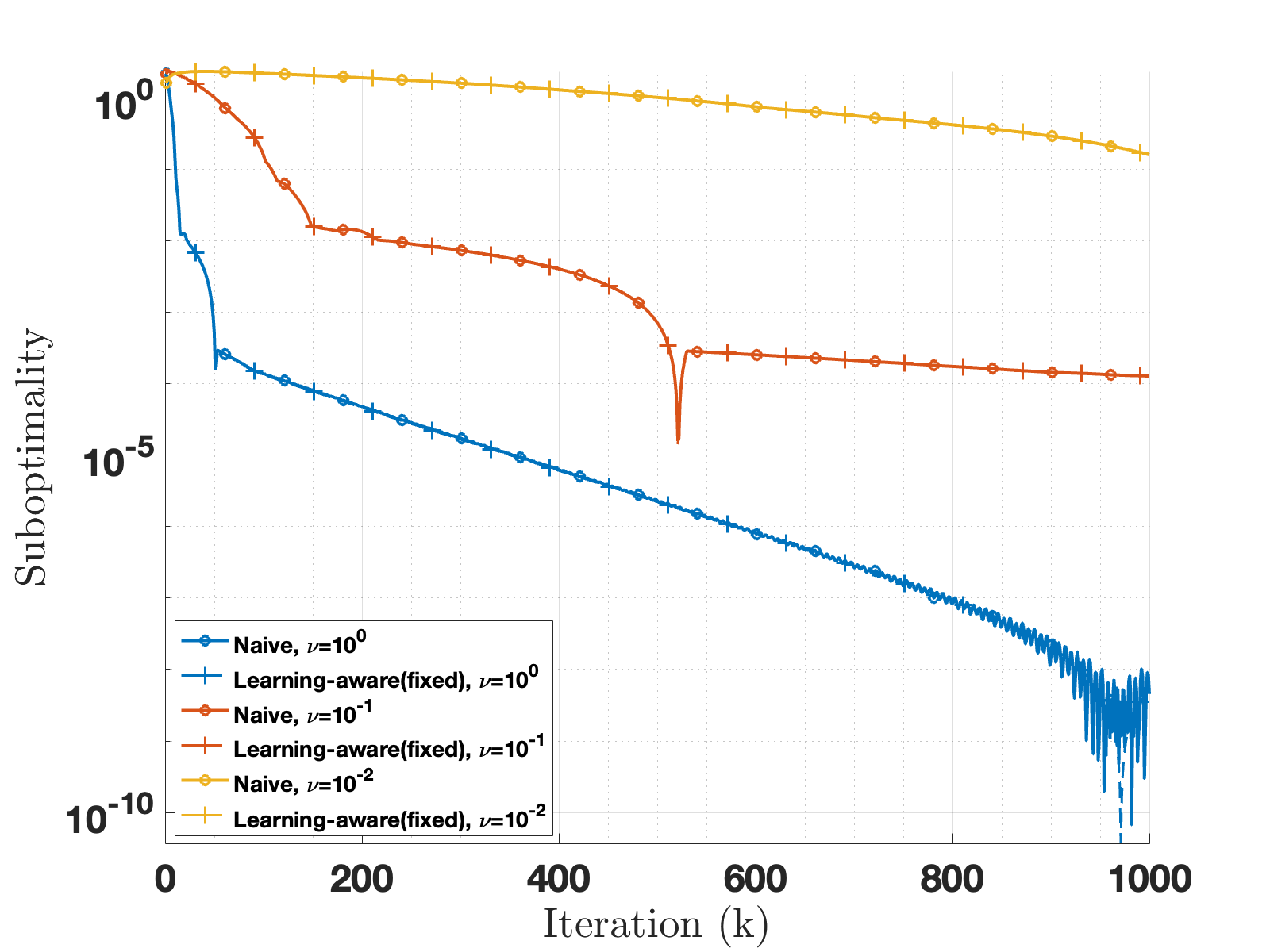}
  \label{fig:sensitivity-subopt}
\end{subfigure}\hfill
\begin{subfigure}[t]{0.48\linewidth}
  \centering
  \includegraphics[width=\linewidth]{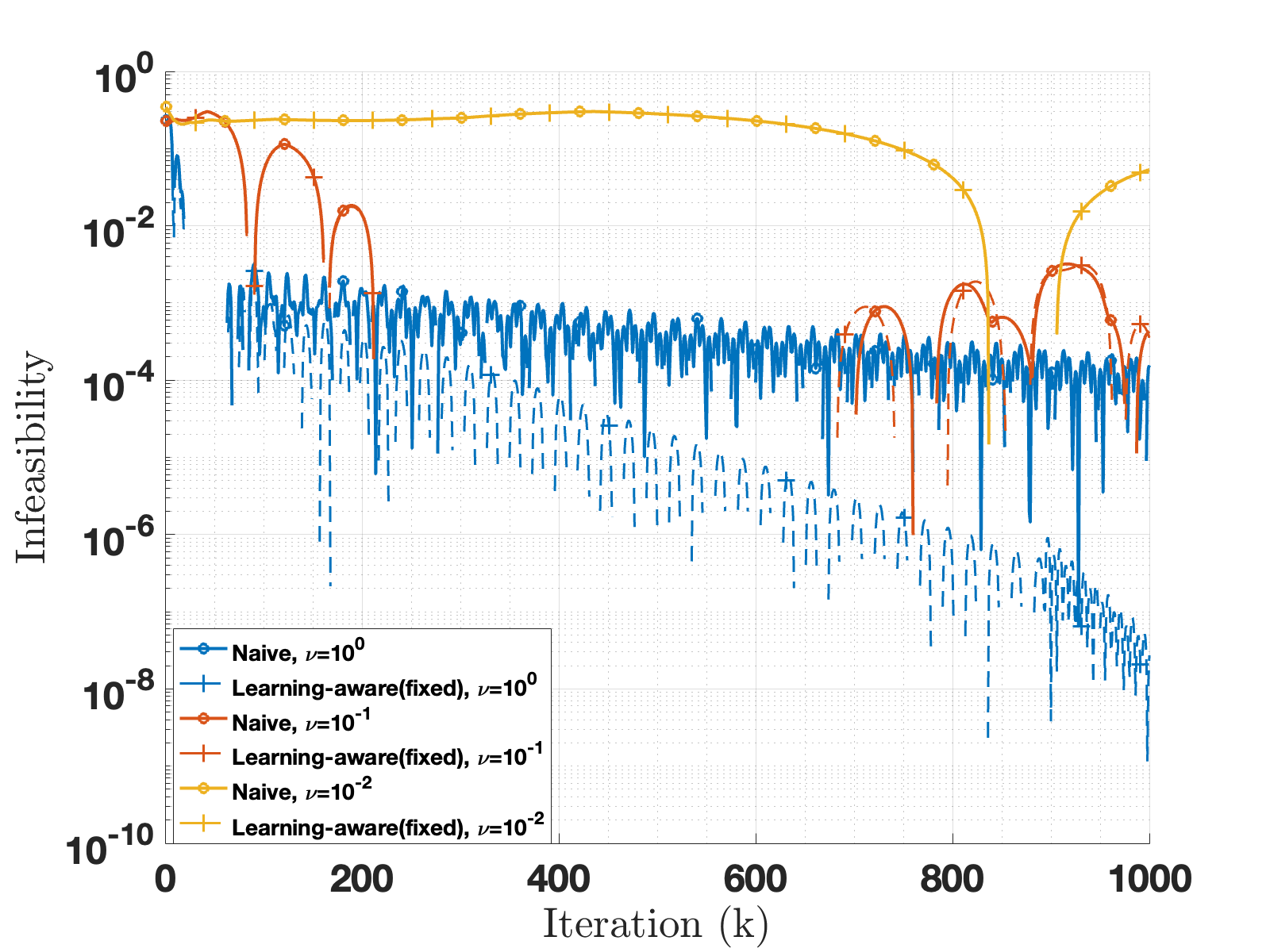}
  \label{fig:sensitivity-infeas}
\end{subfigure}
\caption{Sensitivity to the step-size on the Dow Jones dataset, in terms of
suboptimality (left) and infeasibility (right). The primal--dual step-sizes are scaled by a common factor $\nu\in\{10^{0},10^{-1},10^{-2}\}$ for the Naive and fixed-step Learning-aware approaches; solid and dashed curves of the same color correspond to the two approaches at the same $\nu$.}
\label{fig:sensitivity}
\end{figure}

\section{Conclusion}

In this paper, we studied a class of misspecified saddle point (SP) problems, where the optimization objective depends on an unknown parameter that must be simultaneously learned from data. Unlike prior works that assume fully known or pre-estimated parameters, our framework couples the optimization and learning processes within a unified saddle point formulation.  

We proposed two accelerated primal-dual (APD) algorithms for solving this problem. The first method, Naive APD, directly substitutes the current parameter estimate into the primal-dual updates. While it achieves a convergence rate of $\mathcal{O}(\log K / K)$, its dependence on large Lipschitz and diameter constants may reduce practical performance. To address this limitation, we introduced the Learning-aware APD algorithm, which explicitly incorporates the dynamics of the evolving parameter into the acceleration step and employs an adaptive backtracking line search for step-size selection. This design preserves the same $\mathcal{O}(\log K / K)$ rate but with a smaller $\mathcal O(1)$ constant. 

Furthermore, we extended our framework to the setting where the learning subproblem admits multiple optimal solutions. By considering a structured min-max and adopting a pessimistic formulation, we reformulated the problem as a misspecified saddle point model and developed a modified version of the Learning-aware APD algorithm to address it. This extension achieves a convergence rate of $\mathcal{O}(1/\sqrt{K})$, highlighting the robustness and flexibility of our approach in handling more complex learning-optimization interactions.  

Overall, our contributions provide new algorithmic frameworks and theoretical guarantees for jointly solving optimization and learning problems under parameter misspecification, representing a novel direction in the study of saddle point methods.


\appendix

\section{Supporting Lemmas}\label{sec:supporting-lemmas}

In Lemma \ref{lem:f-bound}, we restate Lemma 7.1 in \cite{hamedani2021primal} along with its proof for the sake of completeness.
\begin{lemma}\label{lem:f-bound}

    Let $\mathcal{X}$ be a finite-dimensional normed space equipped with norm $\|\cdot\|_{\mathcal{X}}$. Consider a closed convex function $f: \mathcal{X} \rightarrow \mathbb{R} \cup \{+\infty\}$ that is $\mu$-strongly convex with respect to $\|\cdot\|_{\mathcal{X}}$, for some $\mu \geq 0$. Let $\mathbf{D}: \mathcal{X} \times \mathcal{X} \to \mathbb{R}_+$ denote the Bregman distance associated with a strictly convex and differentiable function $\phi: \mathcal{X} \rightarrow \mathbb{R}$, where $\phi$ is differentiable on an open domain containing $\textbf{dom}\,f$. For any $\overline{x} \in \textbf{dom}\,f$ and $t > 0$, define
\begin{align}\label{eq:prox-breg-update}
    x^{+} = \argmin_{x \in \mathcal{X}} \left\{ f(x) + t \mathbf{D}(x, \overline{x}) \right\}.
\end{align}
Then, for any $x \in \mathcal{X}$, the following inequality holds:
\begin{align}\label{eq:prox-breg-bound}
    f(x) + t \mathbf{D}(x, \overline{x}) \geq f(x^{+}) + t \mathbf{D}(x^{+}, \overline{x}) + t \mathbf{D}(x, x^{+}) + \frac{\mu}{2} \|x - x^{+}\|_{\mathcal{X}}^2.
\end{align}
\end{lemma}
\begin{proof}
    See Lemma 7.1 in \cite{hamedani2021primal}.
\end{proof}




\section{Proof of Theorem \ref{thm:naive-approach}}\label{AppndxB}
\begin{proof}
In this scenario, the parameter $\theta$ is updated after the variables $x,y$ are updated. Hence, we can define:
\begin{align*}
    &q_k \triangleq \nabla_y \Phi(x_k,y_k;\theta_k) - \nabla_y \Phi(x_{k-1},y_{k-1};\theta_{k}),\\
    &s_k \triangleq \nabla_y \Phi(x_k,y_k;\theta_k) + \eta_k q_k.
\end{align*}
In case of having the optimal parameter $\theta^*$, we can obtain the following bounds:
\begin{align*}
    \|q_k\|_{\mathcal Y^*} 
    & \leq \|\nabla_y \Phi(x_k,y_k;\theta_k) - \nabla_y \Phi(x_{k},y_{k};\theta^*) \|_{\mathcal Y^*} + \| \nabla_y \Phi(x_k,y_k;\theta^*)  \nonumber\\
    & \quad - \nabla_y \Phi(x_{k-1},y_{k-1};\theta^*)\|_{\mathcal Y^*}+\|\nabla_y \Phi(x_{k-1},y_{k-1};\theta^*) - \nabla_y \Phi(x_{k-1},y_{k-1};\theta_{k}) \|_{\mathcal Y^*} \nonumber\\
    & \leq \underbrace{\| \nabla_y \Phi(x_k,y_k;\theta^*) - \nabla_y \Phi(x_{k-1},y_{k-1};\theta^*)\|_{\mathcal Y^*}}_{\triangleq q^*_k} + 2 L^{\Phi}_{y\theta} \|\theta_k - \theta^*\|_{\Theta},
\end{align*}%
where the last inequality is obtained using the triangle inequality and Assumption \ref{assump:phi-lip}.\\

Based on the $\theta$ updating step in the Naive Approach in Algorithm \ref{alg:naive-mis-sp}, and Lemma \ref{lem:f-bound} for the $y$- and $x$-subproblems in Algorithm \ref{alg:naive-mis-sp}, we obtain the following two inequalities that hold for any $y \in \mathcal{Y}$ and $x \in \mathcal{X}$:
\begin{align}
    h(y_{k+1}) & - \langle s_k, y_{k+1}-y \rangle \nonumber\\
    &\leq h(y) + \frac{1}{\sigma_k}\Big[\mathbf D_{\mathcal{Y}}(y,y_k) - \mathbf D_{\mathcal{Y}}(y,y_{k+1}) - \mathbf D_{\mathcal{Y}}(y_{k+1},y_k)\Big] \label{eq:h-bound-NA},\\
    f(x_{k+1}) & + \langle \nabla_x \Phi(x_k,y_{k+1};\theta_{k}),x_{k+1}-x \rangle \nonumber\\
    & \leq f(x) + \frac{1}{\tau_k}\Big[\mathbf D_{\mathcal{X}}(x,x_k)-\mathbf D_{\mathcal{X}}(x,x_{k+1})-\mathbf D_{\mathcal{X}}(x_{k+1},x_k)\Big]\label{eq:f-bound-NA}.
\end{align}
For all $k \geq 0$, let $A_{k+1} \triangleq \frac{1}{\sigma_k}\Big[\mathbf D_{\mathcal{Y}}(y,y_k) - \mathbf D_{\mathcal{Y}}(y,y_{k+1}) - \mathbf D_{\mathcal{Y}}(y_{k+1},y_k)\Big]$ and $B_{k+1} \triangleq \frac{1}{\tau_k}\Big[\mathbf D_{\mathcal{X}}(x,x_k)-\mathbf D_{\mathcal{X}}(x,x_{k+1})-\mathbf D_{\mathcal{X}}(x_{k+1},x_k)\Big]$. In order to see the effect of replacing $\theta^*$ with $\theta_k$ in \eqref{eq:h-bound-NA}, we need to find a bound for $\langle s_k,y_{k+1}-y \rangle$. Suppose $s^*_k \triangleq \nabla_y \Phi(x_k,y_k;\theta^*) + \eta_k q^*_k$, then
\begin{align*}
    \langle s_k,y_{k+1}-y \rangle &= \langle s_k + s^*_k - s^*_k,y_{k+1}-y \rangle\nonumber\\
    & = \langle s^*_k,y_{k+1}-y \rangle + \langle s_k - s^*_k, y_{k+1}-y \rangle.
\end{align*}
Using Cauchy-Schwarz inequality and Assumption \ref{assump:phi-lip}, we can obtain
\begin{align}\label{eq:inner-sk}
    &\langle s_k,y_{k+1}-y \rangle \nonumber\\
    & \quad \leq \langle s^*_k,y_{k+1}-y \rangle + \|s_k - s^*_k \|_{\mathcal Y^*} \|y_{k+1}-y \|_{\mathcal Y}\nonumber\\
    &\quad \leq \langle s_k^*,y_{k+1}-y \rangle + \|\nabla_y \Phi(x_k,y_k;\theta_k) + \eta_k q_k - \nabla_y \Phi(x_k,y_k;\theta^*)- \eta_k q_k^* \|_{\mathcal Y^*} \|y_{k+1}-y \|_{\mathcal Y}  \nonumber\\
    &\quad\leq \langle s^*_k,y_{k+1}-y \rangle + \| \nabla_y \Phi(x_k,y_k;\theta_k) - \nabla_y \Phi(x_k,y_k;\theta^*) \nonumber\\
    & \qquad + \eta_k \Big( \nabla_y \Phi(x_k,y_k;\theta_k) - \nabla_y \Phi(x_k,y_k;\theta^*) + \nabla_y \Phi(x_{k-1},y_{k-1};\theta^*) \nonumber\\
    &\qquad - \nabla_y \Phi(x_{k-1},y_{k-1};\theta_k) \Big)\|_{\mathcal Y^*} \|y_{k+1}-y \|_{\mathcal Y}\nonumber\\
    &\quad \leq \langle s^*_k,y_{k+1}-y \rangle + \Big( \|\nabla_y \Phi(x_k,y_k;\theta_k) - \nabla_y \Phi(x_k,y_k;\theta^*) \|_{\mathcal Y^*} + \eta_k \|\nabla_y \Phi(x_k,y_k;\theta_k) \nonumber\\
    & \qquad - \nabla_y \Phi(x_k,y_k;\theta^*) \|_{\mathcal Y^*} + \eta_k \|\nabla_y \Phi(x_{k-1},y_{k-1};\theta^*) \nonumber\\
    & \qquad- \nabla_y \Phi(x_{k-1},y_{k-1};\theta_k) \|_{\mathcal Y^*} \Big) \|y_{k+1}-y \|_{\mathcal Y}\nonumber\\
    &\quad \leq  \Big(L^{\Phi}_{y\theta} \|\theta_k - \theta^*\|_{\Theta} + \eta_k L^{\Phi}_{y\theta} \|\theta_k - \theta^*\|_{\Theta}+ \eta_k L^{\Phi}_{y\theta} \|\theta_k - \theta^*\|_{\Theta} \Big) \|y_{k+1}-y \|_{\mathcal Y} \nonumber\\
    & \qquad + \langle s^*_k,y_{k+1}-y \rangle \nonumber\\
    &\quad\leq \langle s^*_k,y_{k+1}-y \rangle + \underbrace{\Big( (2\eta_k +1) L^{\Phi}_{y\theta} \|\theta_k - \theta^*\|_{\Theta} \Big)}_{\triangleq 
 \mathcal{E}_1^k} \|y_{k+1}-y \|_{\mathcal Y}.
\end{align}
Next, to find the upper-bound for $\langle \nabla_x \Phi(x_k,y_{k+1};\theta_{k}),x_{k+1}-x \rangle$ in \eqref{eq:f-bound-NA}, we have
\begin{align*}
    &\langle \nabla_x \Phi(x_k,y_{k+1};\theta_{k}),x_{k+1}-x \rangle \\
    & \quad= \langle \nabla_x \Phi(x_k,y_{k+1};\theta_{k}) + \nabla_x \Phi(x_k,y_{k+1};\theta^*)- \nabla_x \Phi(x_k,y_{k+1};\theta^*),x_{k+1}-x \rangle \nonumber\\
    &\quad = \langle \nabla_x \Phi(x_k,y_{k+1};\theta^*),x_{k+1}-x \rangle + \langle \nabla_x \Phi(x_k,y_{k+1};\theta_{k}) - \nabla_x \Phi(x_k,y_{k+1};\theta^*),x_{k+1}-x \rangle.
\end{align*}
Using Cauchy-Schwarz inequality and Assumption \ref{assump:phi-lip}, we can obtain
\begin{align}\label{eq:inner-grad-phi}
    &\langle \nabla_x \Phi(x_k,y_{k+1};\theta_{k}),x_{k+1}-x \rangle \nonumber\\
    & \quad \geq \langle \nabla_x \Phi(x_k,y_{k+1};\theta^*),x_{k+1}-x \rangle \nonumber\\
    & \qquad- \| \nabla_x \Phi(x_k,y_{k+1};\theta_{k}) - \nabla_x \Phi(x_k,y_{k+1};\theta^*)\|_{\mathcal X^*} \| x_{k+1}-x\|_{\mathcal X} \nonumber\\
    &\quad \geq \langle \nabla_x \Phi(x_k,y_{k+1};\theta^*),x_{k+1}-x \rangle - \underbrace{L^{\Phi}_{x\theta} \| \theta_k - \theta^*\|_{\Theta}}_{\triangleq \mathcal{E}_2^k} \| x_{k+1}-x\|_{\mathcal X}.
\end{align}
Now, based on the results obtained in \eqref{eq:inner-sk} and \eqref{eq:inner-grad-phi}, we will update \eqref{eq:h-bound-NA} and \eqref{eq:f-bound-NA} inequalities as
\begin{align}
    h(y_{k+1}) & - \langle s_k^*, y_{k+1}-y \rangle - \mathcal{E}_1^k \|y_{k+1} - y \|_{\mathcal Y} \nonumber\\
    &\leq h(y) + \frac{1}{\sigma_k}\Big[\mathbf D_{\mathcal{Y}}(y,y_k) - \mathbf D_{\mathcal{Y}}(y,y_{k+1}) - \mathbf D_{\mathcal{Y}}(y_{k+1},y_k)\Big] \label{eq:h-bound-Er},\\
    f(x_{k+1}) & + \langle \nabla_x \Phi(x_k,y_{k+1};\theta^*),x_{k+1}-x \rangle - \mathcal{E}_2^k\|x_{k+1}-x \|_{\mathcal X} \nonumber\\
    & \leq f(x) + \frac{1}{\tau_k}\Big[\mathbf D_{\mathcal{X}}(x,x_k)-\mathbf D_{\mathcal{X}}(x,x_{k+1})-\mathbf D_{\mathcal{X}}(x_{k+1},x_k)\Big]\label{eq:f-bound-Er}.
\end{align}
The inner product in \eqref{eq:f-bound-Er} can be lower bounded using the convexity of $\Phi(x,y_{k+1};\theta^*)$ in $x$ as
\begin{align}\label{eq:inner-eq-*}
    &\langle \nabla_x \Phi(x_k,y_{k+1};\theta^*),x_{k+1}-x \rangle \nonumber\\
    & \quad \geq \Phi(x_k,y_{k+1};\theta^*) - \Phi(x,y_{k+1};\theta^*)  + \langle \nabla_x \Phi(x_k,y_{k+1};\theta^*),x_{k+1}-x_k \rangle.
\end{align}
Using this inequality after adding $\Phi(x_{k+1},y_{k+1};\theta^*)$ to both sides of \eqref{eq:f-bound-Er}, we can get
\begin{align}\label{eq:f-bound-er}
    &f(x_{k+1}) + \Phi(x_{k+1},y_{k+1};\theta^*) \nonumber\\
    & \quad\leq f(x) + \Phi(x,y_{k+1};\theta^*) + \mathcal{E}_2^k \|x_{k+1}-x \|_{\mathcal X} + B_{k+1} + \Lambda_k^*,
\end{align}
where $\Lambda_k^* \triangleq \Phi(x_{k+1},y_{k+1};\theta^*) - \Phi(x_{k},y_{k+1};\theta^*) - \langle \nabla_x \Phi(x_k,y_{k+1};\theta^*),x_{k+1}-x_k \rangle$, for $k \geq 0$. For $k\geq 0$, summing \eqref{eq:h-bound-Er} and \eqref{eq:f-bound-er} and rearranging the terms lead to
\begin{align}\label{eq:L-bound-er}
    &\mathcal{L}(x_{k+1},y;\theta^*)-\mathcal{L}(x,y_{k+1};\theta^*) \nonumber\\
    & \quad= f(x_{k+1}) + \Phi(x_{k+1},y;\theta^*)-h(y)-f(x)-\Phi(x,y_{k+1};\theta^*)+h(y_{k+1})\nonumber\\
    &\quad \leq \Phi(x_{k+1},y;\theta^*) - \Phi(x_{k+1},y_{k+1};\theta^*)+ \langle s_k^*, y_{k+1}-y \rangle  + \mathcal{E}_1^k \| y_{k+1}-y\|_{\mathcal Y}\nonumber\\
    & \qquad  + \mathcal{E}_2^k \|x_{k+1}-x \|_{\mathcal X}+ \Lambda_k^* + A_{k+1} + B_{k+1}\nonumber\\
    &\quad \leq - \langle q^*_{k+1},y_{k+1}-y \rangle + \eta_k \langle q^*_k,y_{k+1}-y \rangle + \mathcal{E}_1^k \| y_{k+1}-y\|_{\mathcal Y}\nonumber\\
    & \qquad  + \mathcal{E}_2^k \|x_{k+1}-x \|_{\mathcal X} + \Lambda_k^* + A_{k+1} + B_{k+1},
\end{align}
where in the last inequality we use the concavity of $\Phi(x_{k+1},\cdot;\theta^*)$.\\

To obtain a telescoping sum later, by adding and subtracting $\eta_k q_ky_k$, we can rewrite the bound in \eqref{eq:L-bound-er} as
\begin{align}\label{eq:L-bound-er1}
    &\mathcal{L}(x_{k+1},y;\theta^*)-\mathcal{L}(x,y_{k+1};\theta^*) \nonumber\\
    &\quad \leq \Big[\frac{1}{\tau_k}\mathbf D_{\mathcal{X}}(x,x_k) + \frac{1}{\sigma_k}\mathbf D_{\mathcal{Y}}(y,y_k) + \eta_k \langle q^*_k,y_k-y \rangle \Big]\nonumber\\
    & \qquad - \Big[\frac{1}{\tau_k}\mathbf D_{\mathcal{X}}(x,x_{k+1}) + \frac{1}{\sigma_k}\mathbf D_{\mathcal{Y}}(y,y_{k+1})+ \langle q^*_{k+1},y_{k+1}-y \rangle \Big] \nonumber\\
    & \qquad  + \Lambda_k^* + \mathcal{E}_1^k \| y_{k+1}-y\|_{\mathcal Y} + \mathcal{E}_2^k \|x_{k+1}-x \|_{\mathcal X} \nonumber\\
    & \qquad - \frac{1}{\tau_k}\mathbf D_{\mathcal{X}}(x_{k+1},x_k) - \frac{1}{\sigma_k}\mathbf D_{\mathcal{Y}}(y_{k+1},y_k) + \eta_k \langle q^*_k, y_{k+1} - y_k \rangle.
\end{align}
To bound $\langle q^*_k,y-y_k \rangle$ for any given $y \in \mathcal{Y}$, Let define $p^{*x}_k \triangleq \nabla_y \Phi(x_k,y_k;\theta^*) - \nabla_y \Phi(x_{k-1},y_k;\theta^*)$ and $p^{*y}_k \triangleq \nabla_y \Phi(x_{k-1},y_k;\theta^*) - \nabla_y \Phi(x_{k-1},y_{k-1};\theta^*)$ which imply that $q^*_k = p^{*x}_k + p^{*y}_k$. Moreover, for any $y\in \mathcal{Y}$, $y'\in \mathcal{Y}^*$, and $a>0$, we have $\left|\langle y',y \rangle\right| \leq \frac{a}{2} \|y\|^2_{\mathcal{Y}} + \frac{1}{2a} \|y'\|^2_{\mathcal{Y}^*}$. Therefore, we can obtain for all $k\geq0$:
\begin{align}\label{eq:p*xyk}
    \left|\langle q^*_k,y-y_k \rangle\right| &\leq \alpha_k \mathbf D_{\mathcal{Y}}(y,y_k) + \frac{1}{2\alpha_k}\|p^{*x}_k\|^2_{\mathcal{Y}^*} + \beta_k \mathbf D_{\mathcal{Y}}(y,y_k) + \frac{1}{2\beta_k}\|p^{*y}_k\|^2_{\mathcal{Y}^*},
\end{align}
which holds for any $\alpha_k, \beta_k >0$. Therefore, using \eqref{eq:p*xyk} within \eqref{eq:L-bound-er1}, we can show for $k\geq0$
\begin{subequations}
\begin{align}
    &\mathcal{L}(x_{k+1},y;\theta^*)-\mathcal{L}(x,y_{k+1};\theta^*)\nonumber\\
    & \quad\leq \mathcal{E}_1^k \| y_{k+1}-y\|_{\mathcal Y} + \mathcal{E}_2^k \|x_{k+1}-x \|_{\mathcal X} + Q^*_k(z) - R^*_{k+1}(z) + E^*_k \label{eq:L-bound-naive},\\
    &Q^*_k(z) \triangleq \frac{1}{\tau_k}\mathbf D_{\mathcal{X}}(x,x_k) + \frac{1}{\sigma_k}\mathbf D_{\mathcal{Y}}(y,y_k) + \eta_k \langle q^*_k,y_k-y \rangle + \frac{\eta_k}{2\alpha_k}\|p^{*x}_k\|^2_{\mathcal{Y}^*} \nonumber\\
    &\qquad \qquad+ \frac{\eta_k}{2\beta_k}\|p^{*y}_k\|^2_{\mathcal{Y}^*},\\
    &R^*_{k+1}(z) \triangleq \frac{1}{\tau_k}\mathbf D_{\mathcal{X}}(x,x_{k+1}) + \frac{1}{\sigma_k}\mathbf D_{\mathcal{Y}}(y,y_{k+1}) + \langle q^*_{k+1},y_{k+1}-y \rangle + \frac{1}{2\alpha_{k+1}}\|p^{*x}_{k+1}\|^2_{\mathcal{Y}^*}\nonumber\\
    & \qquad \qquad  + \frac{1}{2\beta_{k+1}}\|p^{*y}_{k+1}\|^2_{\mathcal{Y}^*},\\
    & E^*_k \triangleq \Lambda_k^* + \frac{1}{2\alpha_{k+1}}\|p^{*x}_{k+1}\|^2_{\mathcal{Y}^*} - \frac{1}{\tau_k}\mathbf D_{\mathcal{X}}(x_{k+1},x_k) + \frac{1}{2\beta_{k+1}}\|p^{*y}_{k+1}\|^2_{\mathcal{Y}^*} \nonumber\\
    & \qquad \qquad - \Big(\frac{1}{\sigma_k} - \eta_k(\alpha_k + \beta_k) \Big)\mathbf D_{\mathcal{Y}}(y_{k+1},y_k),\label{eq:L-bound-naive-d}
\end{align}

\end{subequations}
where all the derivations in \eqref{eq:L-bound-naive}-\eqref{eq:L-bound-naive-d} hold for any Bregman distance functions $\mathbf D_{\mathcal{X}}$ and $\mathbf D_{\mathcal{Y}}$.\\
By multiplying both sides by $t_k>0$, summing over $k=0$ to $K-1$, and then using Jensen's inequality, we obtain
\begin{align}\label{eq:sum-bound-naive}
    &T_K (\mathcal{L}(\overline{x}_K,y;\theta^*) - \mathcal{L}(x,\overline{y}_K;\theta^*)) \nonumber\\
    &\quad\leq \sum_{k=0}^{K-1} t_k \Big(Q^*_k(z) - R^*_{k+1}(z) + E^*_k + \mathcal{E}_1^k \| y_{k+1}-y\|_{\mathcal Y} + \mathcal{E}_2^k \|x_{k+1}-x \|_{\mathcal X} \Big),\nonumber\\
    &\quad\leq t_0Q^*_0(z) - t_{K-1}R^*_K(z) + \sum_{k=0}^{K-1}t_kE^*_k + \sum_{k=0}^{K-1}t_k \mathcal{E}_1^k \| y_{k+1}-y\|_{\mathcal Y} \nonumber\\
    &\qquad + \sum_{k=0}^{K-1} t_k \mathcal{E}_2^k \|x_{k+1}-x \|_{\mathcal X},
\end{align}
where $T_K = \sum_{k=0}^{K-1}t_k$ and the last inequality follows from the step-size conditions in Assumption \ref{assump:step-size-cond}-(ii), which imply that $t_{k+1}Q_{k+1}^*(z) - t_kR_{k+1}^*(z)\leq0$ for $k=0$ to $K-2$.\\

\ma{By descent lemma, using Assumption \ref{assump:phi-lip}, and Definition \ref{def:Breg-dis}, for any $k\geq 0$, and $(x,y) \in \mathcal{X} \times \mathcal{Y}$, we obtain
    \begin{align*}
        \Phi(x,y;\theta^*)-\Phi(x_k,y;\theta^*)-\langle \nabla_x \Phi(x_k,y;\theta^*),x-x_k \rangle &\leq \frac{L^{\Phi}_{xx}}{2}\|x-x_k\|^2_{\mathcal{X}} \leq L^{\Phi}_{xx}\mathbf D_{\mathcal{X}}(x,x_k),\\
        \frac{1}{2}\|\nabla_y \Phi(x,y;\theta^*)-\nabla_y \Phi(x_k,y;\theta^*)\|^2_{\mathcal{Y}^*}&\leq \frac{(L^{\Phi}_{yx})^2}{2}\|x-x_k\|^2_{\mathcal{X}}\leq (L^{\Phi}_{yx})^2 \mathbf D_{\mathcal{X}}(x,x_k),\\
        \frac{1}{2}\| \nabla_y \Phi(x_k,y;\theta^*) - \nabla_y \Phi(x_k,y_k;\theta^*)\|^2_{\mathcal{Y}^*}&\leq \frac{(L^{\Phi}_{yy})^2}{2}\|y-y_k\|^2_{\mathcal{Y}}\leq (L^{\Phi}_{yy})^2\mathbf D_{\mathcal{Y}}(y,y_k).
        \end{align*}}
        
    \ma{By evaluating the above inequality at $(x,y;\theta)=(x_{k+1},y_{k+1};\theta^*)$, and using Assumption \ref{assump:step-size-cond2} along with the nonnegativity of Bregman functions, we can show
    \begin{align*}
        E_k^* \triangleq E_k^*(x_{k+1},y_{k+1};\theta^*)&\leq \big(\frac{(L^{\Phi}_{yx})^2}{\alpha_{k+1}}+L^{\Phi}_{xx}-\frac{1}{\tau_k} \big)\mathbf D_{\mathcal{X}}(x_{k+1},x_k)\\
        &\quad + \big(\frac{(L^{\Phi}_{yy})^2}{\beta_{k+1}}+\eta_k(\alpha_k+\beta_k)-\frac{1}{\sigma_k} \big)\mathbf D_{\mathcal{Y}}(y_{k+1},y_k).
    \end{align*}}
Therefore, according to Assumption \ref{assump:step-size-cond2}, $\tau_k,\sigma_k$ and $\eta_k$ are chosen such that $E_k^* \leq0$ for $k=0,...,K-1$; then \eqref{eq:sum-bound-naive} implies that
\begin{align}\label{eq:sum-bound-naive2}
    &T_K (\mathcal{L}(\overline{x}_K,y;\theta^*) - \mathcal{L}(x,\overline{y}_K;\theta^*)) \nonumber\\
    &\quad\leq t_0Q^*_0(z) - t_{K-1}R^*_K(z) + \sum_{k=0}^{K-1} t_k \mathcal{E}_1^k \| y_{k+1}-y\|_{\mathcal Y} + \sum_{k=0}^{K-1} t_k \mathcal{E}_2^k \|x_{k+1}-x \|_{\mathcal X},\nonumber\\
    &\quad\leq \frac{t_0}{\tau_0} \mathbf D_{\mathcal{X}}(x,x_0) + \frac{t_0}{\sigma_0} \mathbf D_{\mathcal{Y}}(y,y_0) + t_K \eta_K \langle q^*_K,y-y_K\rangle \nonumber\\
    & \qquad - t_K \big[\frac{1}{\tau_K} \mathbf D_{\mathcal{X}}(x,x_K) + \frac{1}{\sigma_K} \mathbf D_{\mathcal{Y}}(y,y_K) + \frac{\eta_K}{2\alpha_K}\|p^{*x}_K \|^2_{\mathcal{Y}^*} \nonumber\\
    &\qquad + \frac{\eta_K}{2\beta_K}\|p^{*y}_K \|^2_{\mathcal{Y}^*}\big] + \sum_{k=0}^{K-1} t_k \mathcal{E}_1^k \| y_{k+1}-y\|_{\mathcal Y} + \sum_{k=0}^{K-1} t_k \mathcal{E}_2^k \|x_{k+1}-x \|_{\mathcal X},
\end{align}
where in the last inequality we use $t_KQ_K^*(z) \leq t_{K-1}R_K^*(z)$ and $q_0^* = p^{*x}_0 = p^{*y}_0 = \bold{0}$. By using \eqref{eq:p*xyk} to bound $\langle q_K^*,y-y_K\rangle$   and dividing both sides of \eqref{eq:sum-bound-naive2} by $T_K$, we can get
\begin{align}\label{eq:naive-bound}
    &\mathcal{L}(\overline{x}_K,y;\theta^*) - \mathcal{L}(x,\overline{y}_K;\theta^*) \nonumber\\
    &\quad\leq \frac{1}{T_K} \Big(\frac{1}{\tau_0}\mathbf D_{\mathcal{X}}(x,x_0) + \frac{1}{\sigma_0}\mathbf D_{\mathcal{Y}}(y,y_0) \Big) - \frac{t_K}{T_K}\Big[\frac{1}{\tau_K}\mathbf D_{\mathcal{X}}(x,x_K)  + \Big(\frac{1}{\sigma_K}  \nonumber\\
    &\qquad -\eta_K(\alpha_K+ \beta_K) \Big)\mathbf D_{\mathcal{Y}}(y,y_K) \Big]+ \frac{1}{T_K} \Big(\sum_{k=0}^{K-1} t_k \mathcal{E}_1^k \| y_{k+1}-y\|_{\mathcal Y} \nonumber\\
    & \qquad+ \sum_{k=0}^{K-1} t_k \mathcal{E}_2^k \|x_{k+1}-x \|_{\mathcal X} \Big).
\end{align}
By applying $sup$ to both sides of the inequality \eqref{eq:naive-bound} and using Definition \ref{def:gap-function}, we obtain the desired result in \eqref{eq:naive-sup}.
\end{proof}

\section{Proof of Theorem \ref{thm:ext-multsol}}\label{AppndxC}
\begin{proof}
    Based on Algorithm \ref{alg:ext-multisol}, we can define the following:
\begin{align*}
    &q^y_k \triangleq \nabla_y g_2(x_k,y_k) - \nabla_y g_2(x_{k-1},y_{k-1}),\\
    &s^y_k \triangleq \nabla_y g_2(x_k,y_k) + \eta_k q^y_k,\\
    &q^{\theta}_k \triangleq \left( \nabla_{\theta} g_1(x_k,\theta_k) -w_k \nabla \ell(\theta_k) \right) - \left( \nabla_{\theta} g_1(x_{k-1},\theta_{k-1})-w_{k-1} \nabla \ell(\theta_{k-1}) \right),\\
    &s^{\theta}_k \triangleq \left( \nabla_{\theta} g_1(x_k,\theta_k) -w_k \nabla \ell(\theta_k) \right) + \eta_k q^{\theta}_k,
\end{align*}
where $q^y_k$ and $q^{\theta}_k$ are the momentum parameters. For $k \geq  0$, using Lemma \ref{lem:f-bound} for the $y$-, $x$-, $\theta$- and $w$- subproblems in Algorithm \ref{alg:ext-multisol}, we get the following two inequalities that hold for any $y \in \mathcal{Y}$, $x \in \mathcal{X}$, $\theta \in \Xi$, and $w \in W$:
\begin{align}
    &h(y_{k+1}) - \langle s^y_k, y_{k+1}-y \rangle \nonumber\\
    & \qquad\leq h(y) + \frac{1}{\sigma_k}\Big[\mathbf D_{\mathcal{Y}}(y,y_k) - \mathbf D_{\mathcal{Y}}(y,y_{k+1}) - \mathbf D_{\mathcal{Y}}(y_{k+1},y_k)\Big] \label{eq:h-bound_ext},\\
    &f(x_{k+1}) + \langle \nabla_x g_1(x_k,\theta_{k+1})+\nabla_x g_2(x_k,y_{k+1}),x_{k+1}-x \rangle \nonumber\\
    &\qquad \leq f(x) + \frac{1}{\tau_k}\Big[\mathbf D_{\mathcal{X}}(x,x_k)-\mathbf D_{\mathcal{X}}(x,x_{k+1})-\mathbf D_{\mathcal{X}}(x_{k+1},x_k)\Big]\label{eq:f-bound_ext},\\
    &-\langle s^{\theta}_k, \theta_{k+1}-\theta \rangle \nonumber\\
    & \qquad \leq \frac{1}{\sigma_k}\Big[\mathbf D_{\Theta}(\theta,\theta_k) - \mathbf D_{\Theta}(\theta,\theta_{k+1}) - \mathbf D_{\Theta}(\theta_{k+1},\theta_k)\Big] \label{eq:theta-bound_ext},\\
    &-\langle \ell(\theta_{k+1})-\ell_{k+1}-\epsilon, w_{k+1}-w \rangle\nonumber\\
    &\qquad \leq \frac{1}{\tau_k}\Big[\mathbf D_{\mathcal W}(w,w_k)-\mathbf D_{\mathcal W}(w,w_{k+1})-\mathbf D_{\mathcal W}(w_{k+1},w_k)\Big]\label{eq:w-bound_ext}.
\end{align}
For all $k \geq 0$, let $a_{k+1} \triangleq \frac{1}{\sigma_k}\Big[\mathbf D_{\mathcal{Y}}(y,y_k) - \mathbf D_{\mathcal{Y}}(y,y_{k+1}) - \mathbf D_{\mathcal{Y}}(y_{k+1},y_k)\Big]$, $b_{k+1} \triangleq \frac{1}{\tau_k}\Big[\mathbf D_{\mathcal{X}}(x,x_k)-\mathbf D_{\mathcal{X}}(x,x_{k+1})-\mathbf D_{\mathcal{X}}(x_{k+1},x_k)\Big]$, $c_{k+1} \triangleq \frac{1}{\sigma_k}\Big[\mathbf D_{\Theta}(\theta,\theta_k) - \mathbf D_{\Theta}(\theta,\theta_{k+1}) - \mathbf D_{\Theta}(\theta_{k+1},\theta_k)\Big]$ and $d_{k+1} \triangleq \frac{1}{\tau_k}\Big[\mathbf D_{\mathcal W}(w,w_k)-\mathbf D_{\mathcal W}(w,w_{k+1})-\mathbf D_{\mathcal W}(w_{k+1},w_k)\Big]$. 
By adding and subtracting $\nabla_x g_1(x_k, \theta_{k+1})x_k$ and $\nabla_x g_2(x_k, y_{k+1})x_k$, the inner product in \eqref{eq:f-bound_ext} can be lower bounded using convexity of $g_1(x_k,\theta_{k+1})$ and $g_2(x_k,y_{k+1})$ as follows:
\begin{align}\label{eq:inner-eq-ext}
    &\langle \nabla_x g_1(x_k,\theta_{k+1})+\nabla_x g_2(x_k,y_{k+1}),x_{k+1}-x \rangle \nonumber\\
    & \quad= \langle \nabla_x g_1(x_k,\theta_{k+1})+\nabla_x g_2(x_k,y_{k+1}),x_{k}-x \rangle \nonumber\\
    & \qquad+ \langle \nabla_x g_1(x_k,\theta_{k+1})+\nabla_x g_2(x_k,y_{k+1}),x_{k+1}-x_k \rangle,\nonumber\\
    &\quad \geq g_1(x_k,\theta_{k+1})-g_1(x,\theta_{k+1}) + g_2(x_k,y_{k+1})-g_2(x,y_{k+1}) \nonumber\\
    & \qquad+ \langle \nabla_x g_1(x_k,\theta_{k+1})+\nabla_x g_2(x_k,y_{k+1}),x_{k+1}-x_k \rangle.
\end{align}
Using this inequality after adding $g_1(x_{k+1},\theta_{k+1})$ and $g_2(x_{k+1},y_{k+1})$ to both sides of \eqref{eq:f-bound_ext}, we can get
\begin{align}\label{eq:f-bound2-ext}
    &f(x_{k+1}) + g_1(x_{k+1},\theta_{k+1}) +g_2(x_{k+1},y_{k+1})\nonumber\\
    & \quad \leq f(x) + g_1(x,\theta_{k+1}) + g_2(x,y_{k+1}) + b_{k+1} + \bar \Lambda_k,
\end{align}
where $\bar \Lambda_k \triangleq g_1(x_{k+1},\theta_{k+1})+g_2(x_{k+1},y_{k+1})-\left(g_1(x_k,\theta_{k+1}) + g_2(x_k,y_{k+1}) \right)-\langle \nabla_x g_1(x_k,\theta_{k+1})+\nabla_x g_2(x_k,y_{k+1}),x_{k+1}-x_k\rangle$, for $k \geq 0$. For $k\geq 0$, summing \eqref{eq:h-bound_ext}, \eqref{eq:theta-bound_ext}, \eqref{eq:w-bound_ext} and \eqref{eq:f-bound2-ext} and rearranging the terms lead to
\begin{align}\label{eq:L-bound1-ext}
    &\mathcal{L}(x_{k+1},w_{k+1},y,\theta;\ell_{k+1})-\mathcal{L}(x,w,y_{k+1},\theta_{k+1};\ell_{k+1})\nonumber\\
    & \quad =f(x_{k+1})+g_1(x_{k+1},\theta) + g_2(x_{k+1},y)-\langle w_{k+1},\ell(\theta)-\ell_{k+1}-\epsilon \rangle - h(y)\nonumber\\
    & \qquad -f(x) -g_1(x,\theta_{k+1})-g_2(x,y_{k+1})+\langle w, \ell(\theta_{k+1})-\ell_{k+1}-\epsilon \rangle +h(y_{k+1})\nonumber\\
    &\quad \leq -\langle q^y_{k+1}, y_{k+1}-y \rangle + \eta_k \langle q^y_k, y_{k+1}-y \rangle -\langle q^{\theta}_{k+1}, \theta_{k+1}-\theta \rangle + \eta_k \langle q^{\theta}_k, \theta_{k+1}-\theta \rangle\nonumber\\
    &\qquad + \bar \Lambda_k + a_{k+1} + b_{k+1} + c_{k+1} + d_{k+1},
\end{align}
where in the last inequality we use the concavity of $g_1(x_{k+1},\cdot)$ and $g_2(x_{k+1},\cdot)$, and the convexity of $\ell(\cdot)$ along with adding and subtracting $\langle w_{k+1} \nabla \ell(\theta_{k+1}),\theta-\theta_{k+1}\rangle$.\\
To obtain a telescoping sum later, by adding and subtracting $\eta_k q^y_ky_k$ and $\eta_k q^{\theta}_k\theta_k$, we can rewrite the bound in \eqref{eq:L-bound1-ext} as
\begin{align}\label{eq:L-bound2-ext}
    &\mathcal{L}(x_{k+1},w_{k+1},y,\theta;\ell_{k+1})-\mathcal{L}(x,w,y_{k+1},\theta_{k+1};\ell_{k+1})\nonumber\\
    & \quad \leq \Big[\frac{1}{\tau_k}\mathbf D_{\mathcal{X}}(x,x_k)+\frac{1}{\tau_k}\mathbf D_{\mathcal W}(w,w_k)+\frac{1}{\sigma_k}\mathbf D_{\mathcal{Y}}(y,y_k)+\frac{1}{\sigma_k}\mathbf D_{\Theta}(\theta,\theta_k)+\eta_k \langle q^y_k,y_k-y \rangle\nonumber\\
    & \qquad  + \eta_k \langle q^{\theta}_k,\theta_k-\theta \rangle\Big]\nonumber\\
    &\qquad -\Big[\frac{1}{\tau_k}\mathbf D_{\mathcal{X}}(x,x_{k+1})+\frac{1}{\tau_k}\mathbf D_{\mathcal W}(w,w_{k+1})+\frac{1}{\sigma_k}\mathbf D_{\mathcal{Y}}(y,y_{k+1})+\frac{1}{\sigma_k}\mathbf D_{\Theta}(\theta,\theta_{k+1})\nonumber\\
    & \qquad+ \langle q^y_{k+1}, y_{k+1}-y \rangle + \langle q^{\theta}_{k+1},\theta_{k+1}-\theta \rangle \Big]\nonumber\\
    & \qquad + \bar \Lambda_k -\frac{1}{\tau_k}\mathbf D_{\mathcal{X}}(x_{k+1},x_k)-\frac{1}{\tau_k}\mathbf D_{\mathcal W}(w_{k+1},w_k)-\frac{1}{\sigma_k}\mathbf D_{\mathcal{Y}}(y_{k+1},y_k)-\frac{1}{\sigma_k}\mathbf D_{\Theta}(\theta_{k+1},\theta_k)\nonumber\\
    &\qquad+ \eta_k \langle q^y_k, y_{k+1}-y_k \rangle + \eta_k \langle q^{\theta}_k, \theta_{k+1}-\theta_k \rangle.
\end{align}
One can bound $\langle q^y_k,y-y_k \rangle$ and $\langle q^{\theta}_k,\theta-\theta_k \rangle$ for any given $y \in \mathcal{Y}$ and $\theta \in \Xi$, respectively, as follows. Let $p^{y_1}_k \triangleq \nabla_y g_2(x_k,y_k)-\nabla_yg_2(x_{k-1},y_k)$, and $p^{y_2}_k \triangleq \nabla_y g_2(x_{k-1},y_k)-\nabla_yg_2(x_{k-1},y_{k-1})$ which imply that $q^y_k = p^{y_1}_k + p^{y_2}_k$. Moreover, let $p^{\theta_1}_k=\nabla_{\theta}g_1(x_k,\theta_k)-\nabla_{\theta}g_1(x_{k-1},\theta_k)$, $p^{\theta_2}_k=\nabla_{\theta}g_1(x_{k-1},\theta_k)-\nabla_{\theta}g_1(x_{k-1},\theta_{k-1})$, $p^{w_1}_k= -w_k \nabla \ell(\theta_k)+w_{k-1}\nabla \ell(\theta_{k})$, and $p^{w_2}_k= -w_{k-1} \nabla \ell(\theta_{k}) +w_{k-1}\nabla \ell(\theta_{k-1})$ which imply that $q^{\theta}_k = p^{\theta_1}_k + p^{\theta_2}_k + p^{w_1}_k + p^{w_2}_k$. Using Young's inequality, we obtain for all $k\geq0$:
\begin{align}\label{eq:pxyk1-ext}
    |\langle q^y_k,y-y_k \rangle| &\leq \alpha_k \mathbf D_{\mathcal{Y}}(y,y_k) + \frac{1}{2\alpha_k}\|p^{y_1}_k\|^2_{\mathcal{Y}^*} + \beta_k \mathbf D_{\mathcal{Y}}(y,y_k) + \frac{1}{2\beta_k}\|p^{y_2}_k\|^2_{\mathcal{Y}^*},
\end{align}
\begin{align}\label{eq:pxyk2-ext}
    |\langle q^{\theta}_k,\theta-\theta_k \rangle| &\leq \alpha_k \mathbf D_{\Theta}(\theta,\theta_k) + \frac{1}{2\alpha_k}\|p^{\theta_1}_k + p^{\theta_2}_k\|^2_{\Theta^*} + \beta_k \mathbf D_{\Theta}(\theta,\theta_k) + \frac{1}{2\beta_k}\|p^{w_1}_k + p^{w_2}_k\|^2_{\Theta^*},
\end{align}
which holds for any $\alpha_k, \beta_k >0$. Therefore, using \eqref{eq:pxyk1-ext} and \eqref{eq:pxyk2-ext} within \eqref{eq:L-bound2-ext}, we can show for $k\geq0$
\begin{subequations}
\begin{align}
    &\mathcal{L}(x_{k+1},w_{k+1},y,\theta;\ell_{k+1})-\mathcal{L}(x,w,y_{k+1},\theta_{k+1};\ell_{k+1})\leq \bar Q_k(z) - \bar R_{k+1}(z) + \bar E_k \label{eq:L-bound-fnl-ext},\\
    &\bar Q_k(z) \triangleq \frac{1}{\tau_k}\mathbf D_{\mathcal{X}}(x,x_k)+\frac{1}{\tau_k}\mathbf D_{\mathcal W}(w,w_k)+\frac{1}{\sigma_k}\mathbf D_{\mathcal{Y}}(y,y_k)+\frac{1}{\sigma_k}\mathbf D_{\Theta}(\theta,\theta_k)+\eta_k \langle q^y_k,y_k-y \rangle\nonumber\\
    & \qquad \qquad + \eta_k \langle q^{\theta}_k,\theta_k-\theta \rangle + \frac{\eta_k}{2\alpha_k}\|p^{y_1}_k\|^2_{\mathcal{Y}^*}+ \frac{\eta_k}{2\beta_k}\|p^{y_2}_k\|^2_{\mathcal{Y}^*}+ \frac{\eta_k}{2\alpha_k}\|p^{\theta_1}_k + p^{\theta_2}_k\|^2_{\Theta^*} \nonumber\\
    & \qquad \qquad+ \frac{\eta_k}{2\beta_k}\|p^{w_1}_k + p^{w_2}_k\|^2_{\Theta^*},\\
    &\bar R_{k+1}(z) \triangleq \frac{1}{\tau_k}\mathbf D_{\mathcal{X}}(x,x_{k+1})+\frac{1}{\tau_k}\mathbf D_{\mathcal W}(w,w_{k+1})+\frac{1}{\sigma_k}\mathbf D_{\mathcal{Y}}(y,y_{k+1})+\frac{1}{\sigma_k}\mathbf D_{\Theta}(\theta,\theta_{k+1})\nonumber\\
    & \qquad+ \langle q^y_{k+1}, y_{k+1}-y \rangle + \langle q^{\theta}_{k+1},\theta_{k+1}-\theta \rangle + \frac{1}{2\alpha_{k+1}}\|p^{y_1}_{k+1}\|^2_{\mathcal{Y}^*}+ \frac{1}{2\beta_{k+1}}\|p^{y_2}_{k+1}\|^2_{\mathcal{Y}^*}\nonumber\\
    &\qquad \qquad+ \frac{1}{2\alpha_{k+1}}\|p^{\theta_1}_{k+1}+p^{\theta_2}_{k+1}\|^2_{\Theta^*}+ \frac{1}{2\beta_{k+1}}\|p^{w_1}_{k+1}+p^{w_2}_{k+1}\|^2_{\Theta^*},\\
    & \bar E_k \triangleq \bar \Lambda_k + \frac{1}{2\alpha_{k+1}}\|p^{y_1}_{k+1}\|^2_{\mathcal{Y}^*}+ \frac{1}{2\beta_{k+1}}\|p^{y_2}_{k+1}\|^2_{\mathcal{Y}^*}+ \frac{1}{2\alpha_{k+1}}\|p^{\theta_1}_{k+1}+p^{\theta_2}_{k+1}\|^2_{\Theta^*} \nonumber\\
    & \qquad \qquad + \frac{1}{2\beta_{k+1}}\|p^{w_1}_{k+1}+p^{w_2}_{k+1}\|^2_{\Theta^*}- \Big(\frac{1}{\sigma_k} - \eta_k(\alpha_k + \beta_k) \Big)\left(\mathbf D_{\mathcal{Y}}(y_{k+1},y_k)+\mathbf D_{\Theta}(\theta_{k+1},\theta_k)\right)\nonumber\\
    & \qquad \qquad -\frac{1}{\tau_k}\mathbf D_{\mathcal{X}}(x_{k+1},x_k)-\frac{1}{\tau_k}\mathbf D_{\mathcal W}(w_{k+1},w_k).
\end{align}
\end{subequations}
Considering that the objective function $\mathcal{L}(x,w,y,\theta;\ell^*)$ is Lipschitz \ma{continuous with respect to the parameter $\ell^*$ with constant $L_{\ell}$ (see \eqref{eq:lip-L-ext}),} we can show
\begin{align*}
    &|\mathcal{L}(x_{k+1},w_{k+1},y,\theta;\ell_{k+1}) - \mathcal{L}(x_{k+1},w_{k+1},y,\theta;\ell^*)| \leq L_{\ell} \|\ell_{k+1} - \ell^*\|_2,\\
    &|\mathcal{L}(x,w,y_{k+1},\theta_{k+1};\ell_{k+1}) - \mathcal{L}(x,w,y_{k+1},\theta_{k+1};\ell^*)| \leq L_{\ell} \|\ell_{k+1} - \ell^*\|_2.
\end{align*}
For $\ell^*, \ell_{k+1}$ we can show
\begin{align*}
    &{\mathcal{L}(x_{k+1},w_{k+1},y,\theta;\ell^*) - \mathcal{L}(x,w,y_{k+1},\theta_{k+1};\ell^*)} \nonumber\\
    & \quad= \mathcal{L}(x_{k+1},w_{k+1},y,\theta;\ell^*) - \mathcal{L}(x_{k+1},w_{k+1},y,\theta;\ell_{k+1})  + \mathcal{L}(x_{k+1},w_{k+1},y,\theta;\ell_{k+1})  \nonumber\\
    & \qquad - \mathcal{L}(x,w,y_{k+1},\theta_{k+1};\ell_{k+1})+ \mathcal{L}(x,w,y_{k+1},\theta_{k+1};\ell_{k+1}) - \mathcal{L}(x,w,y_{k+1},\theta_{k+1};\ell^*).
\end{align*}
Then, using the triangle inequality, we can obtain
\begin{align}\label{eq:L-bound-missp-ext}
    &{\mathcal{L}(x_{k+1},w_{k+1},y,\theta;\ell^*) - \mathcal{L}(x,w,y_{k+1},\theta_{k+1};\ell^*)} \nonumber\\
    &\quad \leq | \mathcal{L}(x_{k+1},w_{k+1},y,\theta;\ell^*) - \mathcal{L}(x_{k+1},w_{k+1},y,\theta;\ell_{k+1})| + \mathcal{L}(x_{k+1},w_{k+1},y,\theta;\ell_{k+1})  \nonumber\\
    & \qquad - \mathcal{L}(x,w,y_{k+1},\theta_{k+1};\ell_{k+1}) + | \mathcal{L}(x,w,y_{k+1},\theta_{k+1};\ell_{k+1}) - \mathcal{L}(x,w,y_{k+1},\theta_{k+1};\ell^*)|\nonumber\\
    &\quad\leq \bar Q_k(z) - \bar R_{k+1}(z)+ \bar E_k +  2L_{\ell}\|\ell_{k+1}-\ell^*\|_2,
\end{align}
where the inequality in \eqref{eq:L-bound-missp-ext} is obtained by using \eqref{eq:lip-L-ext} in \eqref{eq:L-bound-fnl-ext}.\\

Next, multiplying both sides by $t_k>0$, summing over $k=0$ to $K-1$, and then using Jensen's inequality, we obtain
\begin{align}\label{eq:sum-bound-missp-ext}
    &T_K (\mathcal{L}(\overline{x}_K,\overline{w}_K,y,\theta;\ell^*) - \mathcal{L}(x,w,\overline{y}_K,\overline{\theta}_K;\ell^*)) \nonumber\\
    & \quad \leq \sum_{k=0}^{K-1} t_k \Big(\bar Q_k(z) - \bar R_{k+1}(z) + \bar E_k + 2L_{\ell}\|\ell_{k+1}-\ell^*\|_2 \Big),\nonumber\\
    &\quad\leq t_0 \bar Q_0(z) - t_{K-1} \bar R_K(z) + \sum_{k=0}^{K-1}t_k \bar E_k + 2L_{\ell}\sum_{k=0}^{K-1}t_k\|\ell_{k+1}-\ell^*\|_2,
\end{align}
where $T_K = \sum_{k=0}^{K-1}t_k$ and the last inequality follows from the step-size conditions in Assumption \ref{assump:step-size-cond-ext}-(ii), which imply that $t_{k+1} \bar Q_{k+1}(z) - t_k \bar R_{k+1}(z)\leq0$ for $k=0$ to $K-2$.\\

According to Assumption \ref{assump:step-size-cond-ext}, $\tau_k,\sigma_k$ and $\eta_k$ are chosen such that $\bar E_k \leq0$ for $k=0,...,K-1$; then \eqref{eq:sum-bound-missp-ext} implies that
\begin{align}\label{eq:sum-bound-missp2-ext}
    &T_K (\mathcal{L}(\overline{x}_K,\overline{w}_K,y,\theta;\ell^*) - \mathcal{L}(x,w,\overline{y}_K,\overline \theta_K;\ell^*)) \nonumber\\
    & \quad \leq t_0 \bar Q_0(z) - t_{K-1} \bar R_K(z) + 2L_{\ell}\sum_{k=0}^{K-1}t_k\|\ell_{k+1}-\ell^*\|_2,\nonumber\\
    &\quad\leq \frac{t_0}{\tau_0} \mathbf D_{\mathcal{X}}(x,x_0) + \frac{t_0}{\tau_0} \mathbf D_{\mathcal W}(w,w_0) + \frac{t_0}{\sigma_0} \mathbf D_{\mathcal{Y}}(y,y_0) + \frac{t_0}{\sigma_0} \mathbf D_{\Theta}(\theta,\theta_0) + t_K \eta_K \langle q^y_K,y-y_K\rangle \nonumber\\
    & \qquad + t_K \eta_K \langle q^{\theta}_K,\theta-\theta_K\rangle- t_K \big[\frac{1}{\tau_K} \mathbf D_{\mathcal{X}}(x,x_K) + \frac{1}{\tau_K} \mathbf D_{\mathcal W}(w,w_K) + \frac{1}{\sigma_K} \mathbf D_{\mathcal{Y}}(y,y_K)  \nonumber\\
    & \qquad  + \frac{1}{\sigma_K} \mathbf D_{\Theta}(\theta,\theta_K) + \frac{\eta_K}{2\alpha_K}\|p^{y_1}_K \|^2_{\mathcal{Y}^*}+\frac{\eta_K}{2\beta_K}\|p^{y_2}_K \|^2_{\mathcal{Y}^*}+ \frac{\eta_K}{2\alpha_K}\|p^{\theta_1}_K +p^{\theta_2}_K\|^2_{\Theta^*} \nonumber\\
    & \qquad + \frac{\eta_K}{2\beta_K}\|p^{w_1}_K +p^{w_2}_K\|^2_{\Theta^*}\big] +2L_{\ell}\sum_{k=0}^{K-1}t_k\|\ell_{k+1}-\ell^*\|_2,
\end{align}
where in the last inequality we use $t_K \bar Q_K(z) \leq t_{K-1} \bar R_K(z)$ and $q^y_0,q^{\theta}_0,p^{y_1}_0,p^{y_2}_0, p^{\theta_1}_0, p^{\theta_2}_0,p^{w_1}_0,p^{w_2}_0 = \bold{0}$. Using \eqref{eq:pxyk1-ext} and \eqref{eq:pxyk2-ext} to bound $\langle q^y_K,y-y_K\rangle$ and $\langle q^{\theta}_K,\theta-\theta_K\rangle$ and dividing both sides of \eqref{eq:sum-bound-missp2-ext} by $T_K$, we can get

\begin{align}\label{eq:missp-bound-ext}
    &\mathcal{L}(\overline{x}_K,\overline{w}_K,y,\theta;\ell^*) - \mathcal{L}(x,w,\overline{y}_K,\overline \theta_K;\ell^*)\nonumber\\
    & \quad\leq \frac{1}{T_K} \Big(\frac{1}{\tau_0} \mathbf D_{\mathcal{X}}(x,x_0) + \frac{1}{\tau_0} \mathbf D_{\mathcal W}(w,w_0) + \frac{1}{\sigma_0} \mathbf D_{\mathcal{Y}}(y,y_0) + \frac{1}{\sigma_0} \mathbf D_{\Theta}(\theta,\theta_0) \Big)  \nonumber\\
    &\qquad + \frac{2L_{\ell}}{T_K} \sum_{k=0}^{K-1}t_k\|\ell_{k+1}-\ell^*\|_2 - \frac{t_K}{T_K}\Big[\frac{1}{\tau_K}\mathbf D_{\mathcal{X}}(x,x_K) + \frac{1}{\tau_K}\mathbf D_{\mathcal W}(w,w_K) \nonumber\\
    & \qquad+ \Big(\frac{1}{\sigma_K}-\eta_K(\alpha_K+ \beta_K) \Big) \Big(\mathbf D_{\mathcal{Y}}(y,y_K) + \mathbf D_{\Theta}(\theta,\theta_K) \Big) \Big].
\end{align}
By applying $sup$ to both sides of the inequality \eqref{eq:missp-bound-ext} and using Definition \ref{def:gap-function-ext}, we obtain the desired result in \eqref{eq:missp-sup-ext}.
\end{proof}




\bibliographystyle{plain}
\bibliography{reference}

\end{document}